\documentclass[review]{elsarticle}

\usepackage{lineno,hyperref}
\modulolinenumbers[5]

\usepackage{amsmath}
\usepackage{amssymb}
\usepackage{amsthm}
\usepackage{color}
\usepackage{subfig}
\usepackage{algorithm}
\usepackage{algorithmic}

\newcommand{\vc}[1]{\mathbf{#1}}
\newcommand{\mat}[1]{\mathbf{#1}}
\newcommand{\Real}{\mathbb R}

\newcommand{\vv}[1]{\boldsymbol{#1}}

\newcommand{\be}{\begin{equation}}
\newcommand{\ee}{\end{equation}}

\newtheorem{thm}{Theorem}[section]
\newtheorem{cor}[thm]{Corollary}
\newtheorem{lem}[thm]{Lemma}

\DeclareMathOperator*{\argmin}{\mathrm{argmin}}
\DeclareMathOperator*{\argmax}{\mathrm{argmax}}

\newcommand{\BigO}[1]{\ensuremath{\operatorname{O}\bigl(#1\bigr)}}

\usepackage[normalem]{ulem}
\usepackage{soul}
\newcounter{ToDo}
\newcounter{guocomm}
\newcounter{Note}
\definecolor{blue-violet}{rgb}{0.54, 0.17, 0.89}
\definecolor{mygreen}{rgb}{0.0, 0.5, 0.0}
\definecolor{awesome}{rgb}{1.0, 0.13, 0.32}
\definecolor{yellow}{rgb}{0.8, 0.8, 0.0}
\definecolor{bostonuniversityred}{rgb}{0.8, 0.0, 0.0}

\journal{Neurocomputing}

\begin{document}

\begin{frontmatter}

\title{Efficient Sparse Subspace Clustering by Nearest Neighbour Filtering}

\author[usyd]{Stephen Tierney\corref{mycorrespondingauthor}}
\cortext[mycorrespondingauthor]{Corresponding author}
\ead{stephen.tierney@sydney.edu.au}

\author[wsu]{Yi Guo}
\ead{y.guo@westernsydney.edu.au}

\author[usyd]{Junbin Gao}
\ead{junbin.gao@sydney.edu.au}

\address[usyd]{Discipline of Business Analytics, The University of Sydney Business School,\\ The University of Sydney, NSW 2006, Australia}
\address[wsu]{Centre for Research in Mathematics, School of Computing, Engineering and Mathematics, Western Sydney University, Parramatta, NSW 2150, Australia}


\begin{abstract}
Sparse Subspace Clustering (SSC) has been used extensively for subspace identification tasks due to its theoretical guarantees and relative ease of implementation. However SSC has quadratic computation and memory requirements with respect to the number of input data points. This burden has prohibited SSCs use for all but the smallest datasets. To overcome this we propose a new method, k-SSC, that screens out a large number of data points to both reduce SSC to linear memory and computational requirements. We provide theoretical analysis for the bounds of success for k-SSC. Our experiments show that k-SSC exceeds theoretical expectations and outperforms existing SSC approximations by maintaining the classification performance of SSC. Furthermore in the spirit of reproducible research we have publicly released the source code for k-SSC\footnote{{\color{red} \url{https://github.com/sjtrny/kssc}}}
\end{abstract}

%

\end{frontmatter}


\section{Introduction}

As the resolution of capture devices continue to increase so does the burden on analytical and classification algorithms. Furthermore high dimensional data is subject to the curse of dimensionality. It is thus necessary to reduce the dimensionality of data to facilitate data analysis. Most common dimension reduction is performed using Principal Component Analysis (PCA) \citep{jolliffe2002principal, RamsaySilverman2005, jacques2014functional}, which takes a collection of data points from their original high dimensional space and fits them to a lower dimensional subspace. PCA and associated techniques assume that the entire dataset occupies a single subspace. However in reality large datasets are often composed of a union of subspaces. Therefore it is imperative that the individual subspaces are identified so that the data can be partitioned and dimension reduction can be performed on each subspace separately. The task of assigning each data point to its respective subspace is known as subspace clustering.

More formally we express the subspace clustering problem as follows: given a data matrix of $N$ observed column-wise samples $\mathbf X = [\mathbf x_1, \mathbf x_2, \dots, \mathbf x_N ]$  $\in \mathbb{R}^{D \times N}$, where $D$ is the dimension of the data, the objective of subspace clustering is to learn the corresponding subspace labels $\mathbf l = [ l_1, l_2, \dots, l_N ] \in \mathbb{N}^{N}$ for all the data points where each $l_i\in \{1,\ldots,p\}$. Data within $\mathbf X$ is assumed to be drawn from a union of $p$ subspaces $\{S_i\}^p_{i=1}$ of dimensions $\{d_i\}^p_{i=1}$. Both the number of subspaces $p$ and the dimension of each subspace $d_i$ are unknown. To further complicate the problem it is rarely the case that $\mathbf X$ is noise or corruption free. The data is often subject to noise or corruption either at the time of capture (e.g. a digital imaging device) or during transmission (e.g. wireless communication). It is quite clear that subspace clustering is a difficult task since one must produce accurate results quickly while contending with numerous unknown parameters and large volume of potentially noisy data.

\cite{elhamifar2009sparse} introduced an elegant method for subspace clustering called ``Sparse Subspace Clustering'' (SSC). SSC exploits the self-expressive property of data \citep{elhamifar2012sparse} to find the subspaces:
\begin{quote}
  {\it{each data point in a union of subspaces can be efficiently reconstructed by a combination of other points in the data}}
\end{quote}
which gives the relation
\begin{align}
\mathbf x_i = \mathbf X \mathbf z_i,
\end{align}
where $\mathbf z_i$ is a vector of reconstruction coefficients for $\mathbf x_i$. We can then construct a model for the entire dataset as $\mathbf X = \mathbf X \mathbf Z$. In this unrestricted case there are near infinite possibilities for the coefficient matrix $\mathbf Z$. Fortunately to reconstruct each data point one only needs $\{d_i\}$ other points. This means that each data point can be sparsely represented by the other points. Sparse Subspace Clustering as its name suggests exploits this fact. The objective function for SSC is
\begin{align}
\min_{\mathbf Z} \| \mathbf Z \|_1 \quad \text{s.t.}\ \mathbf X = \mathbf X \mathbf Z, \textrm{diag}(\mathbf Z) = \mathbf 0,
\end{align}
where $\| \mathbf Z \|_1 = \sum_i \sum_j | Z_{ij} |$ (the $\ell_1$ norm) is used as a surrogate for the $\ell_0$ norm and the diagonal constraint prevents the data point from being represented by itself.


Solving such an objective is only useful when the data $\mathbf X$ is known to be noise free. As previously mentioned this is extremely unlikely in practice. SSC and most other subspace clustering methods assume the following data generation model
\begin{align}
\mathbf{x}_i = \mathbf{a}_i + \mathbf{n}_i.
\end{align}
where $\mathbf a_i$ is the latent original data vector and $\mathbf n_i$ is Gaussian noise. Fortunately it is not necessary to recover the original data to perform subspace clustering. To overcome this one can extend the objective to take noise into consideration by relaxing the constraint 
\begin{align}
\min_{\mathbf Z} \lambda \| \mathbf Z \|_1 + \frac{1}{2} \| \mathbf{X - X Z} \|_F^2 \quad \text{s.t.}\ \textrm{diag}(\mathbf Z) = \mathbf 0.
\label{SSCRelaxedObjective}
\end{align}
Alternatively one can instead solve an exact variant of the objective by incorporating a fitting error term $\mathbf E$ i.e.\  
\begin{align}
\min_{\mathbf Z} \lambda \| \mathbf Z \|_1 + \frac{1}{2} \| \mathbf E \|_F^2 \quad \text{s.t.}\ \mathbf X = \mathbf X \mathbf Z + \mathbf E, \textrm{diag}(\mathbf Z) = \mathbf 0
\label{SSCExactObjective}
\end{align}
where $\lambda$ is used to control the sparsity of $\mathbf Z$. Implementation details for relaxed and exact SSC can be found in Appendices \ref{Appendix:SSC_Relaxed} and \ref{Appendix:SSC_Exact} respectively.

To obtain the final subspace labels the reconstruction coefficients in $\mathbf Z$ are given a secondary interpretation as the affinity or similarity between the data points. Spectral clustering is applied to $\mathbf Z$. Typically N-CUT \citep{shi2000normalized} is used as it produces the most accurate segmentation even for poorly constructed affinity matrices and is relatively fast.

While SSC has promising theoretical guarantees \citep{elhamifar2012sparse} and has shown great performance for small evaluation datasets it is not widely applied in practice. This is due to the following:
\begin{itemize}
\item \BigO{N^2} memory requirements;
\item \BigO{N^2} flop (floating point operations) requirements.
\end{itemize}
The first is easily understood as $\mathbf Z \in \mathbb R^{N \times N}$. One could contend that $\mathbf Z$ could be stored in a sparse format, however since the support of $\mathbf Z$ is unknown and varies between iterations of the SSC algorithm this approach would introduce significant overhead. Similarly the high flop count is due to the dimensions of $\mathbf Z$, since each element must be calculated per iteration.

{\bf Our Contributions:} In this paper we propose a new algorithm called k-SSC, which is designed for big-data applications. k-SSC dramatically reduces the memory requirements and computation time compared with pre-existing algorithms. Furthermore the conditions of correct subspace identification are provided through theoretical analysis. The rest of the paper is structured as follows: In Section \ref{Section:background}, we further discuss the subspace clustering problem and provide an overview of related work. Section \ref{Section:kssc} is dedicated to discussing the motivation for k-SSC along with its operation and a brief sketch of the theoretical clustering analysis. Detailed proofs for the theoretical analysis are provided in Appendix \ref{Appendix:Analysis}. We follow this with an optimisation scheme and complexity analysis. Next in Section \ref{Section:Synthetic} we provide empirical analysis of kSSC using synthetic data. We then finish with a collection of experiments on real world data in Section \ref{Section:experiments} and some concluding remarks in Section \ref{Section:conclusion}.

\section{Background and Related Work}
\label{Section:background}


The union of subspaces model is applicable to a wide variety of data and thus exploited for a large number of applications. Examples include identifying individual rigidly moving objects in video \citep{tomasi1992shape, costeira1998multibody, kanatani2002motion, jacquet2013articulated}, identifying face images of a subject under varying illumination \citep{basri2003lambertian, georghiades2001few}, image compression \citep{hong2006multiscale}, image classification \citep{zhang2013learning, DBLP:conf/dicta/BullG12}, feature extraction \citep{liu2012fixed, liu2011latent}, image segmentation \citep{yang2008unsupervised, cheng2011multi}, segmentation of human activities \citep{zhu2014complex}, temporal video segmentation \citep{tierney2014subspace, vidal2005generalized} and segmentation of hyperspectral mineral data \citep{guo2015low,tierney2014subspace}

This huge range of applications for subspace clustering has spurred the development of subspace clustering algorithms. Early algebraic methods such as Generalised Principal Component Analysis (GPCA) \citep{vidal2005generalized, ma2008estimation} suffered from sensitivity to noise and increasing computational complexity as the number and size of subspaces increases \citep{elhamifar2012sparse}. A number of statistical methods have been developed such as Mixtures of Probabilistic PCA (MPPCA) \citep{tipping1999mixtures} and Multi-Stage Learning (MSL) \citep{sugaya2004geometric} that make Gaussian assumptions about the distribution of the data in the subspace. However such approaches rely on good initialisation and the dimensions of the subspaces must be known before hand.

More recently spectral methods have come to dominate subspace clustering literature as they offer a vast improvement over the previously mentioned methods. First they do not increase in complexity with the number or dimension of subspaces, second they are robust to noise or outliers and finally they provide a simple to understand work pipeline that is easily adapted and modified. Spectral methods consist of two stages: learning a similarity matrix for the data then assignment of class labels through segmentation of the similarity matrix. We have already introduced SSC, the forerunner of spectral subspace clustering. Other spectral methods adopt the general SSC objective but impose different structural constraints over the similarity matrix $\mathbf Z$ instead of the $\ell_1$ norm. For example Low-Rank Representation (LRR) \citep{6180173} imposes a rank penalty to obtain a more globally consistent $\mathbf Z$. Or different penalties are used to minimise fitting error depending upon the type of expected noise and whether or not outliers are likely \citep{xi2012constructing, Vidal201447}. Further regularisation can be applied to incorporate prior knowledge such as the spatial structure or sequence of the data \citep{guo2015low,tierney2014subspace,guo2014spatial}.

Another factor that has spurred on research of spectral methods is that they can be guaranteed to correctly segment the subspaces. For example it was shown by \cite{soltanolkotabi2014robust} that correct subspace identification is guaranteed for SSC provided that data driven regularisation for each $\lambda_i$ (column wise splitting of $\mathbf Z$) is used. This of course is subject to further conditions such as a minimum distance between subspaces,   sufficient sampling of points from each subspace and the noise level in the data. For LRR the requirements for guaranteed success are much stricter, please see \citep{6180173} for full details.

However despite the aforementioned strengths of spectral clustering algorithms they suffer from huge memory or computational requirements that prevent them from being applied to even modestly sized datasets. In the academic literature this problem is generally ignored as the data used for analysis is very limited in the number of data points.  In light of these issues there has been considerable interest in developing tractable subspace clustering algorithms. Unfortunately they either lack theoretical justification or suffer dramatically in terms of clustering accuracy in practice.

SSC approximation methods can be divided into two classes: inductive and heuristic. Inductive methods perform SSC or learn the similarity matrix on a small subset of the data. This full structure of the similarity matrix or labels is then obtained by inductive transfer from the subset. Heuristic methods abandon SSC entirely and directly assign class labels by greedy selection of nearest neighbours based on a defined metric.

Scalable SSC (SSSC) \citep{peng2013scalable} was the first attempt to resolve computational issues. As an inductive method it first selects some candidate samples from the data and performs SSC on these samples. Then the remaining samples are assigned to clusters based on their fit into the clusters formed by the training candidates. This approach has considerable issues. First the candidate samples are selected by uniform random sampling. This does not guarantee that every cluster will be accounted for in the candidate set. Second there must be enough candidate samples for the candidate clusters to generalise to the remaining samples. Correctly choosing the number of samples is a difficult task.

Arguably the most prominent Heuristic method is Orthogonal Matching Pursuit (OMP), which has been long used as a greedy sparse approximation method \citep{tropp2007signal}. For each data point a residual vector is set as the data point. Then the nearest neighbour to the residual is found and then the residual is updated by a projection of the data point onto the span formed by the currently picked up neighbours. This is repeated until the number of neighbours is reached or the norm of the residual is small enough. OMP is also known by other names such as Greedy Feature Selection (GFS) \citep{dyer2013greedy} and is a constant well that researches draw from \citep{you2016scalable}.

Although OMP is advertised as being a fast approximate method however in practice we do not find this to be the case.
First the nearest neighbour search is performed for every iteration. Second the computational and memory requirements of a Naive implementation tend to increase dramatically as $D$ increases due to the need to create a $D \times D$ matrix in each iteration for every data point. This, in some cases, makes it just as intractable as SSC. Third the Naive implementation requires successive computation of the SVD of the span matrix at each iteration. The second and third points have fortunately been mitigated by improvements such as Rank-1 updating scheme of Moore-Penrose Inverse \citep{petersen2008matrix}, factorisation approaches such as QR and Cholesky decomposition \citep{sturm2012comparison} or more esoteric methods \citep{tropp2007signal, yaghoobi2015fast, mailhe2011fast}.

However these accelerations to the OMP algorithm are relatively meanininglyess as we find the reported accuracy results of methods using OMP to be dubious. As shown in Section \ref{Section:experiments} we find that GFS (OMP) performs poorly in terms of clustering accuracy. We stress that we were unable to reproduce the promising results of OMP based methods as claimed by earlier works.

OMP has inspired other methods such as Greedy Subspace Clustering (GSC) \citep{park2014greedy} and ORGEN \citep{you2016oracle}. GSC differs from OMP in neighbour selection. Each neighbour is selected by finding the data point which has the largest norm of projection onto the span formed by the current neighbours. Although at first glance GSC appears to be simple and thus likely to scale well w.r.t.\ $N$, the projection step is quite computationally intensive and just like OMP the nearest neighbour search is performed in each iteration.

ORGEN extends the SSC model to the Elastic-Net model. That is, it uses the $\ell_2$ norm in tandem with the $\ell_1$ norm. From an initialisation point of some neighbours of each $\mathbf x_i$ it solves the Elastic-Net objective then determines an ``oracle point'', which is the residual from fitting the coefficients from the Elastic-Net procedure to the model. This oracle point is then used to find potential new neighbours. The procedure terminates when no new neighbours are added at the end of an iteration. ORGEN suffers from a number of problems. First it is highly initialisation dependant. The authors suggest performing $\ell_2$ sub problem and choosing the largest valued elements as the initialisation pool for each $\mathbf x_i$. This can be slow as when $D$ is large and the $\ell_2$ problem lacks rigorous guarantees of successful subspace identification. Second the repeated computation of elastic net is problematic when the active set grows large. This is a very real concern as termination only occurs when the active set stops growing. The active set could grow to the full size of the data set. Third the claim of improved running time is not evident. The authors show running times for single $\mathbf x_i$ instead of the whole data $\mathbf X$ and do not compare to different approaches such as \citep{peng2013scalable} or \citep{heckel2013robust}.

Heuristic methods are often incredibly simple. For example Robust Subspace Clustering via Thresholding (TSC) \citep{heckel2013robust} essentially performs nearest neighbour based spectral clustering. For each point the nearest neighbours are found and the affinity matrix is constructed using exponential inner product between each of the neighbours.

\section{Efficient Sparse Subspace Clustering}
\label{Section:kssc}


Our contribution to subspace clustering is inspired by the sparsity of SSC. It has been shown repeatedly that each data point can be reconstructed by only $d_i$ (the dimensionality of the underlying subspace) other data points from its corresponding subspace \citep{elhamifar2012sparse, soltanolkotabi2014robust, 6560026}. This is the basis of SSC's operation. By finding the sparsest representation one will be left with the minimum support to represent a data point $\mathbf x_i$, corresponding to data points in the same subspace. Therefore it is clear that blindly considering every data point as a candidate for reconstruction is very wasteful since only a relative few points will be left as support. Furthermore the process is very intensive in terms of memory requirements. The algorithm for solving the exact variant of SSC (see Appendix \ref{Appendix:SSC_Exact}) requires \BigO{N^2} FLOPs per iteration and the storage of \BigO{N^2} floats over the algorithm's entire operation w.r.t.\ $\mathbf Z$.

Therefore if we can safely prune a vast majority of data points as candidates for another data points reconstruction then we can massively reduce computational and memory load. In other words this means that we would only solve for a small subset of the entries of $\mathbf Z$ rather than the entire matrix. To this end we propose kSSC, in which we limit each data point to be represented by at most $k$ other data points. Thus the relaxed objective function for kSSC is
\begin{align}
\min_{\mathbf Z} \lambda \| \mathbf Z \|_1 + \frac{1}{2} \sum_i^N \| \mathbf x_i - \sum_{j \in \Omega_i} \mathbf x_j Z_{ji} \|_F^2 \quad \text{s.t.}\  \textrm{diag}(\mathbf Z) = \mathbf 0.
\label{kssc_objective_relaxed}
\end{align}
where $\Omega_i$ is the set of data points to use for reconstruction of data point $i$. Under this objective we can reduce both the memory and FLOP requirements to \BigO{kN} w.r.t.\ $\mathbf Z$, which when $k \ll N$ provides massive savings. 

\begin{figure*}
\centering
\includegraphics[height=0.3\textwidth]{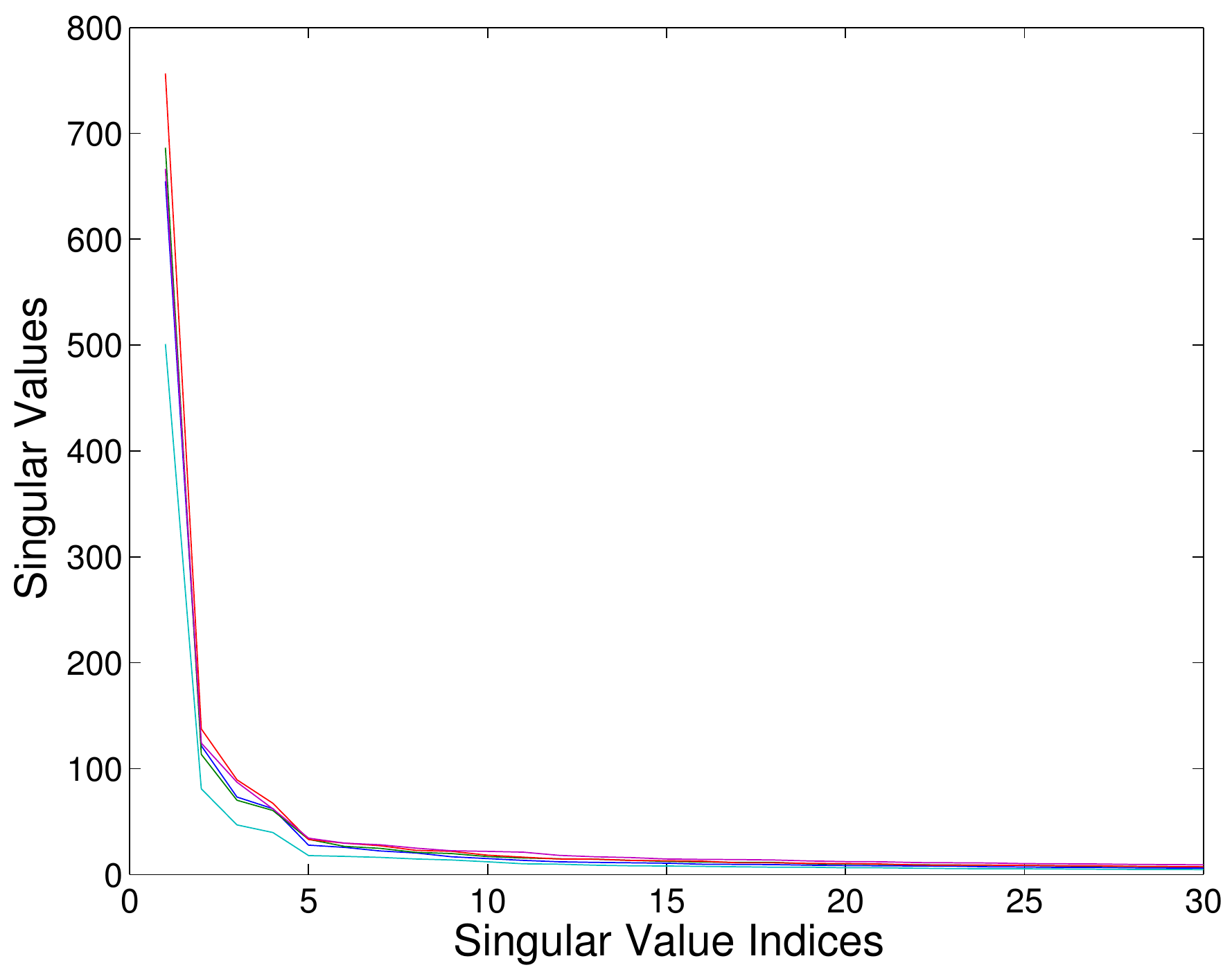}
\includegraphics[height=0.3\textwidth]{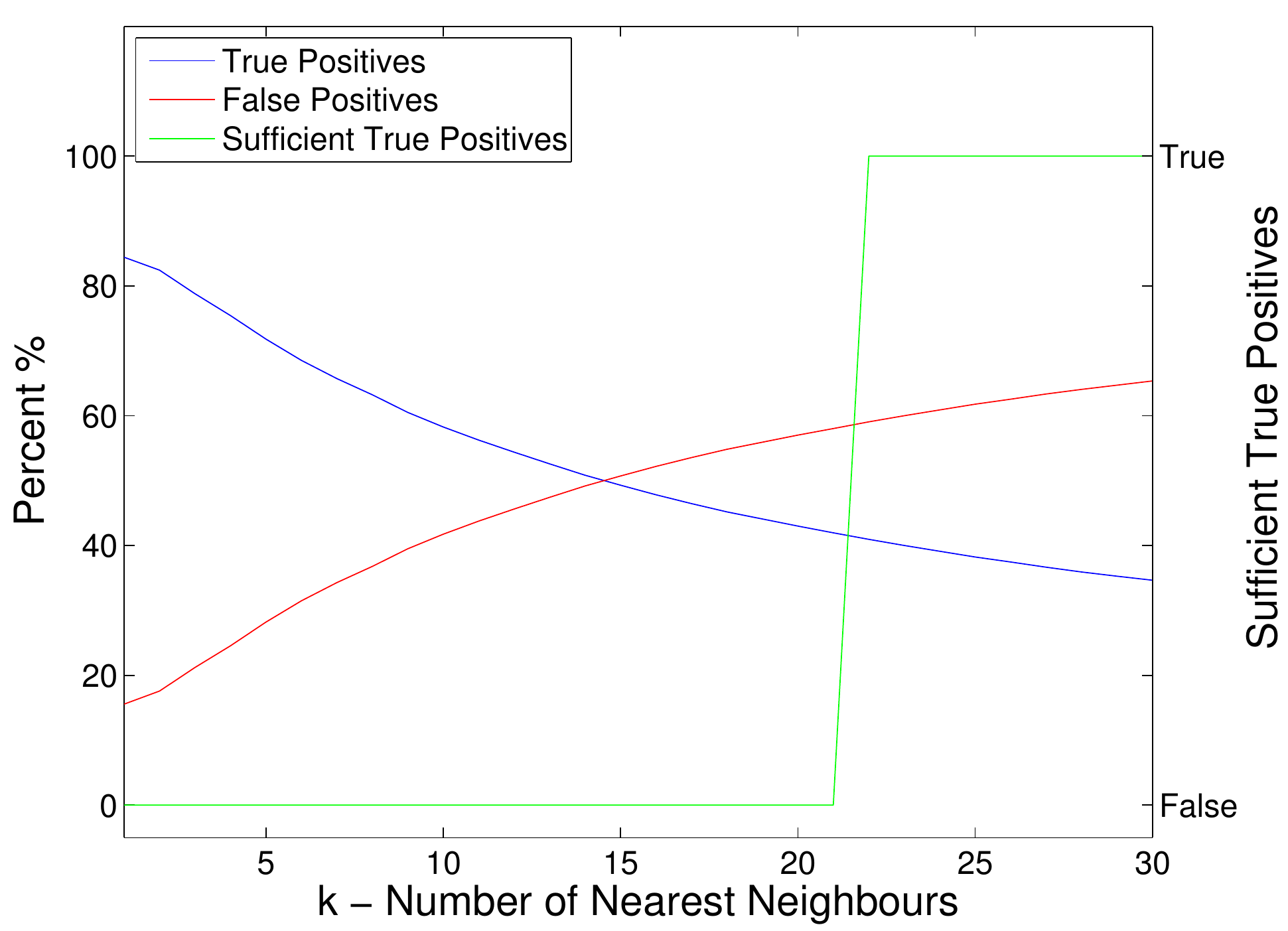}
\caption{Left: Singular values of several faces from the Extended Yale B dataset. Right: Average percentage of true positive and false positive nearest neighbour selection from the entire Extended Yale B dataset and correspondingly, in green, a plot of when the sufficient true positives are reached on average.}
\label{Figure:yale_observation}
\end{figure*}

\begin{figure*}
\centering
\includegraphics[height=0.3\textwidth]{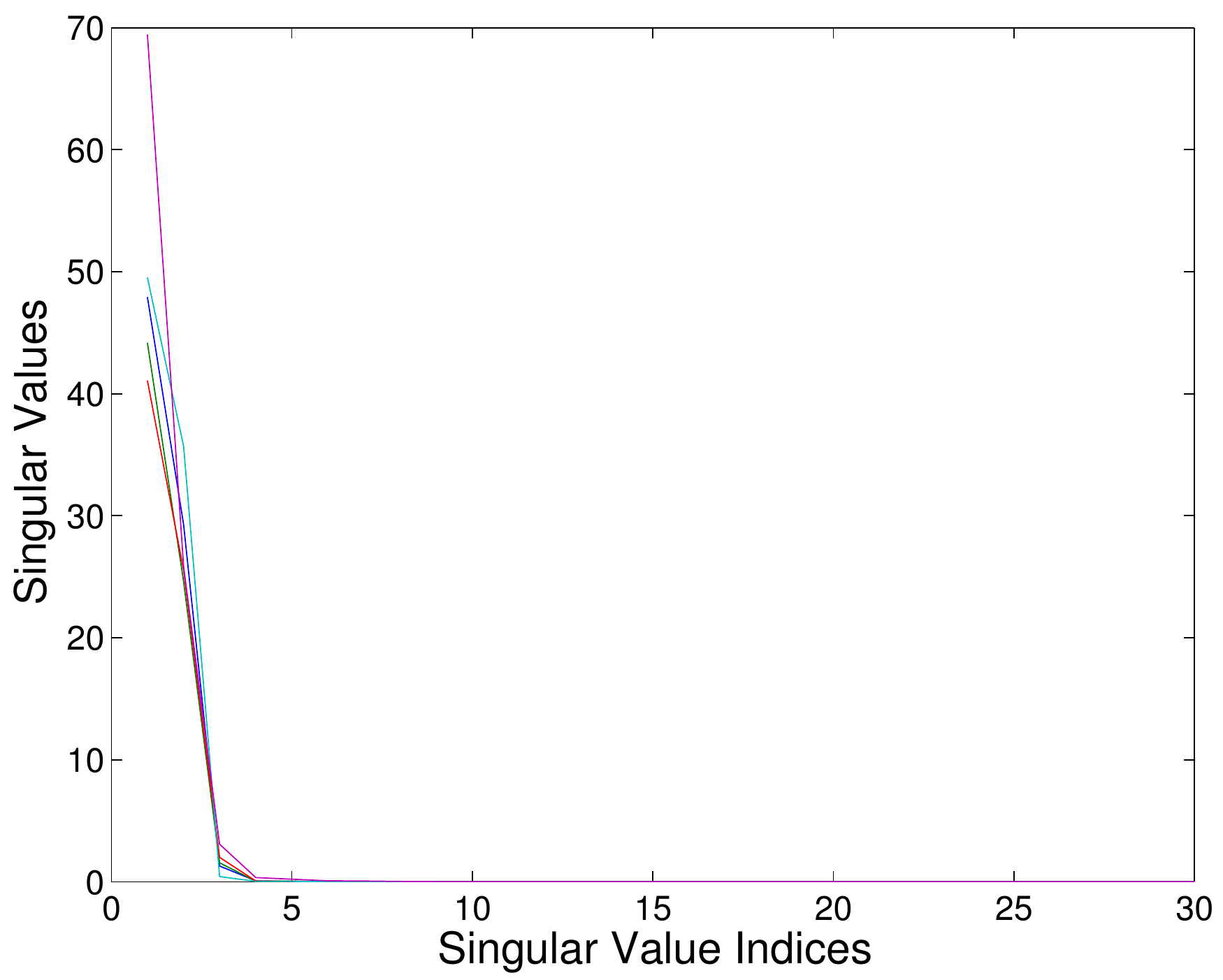}
\includegraphics[height=0.3\textwidth]{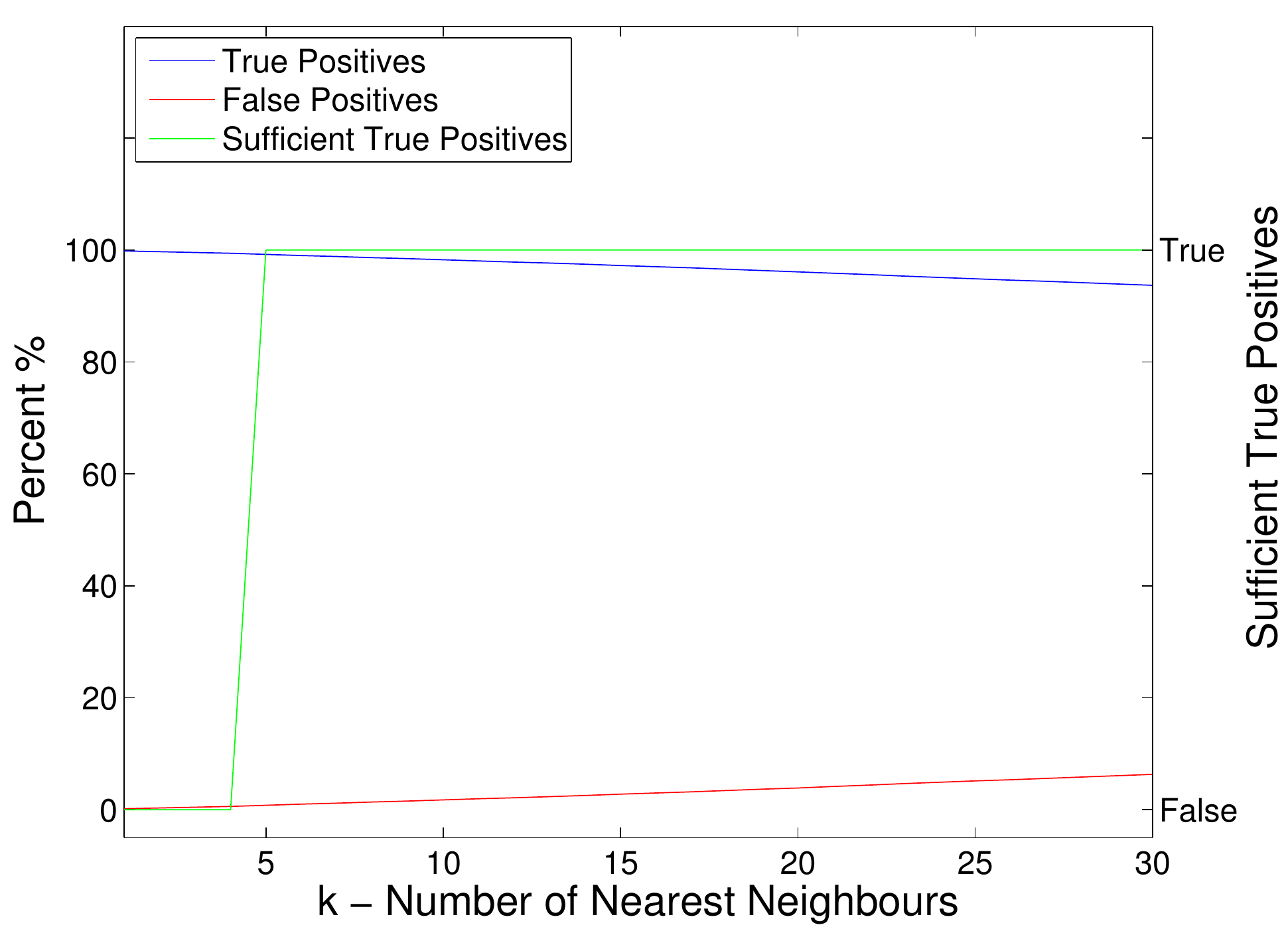}
\caption{Left: Singular values of several motions from the Hopkins 155 Motion Dataset. Right: Average percentage of true positive and false positive nearest neighbour selection from the checkerboard and traffic sequences in the Hopkins 155 Motion Dataset dataset and correspondingly, in green, a plot of when the sufficient true positives are reached on average.}
\label{Figure:hopkins_observation}
\end{figure*}

Evidently the success of kSSC relies heavily upon both the size of $k$ and the scheme that is used to select $\Omega_i$. First one should always choose $k \geq d_i$ since each data point needs at a minimum $d_i$ other points for reconstruction. The dimension of a subspace can be roughly estimated by analysing the singular values of the matrix of samples from a subspace. For example in Figure \ref{Figure:yale_observation} we show singular values for multiple subspaces (a single subspace corresponds to a single subject or face) from the Extended Yale B dataset. The point at which the singular values begin to trail off reveals the underlying subspace dimension $d_i$, which in this case is $9$ \citep{elhamifar2012sparse}. Therefore in that case we must set $k \geq 9$. Similarly in Figure \ref{Figure:hopkins_observation} we perform the same analysis on the motion segmentation dataset and find that the subspace dimension is $4$. Since in almost all cases $d_i \ll N$ and thus $k \ll N$ the computational and memory requirements of kSSC will be much lower than SSC. However even in cases where there is no ground truth or sample data available one can set a large, conservative value for $k$ with little impact on overall performance. For example in Figure \ref{Figure:requirements} we show that increasing $k$ from $100$ to $1000$ barely impacts FLOP count or memory requirements relative to the original requirements of SSC. These values were calculated using Table \ref{Table:requirements}.

Second one must choose $\Omega_i$ such that it contains enough data points from $\mathbf x_i$'s subspace. Uniformly random sampling to choose $k$ neighbours is a poor choice since the selected neighbours may not belong to the same subspace. Recent works such as \citep{dyer2013greedy, park2014greedy, 6560026} have demonstrated that even in noisey cases or cases of subspace intersection that the points closest to each $\mathbf x_i$ in the ambient space usually correspond to the most strongly connected data points in $\mathbf Z$ i.e.\ data points from the same subspace. We come to the same conclusions in Appendix \ref{Appendix:Analysis} where we prove that in both noisey and noise free cases we are able to correctly select points from the subspace using k nearest neighbours. We repeat our two central theorems here for the reader. First in the case of noise free data:
\begin{thm}[Recalled from Theorem \ref{thm:knn}]
Let $d_m$ be the minimum dimensionality of all the subspaces. Given the conditions in Theorem \ref{thm:kconcentration}, if $\mathcal A_{\ell,1}\le\frac{\sqrt{d_m}(1-\epsilon^2/2)}{2\sqrt{d(t\log N_{\ell}+t^2)}}$, then the samples selected for any sample $\vc x_1$ from subspace $\mathcal S_1$ by using $k$NN (with $k=k_0/C$) contains no samples from other subspaces but $\mathcal S_1$ with probability at least $1-2e^{-t}-e^{-k_0}(eC)^{k_0/C}$.
\end{thm}
Second in the case of noisey data
\begin{thm}[Recalled from Theorem \ref{thm:knnnoise}]
Let $d_m$ be the minimum dimensionality of all the subspaces. Given the conditions in Theorem \ref{thm:kconcentration} and the noise model \eqref{e:noisemodel}, for a small positive $\delta$, if $\mathcal A_{\ell,1}\le\frac{\sqrt{d_m}(1-\epsilon^2/2-6\delta)}{2\sqrt{d(t\log N_{\ell}+t^2)}}$, then the samples selected for any sample $\vc y_1$ from subspace $\mathcal S_1$ by using $k$NN (with $k=k_0/C$) contains no samples from other subspaces but $\mathcal S_1$ with probability at least $1-2e^{-t}-e^{-k_0}(eC)^{k_0/C}-2\exp(1-\frac{c\delta^2}{\sigma^2}) - \frac{d\sigma^4}{\delta^2}$.
\end{thm}

Therefore we set $\Omega_i$ for $\mathbf x_i$ as its $k$ nearest neighbours from the original ambient space. However when $\mathbf X$ is subject to noise this assumption no longer holds. For this reason we recommend setting $k$ well above $d_i$ to provide sufficient head room. Furthermore we suggest increasing $k$ as the magnitude of expected noise increases, since as noise increases, so does the likelihood of false positive neighbour selection. In Figure \ref{Figure:yale_observation} and Figure \ref{Figure:hopkins_observation} we demonstrate this effect on the Yale and Rigid Motion datasets respectively. We note that the required value for $k$ to select sufficient true positives via kNN exceeds $d_i$. This is due to the presence of noise and corruptions in the data and the sometimes small distance between subspaces, particularly for the Extended Yale B dataset. Although still extremely small relative to $N$.

In summary we propose to eliminate the calculation of redundant elements of $\mathbf Z$ by computing only $k$ rather than $N$ coefficients for each $\mathbf x_i$. An overview of the entire method can be found in Algorithm \ref{alg_final}. Subspace identification accuracy can be exactly maintained from SSC provided that the following conditions are met:
\begin{itemize}
\item $k$ is equal to or greater than $\max(d_i)$
\item the elements of $\Omega_i$ are nearest neighbours of $\mathbf x_i$
\end{itemize}
These conditions are sufficient but not necessary. In some cases clustering accuracy could be maintained when $k$ is less than $\max(d_i)$ or different filtering method is used. However when these conditions are met kSSC ensures that  SSC's guarantee of correct subspace identification and robustness to noise is preserved since kNN is guaranteed to correctly identify neighbours (see Appendix \ref{Appendix:Analysis}). Furthermore kSSC is easily solved in parallel as $\Omega_i$ and each column of $\mathbf Z$ is independent.

\begin{algorithm}[!t]
\caption{kSSC}
\label{alg_final}
\begin{algorithmic}[1]

\REQUIRE $\mathbf X^{D \times N}$ - observed data, $k$ - number of neighbours, $c$ - number of subspaces

\FOR{$i \to N$ in parallel}

\STATE Set $\Omega_i$ by kNN

\STATE Obtain coefficients $\mathbf z_i$ by solving \eqref{kssc_objective_relaxed}

\ENDFOR

\STATE Form the similarity graph $\mathbf W = |\mathbf Z| + |\mathbf Z|^T$

\STATE Apply N-Cut to $\mathbf W$ to partition the data into $c$ subspaces

\RETURN Subspaces $\{S_i\}^c_{i=1}$

\end{algorithmic}
\end{algorithm}

\subsection{Optimisation}
\label{sec:optimisation}

To solve \eqref{kssc_objective_relaxed} we use FISTA (Fast Iterative Shrinkage-Thresholding Algorithm) \citep{BeckTeboulle2009, ji2009accelerated}. FISTA is an accelerated gradient descent scheme for solving objective functions containing a smooth part and non-smooth part as is the case with \eqref{kssc_objective_relaxed}. One of the key abilities of FISTA is that it guarantees a convergence rate of $\BigO{\frac{1}{\mathtt t^2}}$ where $\mathtt t$ is the iteration counter. This is achieved by dynamically setting the rate of descent parameter (Lipschitz constant) and using the two previous iteration points to accelerate the gradient descent. Furthermore FISTA provides the aforementioned ability with minimal computational and memory overhead. Each iteration of FISTA only requires solving a closed form proximity problem which in the case of $\ell_1$ minimisation can be solved at an element wise level. This allows us to resolve the selective fitting term of \eqref{kssc_objective_relaxed} since we can enforce it by ignoring the elements of $\mathbf Z$ that are outside of $\Omega$.

We begin by re-writing, with some abuse of notation, the original objective \eqref{kssc_objective_relaxed} for a single column of $\mathbf Z$
\begin{align}
\min_{\mathbf z_i} L = \lambda \| \mathbf z_i \|_1 + \frac{1}{2}\| \mathbf x_i - \mathbf X_i \mathbf z_i \|_2^2
\end{align}
where $\mathbf z_i = \mathbf z_{{\Omega_i}i}$ i.e.\ the vector of rows $\Omega_i$ and column $i$ of $\mathbf Z$ and $\mathbf X_i = \mathbf X_{(:,\Omega_i)}$ i.e.\ the matrix formed from the columns of $\mathbf X$ indexed by $\Omega_i$. Note that we have removed the constraint $\textrm{diag}(\mathbf Z) = \mathbf 0$ since we enforce it by ensuring that no diagonal entries are present in each $\Omega_i$.

At each iteration in the FISTA scheme one must solve the $\ell_1$ proximal linearised form of $L$. Denote the linearisation of $L$ at point $\mathbf z_i^{\mathtt t}$
\begin{align}
\min_{\mathbf z_i} \widetilde{L}_{\rho}(\mathbf{z_i, z_i^{\mathtt t}}) =  \lambda\| \mathbf z_i \|_{1} + \frac{\rho}{2} \| \mathbf z_i - ( \mathbf z_i^{\mathtt t} - \frac{1}{\rho} \partial F(\mathbf z_i^{\mathtt t})) \|_2^2,
\label{kssc_linearisation}
\end{align}
where $F = \frac{1}{2}\|\mathbf x_i - \mathbf X_i \mathbf z_i \|^2_2$ and correspondingly $\partial F = - \mathbf X_i^T (\mathbf x_i - \mathbf X_i \mathbf z_i^{\mathtt t})$. The solution to \eqref{kssc_linearisation} is given by the closed-form $\ell_1$ shrinkage function $\mathcal S_{\tau}$ as follows
\begin{align}
\mathcal S_{\frac{\lambda}{\rho}}(\mathbf z_i^{\mathtt t}) = \textrm{sign}(\mathbf z_i^{\mathtt t} - \frac{1}{\rho} \partial F(\mathbf z_i^{\mathtt t})) \max(| \mathbf z_i^{\mathtt t} - \frac{1}{\rho} \partial F(\mathbf z_i^{\mathtt t}) | - \frac{\lambda}{\rho}, 0).
\end{align}
We refer readers to \citep{bach2011convex, liu2010efficient} for further details. The full algorithm is outlined in Algorithm \ref{Algorithm:kSSC_relaxed}.  

\begin{algorithm}
\caption{Solving \eqref{kssc_objective_relaxed} via FISTA}
\begin{algorithmic}
\label{Algorithm:kSSC_relaxed}

\REQUIRE $r_i = \infty$, $\mathbf z_i = \mathbf 0$,  $\mathbf j_i = \mathbf 0$, $\alpha_i = 1$, $\lambda$, $\rho_i$, $\gamma$, $\epsilon$

\WHILE{$r_i^{\mathtt t} - r_i^{\mathtt t - 1} \geq \epsilon$ }

	\WHILE{$L(\mathcal S_{\frac{\lambda}{\rho}}(\mathbf j_i^{\mathtt t})) \geq \widetilde{L}_{\rho}(\mathcal S_{\frac{\lambda}{\rho}}(\mathbf j_i^{\mathtt t}), \mathbf j_i^{\mathtt t})$}
	
		\STATE $\rho_i = \gamma \rho_i$	
	
	\ENDWHILE

	\STATE $\mathbf z_i^{\mathtt t+1} = \mathcal S_{\frac{\lambda}{\rho}}(\mathbf j_i^{\mathtt t}))$
	\STATE $\alpha_i^{\mathtt t+1} = \frac{1 + \sqrt{1 + 4 \alpha_i^t{^2})}}{2}$
	\STATE $\mathbf j_i^{\mathtt t+1} = \mathbf z_i^{\mathtt t+1} + \left ( \frac{\alpha_i^{\mathtt t} - 1}{\alpha_i^{\mathtt t+1}} \right ) (\mathbf z_i^{\mathtt t+1} - \mathbf z_i^{\mathtt t})$
	\STATE $r_i^{\mathtt t+1} = L(\mathcal S_{\frac{\lambda}{\rho}}(\mathbf j_i^{\mathtt t}))$
	
\ENDWHILE

\end{algorithmic}
\end{algorithm}

\subsection{Segmentation}

After solving \eqref{kssc_objective_relaxed} for each $\mathbf z_i$ the next step is to form $\mathbf Z$ and use the information encoded in $\mathbf Z$ to assign each data point to a subspace. A robust approach is to use spectral clustering. The matrix $\mathbf Z$ can be interpreted as the affinity or distance matrix of an undirected graph. Element $Z_{ij}$ corresponds to the edge weight or affinity between vertices (data points) $i$ and $j$. Then we use the spectral clustering technique, Normalised Cuts (N-Cut) \citep{shi2000normalized}, to obtain final segmentation. N-Cut has been shown to be robust in subspace segmentation tasks and is considered state of the art. Since we expect $\mathbf Z$ to be sparse in most cases N-Cut should have reasonable computation time, particularly in comparison to a full $\mathbf Z$ matrix. However in cases where N-Cut is too slow one can use approximate techniques such as the Nystr\"{o}m method \citep{fowlkes2004spectral}. Spectral segmentation techniques such as N-Cut require the number of subspaces $p$ as a parameter. In this paper we assume that the number of subspaces is estimated by the user although automatic techniques exist for the estimation of $p$, see \citep{tierney2014subspace} for details.

Spectral segmentation techniques such as N-Cut require the number of subspaces $p$ as a parameter. In the case  where the number of subspaces is unknown one can use either the Eigen-gap \citep{zelnik2004self,vidal2011subspace,soltanolkotabi2014robust} or the closely related SVD-gap heuristic of \cite{6180173}. The Eigen-gap heuristic uses the eigenvalues of $\mathbf W$ to find the number subspaces. It does this by finding the largest gap between the ordered eigenvalues, the number of eigenvalues before this point is treated as the number of clusters. Let $\{\delta_i\}_{i=1}^N$ be the descending sorted eigenvalues of $\mathbf W$ such that $\delta_1 \geq \delta_2 \geq \dots \geq \delta_N$. Then $p$ can be estimated by
\[
p = \argmax_{i = 1, \dots, N-1} (\delta_i - \delta_{i+1})
\]
The SVD-gap heuristic is the same procedure with eigenvalues of $\mathbf W$ replaced with singular values. Further improvements upon the Eigen-gap heuristic have been made, see \citep{zelnik2004self} for details.

\subsection{Complexity Analysis}

\begin{table}
\centering

\begin{tabular}{c | c | c }
\bf Method & \bf FLOPs (CPU) & \bf Floats (RAM) \\
\hline
SSC Exact & $7N^2 + 4DN^2 + 4DN$ & $2N^2 + DN$\\
\hline
SSC Relaxed & $6N^2 + 4DN^2 + DN$ & $4N^2$\\
\hline
kSSC Exact & $7kN + 4kDN + 4DN$ & $2kN + DN$\\
\hline
kSSC Relaxed & $6kN + 4kDN + DN$ & $4kN$\\
\end{tabular}

\caption{Overview of FLOP requirements w.r.t\ $\mathbf Z$ per iteration of SSC and kSSC and Float requirements over entire operation of SSC and kSSC.}
\label{Table:requirements}
\end{table}

\begin{figure*}
\centering
	\subfloat[Memory requirements w.r.t. $\mathbf Z$ for SSC and kSSC. $D = 500$.]{ \includegraphics[width=0.3\textwidth]{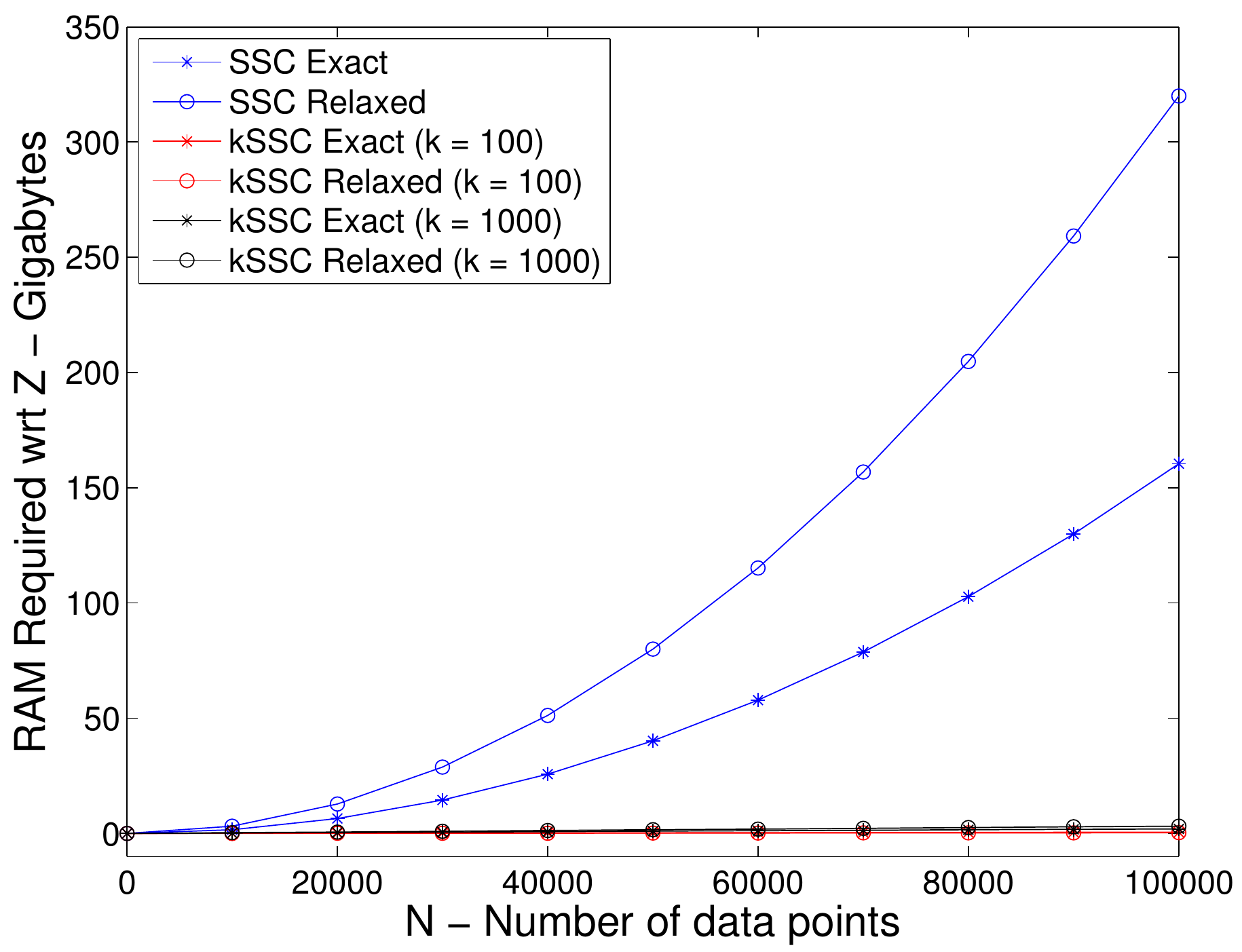}}\;
	\subfloat[Memory requirements w.r.t. $\mathbf Z$ for kSSC as $k \to N$. $D = 500$.]{\includegraphics[width=0.3\textwidth]{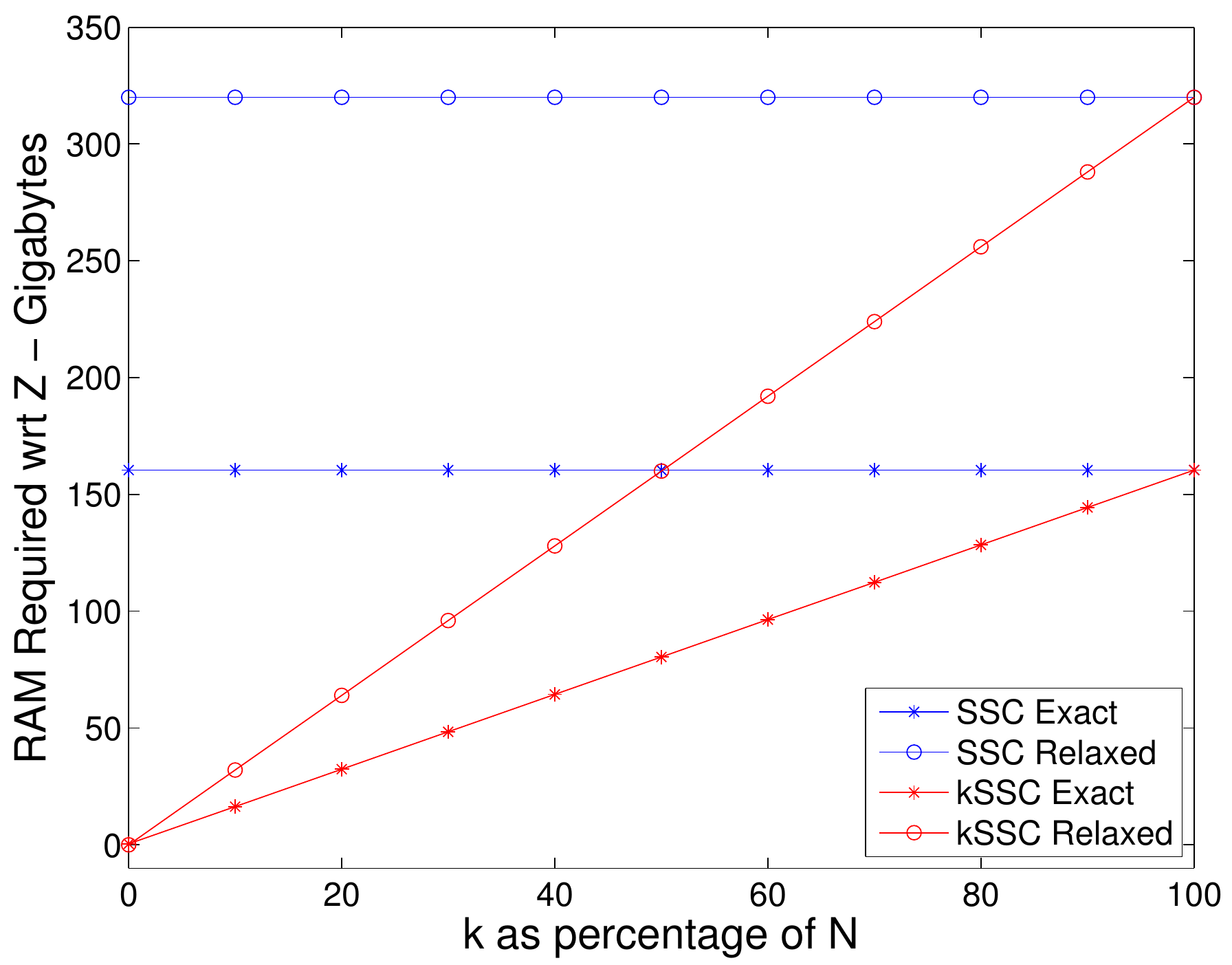}}\;
	\subfloat[FLOP requirements w.r.t. $\mathbf Z$ for a single iteration of SSC and kSSC. $D = 500$.]{\includegraphics[width=0.3\textwidth]{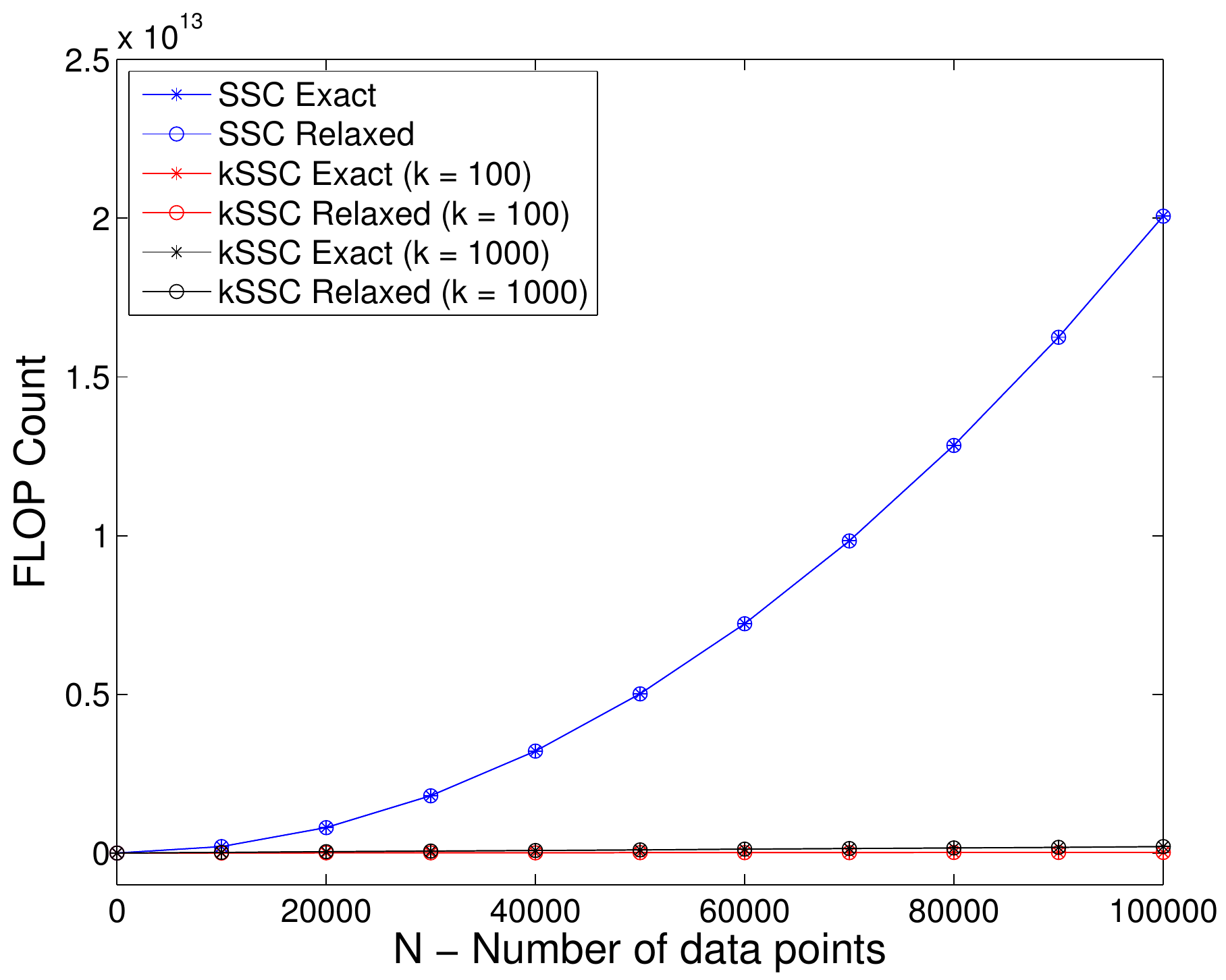}}
\caption{Visual comparison of memory and computation requirements from Table \ref{Table:requirements}. Even as $k$ grows much larger than $d_i$ it has little impact on computation time and memory.}
\label{Figure:requirements}
\end{figure*}

The complexity of kSSC only varies from SSC w.r.t.\ $\mathbf Z$ as can be seen from Algorithm \ref{alg_final} and a comparison of Algorithm \ref{Algorithm:kSSC_relaxed} and \ref{Algorithm:SSC_relaxed}. It differs in two ways. First we must find the $k$ nearest neighbours of each $\mathbf x_i$. Fortunately fast approximate methods exist for computing kNN and are freely available in packages such as FLANN \citep{muja_flann_2009}. The computation time for kNN is on the order of \BigO{N \log N} and \BigO{\log N} for preprocessing and searching respectively \citep{he2012computing,arya1998optimal,flann_pami_2014,muja_flann_2009}.

Second is the updating of $\mathbf z_i$ at each iteration. Since we are only updating $k$ entries of each column of $\mathbf Z$ instead of the full $N$ entries the FLOP count is drastically reduced. We enumerate the different FLOP counts per iteration for all columns of $\mathbf Z$ in Table \ref{Table:requirements} and visualise the dramatic improvement that kSSC offers in Figure \ref{Figure:requirements}. Similarly the amount of RAM required in the form of floats for updating $\mathbf Z$ is drastically reduced. For both FLOP and RAM the complexity is reduced from \BigO{N^2} to \BigO{N}. Note that the relaxed variant has markedly lower FLOP counts than the exact variants. This assumes one execution of the $\ell_1$ shrinkage operator per iteration. However in the case of FISTA, a single iteration may require many executions of the shrinkage operator due to the search scheme for optimal rate of descent parameter. If the rate of descent parameter $\rho$ is initialised poorly this may lead to the exact variant solved by LADMPSAP (see Appendix \ref{Appendix:kSSC_Exact}) to be quicker. However in practice we find that solving the relaxed variant by FISTA is usually faster since choosing $\rho$ is not a difficult task and can be estimated by running the solver on a small sub section of the data. Furthermore the FISTA based solver will converge much faster than the LADMPSAP solver. We provide a brief sample of running time differences in Figure \ref{Figure:runtime_wrt_N} to illustrate the difference between implementation variants. We also demonstrate the effect of varying the number of available cores for the parallel implementations in Figure \ref{Figure:runtime_wrt_cores}. We find that in the case of SSC as the number of cores increase the computation time also increases, which indicates that the performance of SSC is not as straight forward as outlined in Table \ref{Table:requirements}. In fact the performance is markedly worse than expected due to the overhead of sharing and multiple accessing of the full data matrix $\mathbf X$, which further reinforces the point that SSC does not scale well with large datasets. On the other hand kSSC benefits greatly from increasing the core count and eventually plateaus due to it's own overhead.

\begin{figure*}
\centering
\includegraphics[height=0.3\textwidth]{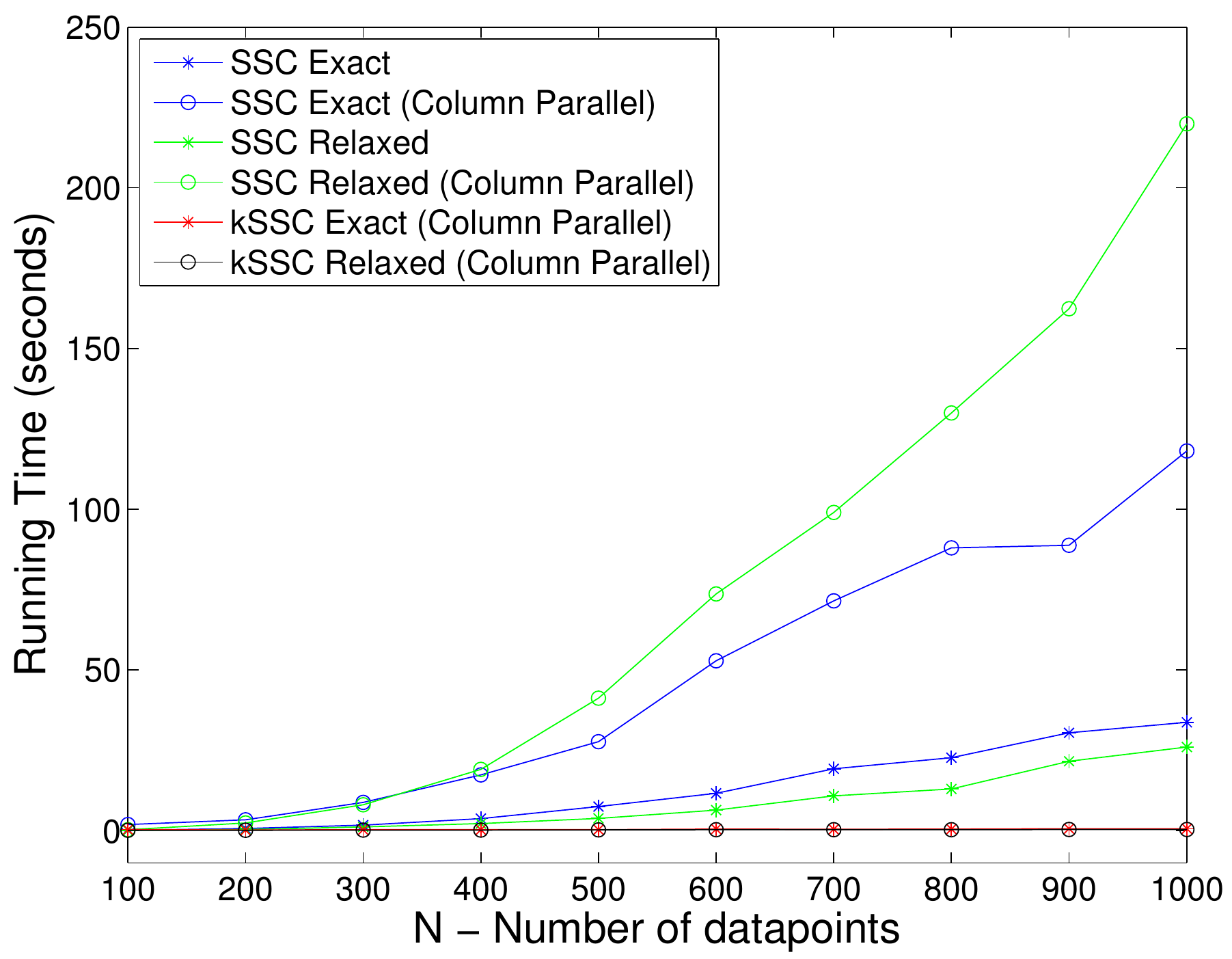}
\includegraphics[height=0.3\textwidth]{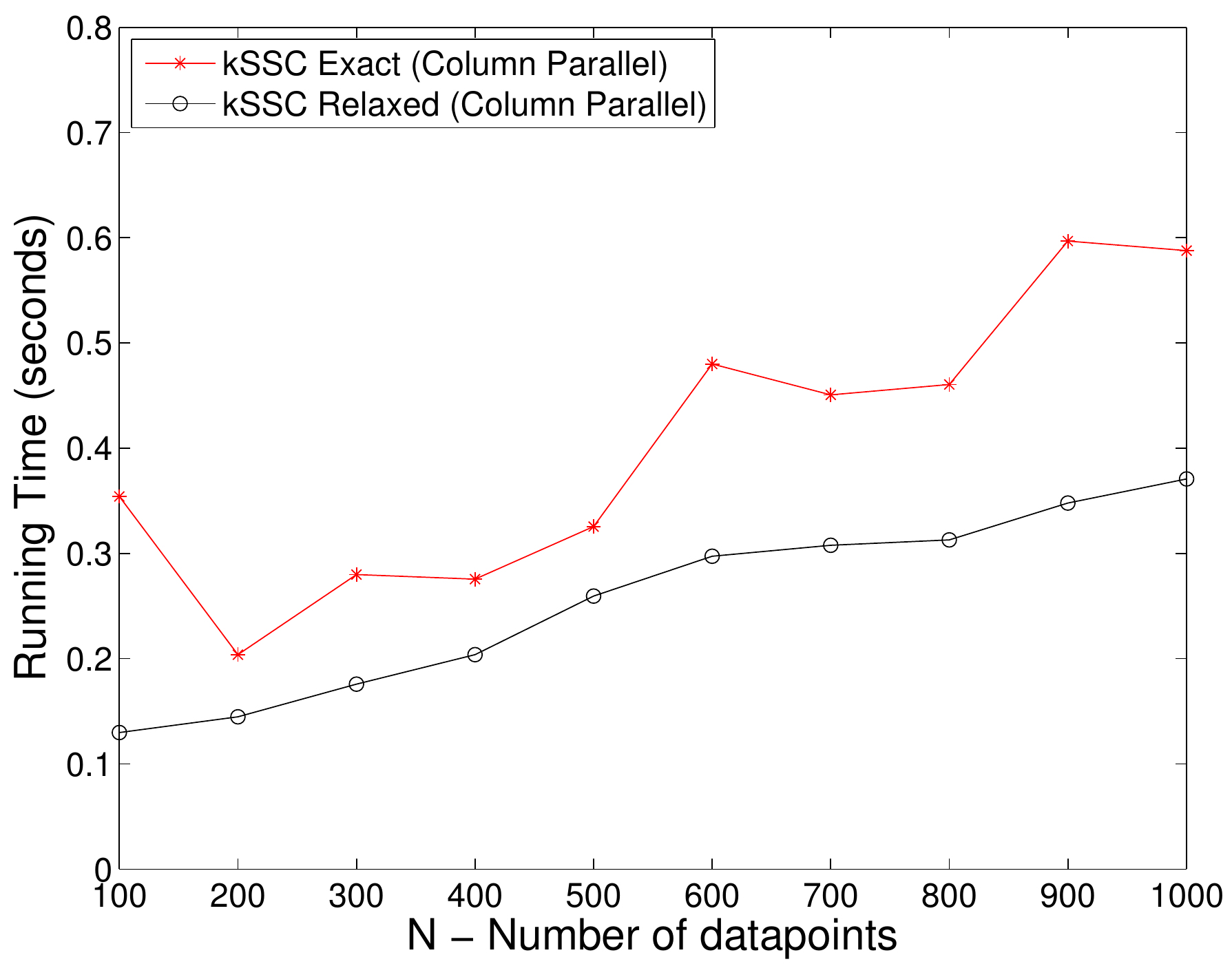}
\caption{Left: A comparison of running times with increasing $N$ between kSSC, SSC and their various implementations. Right: A zoomed comparison of running times with increasing $N$ for kSSC relaxed and exact variants (taken from the Left plot). Note that the scales are different since kSSC takes a fraction of the running time of SSC.}
\label{Figure:runtime_wrt_N}
\end{figure*}

\begin{figure*}
\centering
\includegraphics[height=0.3\textwidth]{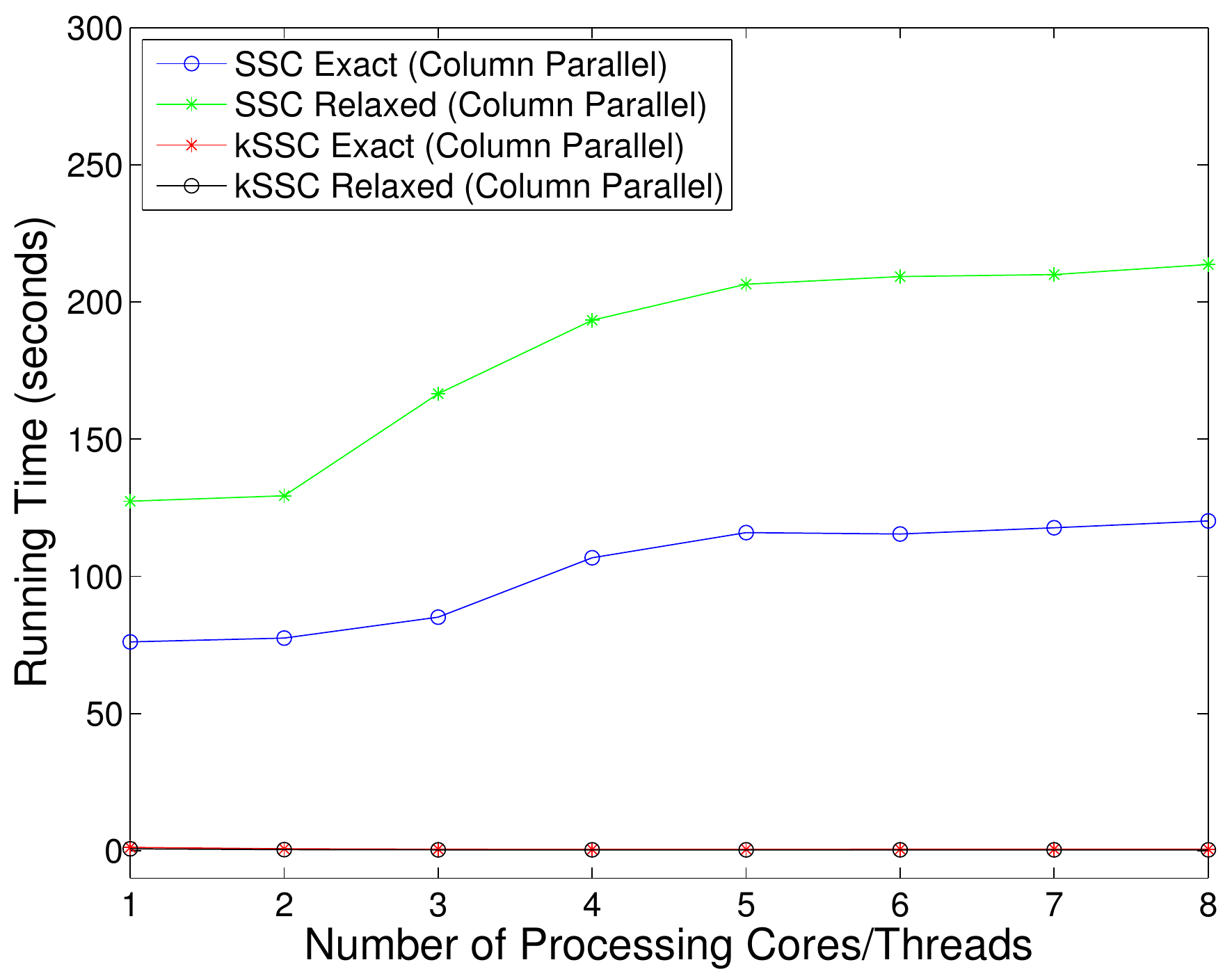}
\includegraphics[height=0.3\textwidth]{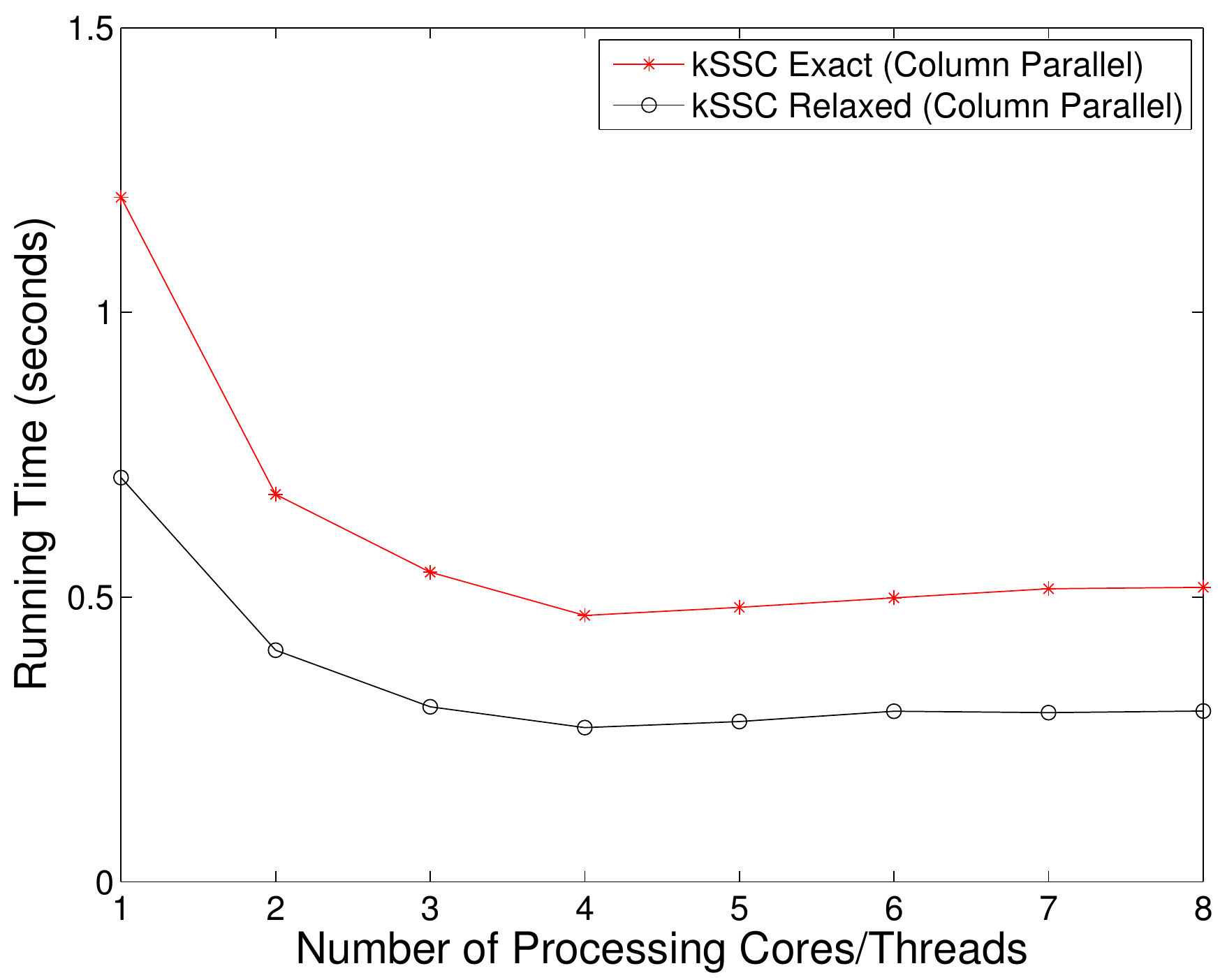}
\caption{Left: A comparison of running times with increasing threads/cores between kSSC, SSC and their various implementations. Right: A zoomed comparison of running times with increasing threads/core for kSSC relaxed and exact variants (taken from the Left plot). Note that the scales are different since kSSC takes a fraction of the running time of SSC.}
\label{Figure:runtime_wrt_cores}
\end{figure*}

\section{Preliminary Synthetic Evaluation}
\label{Section:Synthetic}

In this section we use synthetic data to experimentally evaluate our hypothesis proposed in Section \ref{Section:kssc} and the therotetical analysis in Appendix \ref{Appendix:Analysis} that kSSC can match the clustering accuracy of SSC.

In an effort to maximise transparency and repeatability, all MATLAB code and data used for these experiments and those in Section \ref{Section:experiments} can be found online at \url{https://github.com/sjtrny/kSSC}. To help evaluate consistency parameters except for $k$ were fixed for each experiment, which we further explain in the following subsections and are recorded in the code repository.

\subsection{Metrics}

Segmentation accuracy was measured using the subspace clustering error (SCE) metric \citep{elhamifar2012sparse}, which is defined as
\begin{align}
\text{SCE} = \frac{\text{num.\ of misclassified points}}{\text{total num.\ of points}} \times 100,
\end{align}
where lower subspace clustering error means greater clustering accuracy. In cases where we inject extra noise we report the level of noise using Peak Signal-to-Noise Ratio (PSNR) defined as
\begin{align}
\text{PSNR} = 10 \log_{10} \left( \frac{s^2}{\frac{1}{m n} \sum_i^{m} \sum_j^{n} ( A_{ij} - X_{ij})^2} \right)
\label{PSNR}
\end{align}
where $\mathbf{X = A + N}$, $\mathbf A$ is the original data, $\mathbf N$ is noise and $s$ is the maximum possible value of an element of $\mathbf A$. Decreasing values of PSNR indicate increasing amounts of noise.

\subsection{Effect of subspace dimension, cluster size and ambient dimension}

As noted in other works such as \citep{heckel2013robust, SoltanolkotabiCandesothers2012} the ratio of the subspace dimension $d_i$ to the number of points in each cluster $N_i$ can play a dramatic role in the clustering accuracy of SSC. However these works also ignore the role of ambient dimension $D$. In this section we demonstrate the relationship between all three variables.

We generate $5$ subspaces and vary their dimension $d_i$ from $3$ to $30$ and $N_i$ from $15$ to $150$. Each subspace is created using random orthonormal vectors as the basis with uniform random coefficients. For each pair of $d_i$ and $N_i$ we take the mean of the SCE over $50$ problem instances. We repeat this again for $3$ instances with $D$ set to $30, 50$ and $100$. For this experiment we set $k = \frac{N_i}{2}$ or $k = 1.5D$, whichever is smaller. The results shown in Figure \ref{Figure:synthetic_dvsn} that kSSC can match the performance of SSC even when $k \ll d_i$.

\begin{figure}[]
\centering
	\subfloat[SSC, $D = 30$]{
	\includegraphics[width=0.25\textwidth]{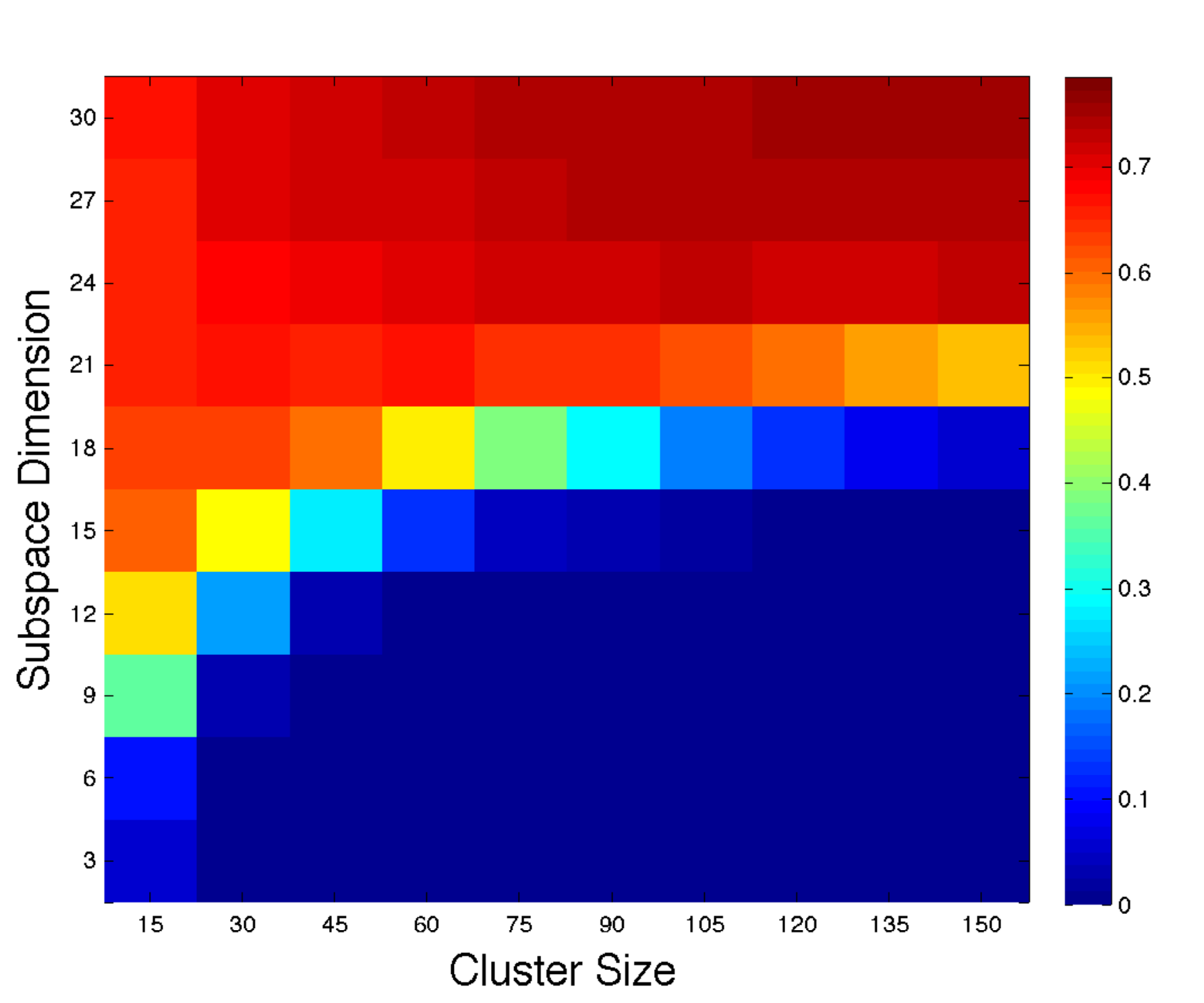}}
	\subfloat[SSC, $D = 50$]{
	\includegraphics[width=0.25\textwidth]{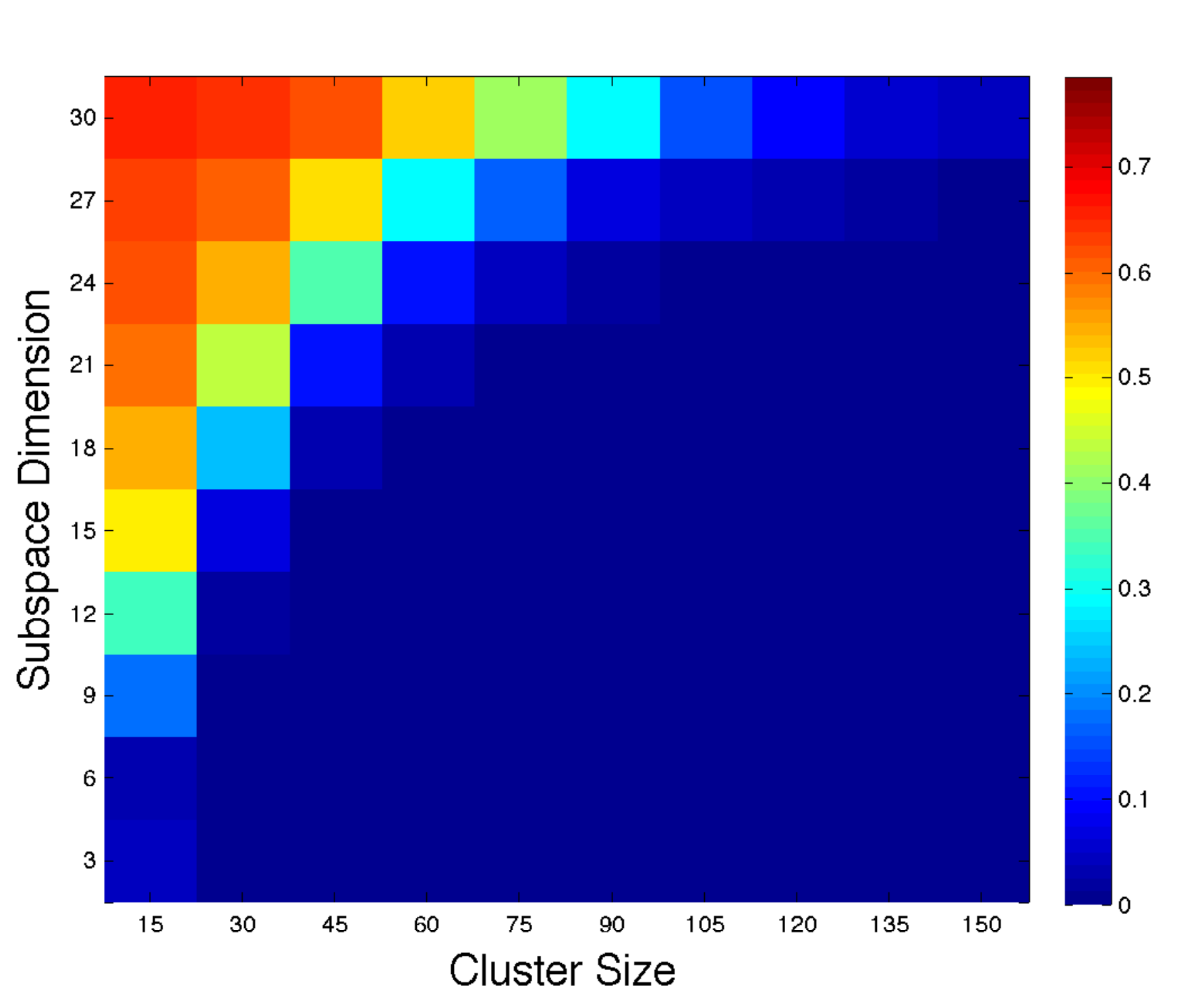}}
	\subfloat[SSC, $D = 100$]{
	\includegraphics[width=0.25\textwidth]{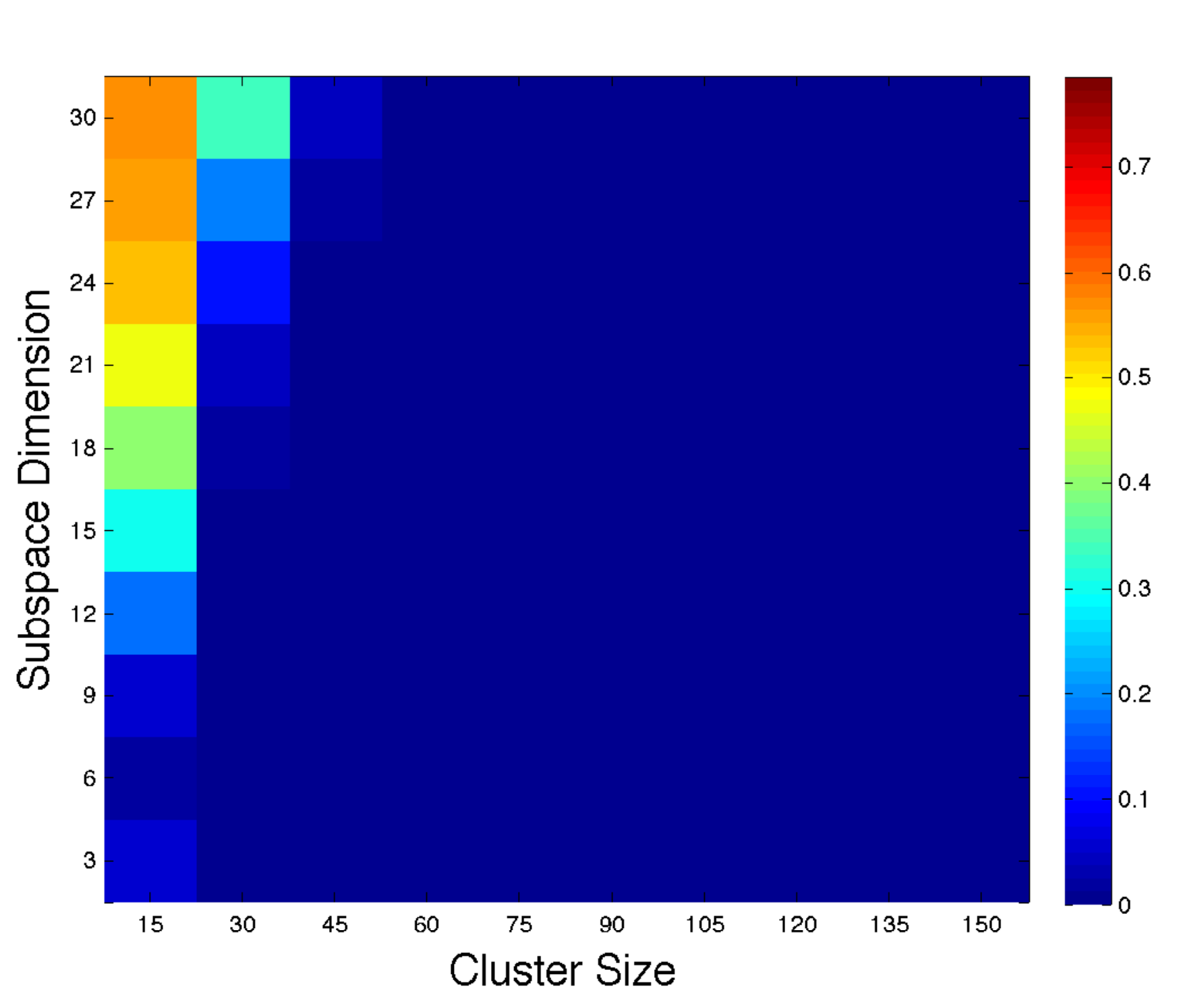}}\\
	\subfloat[kSSC, $D = 30$]{
	\includegraphics[width=0.25\textwidth]{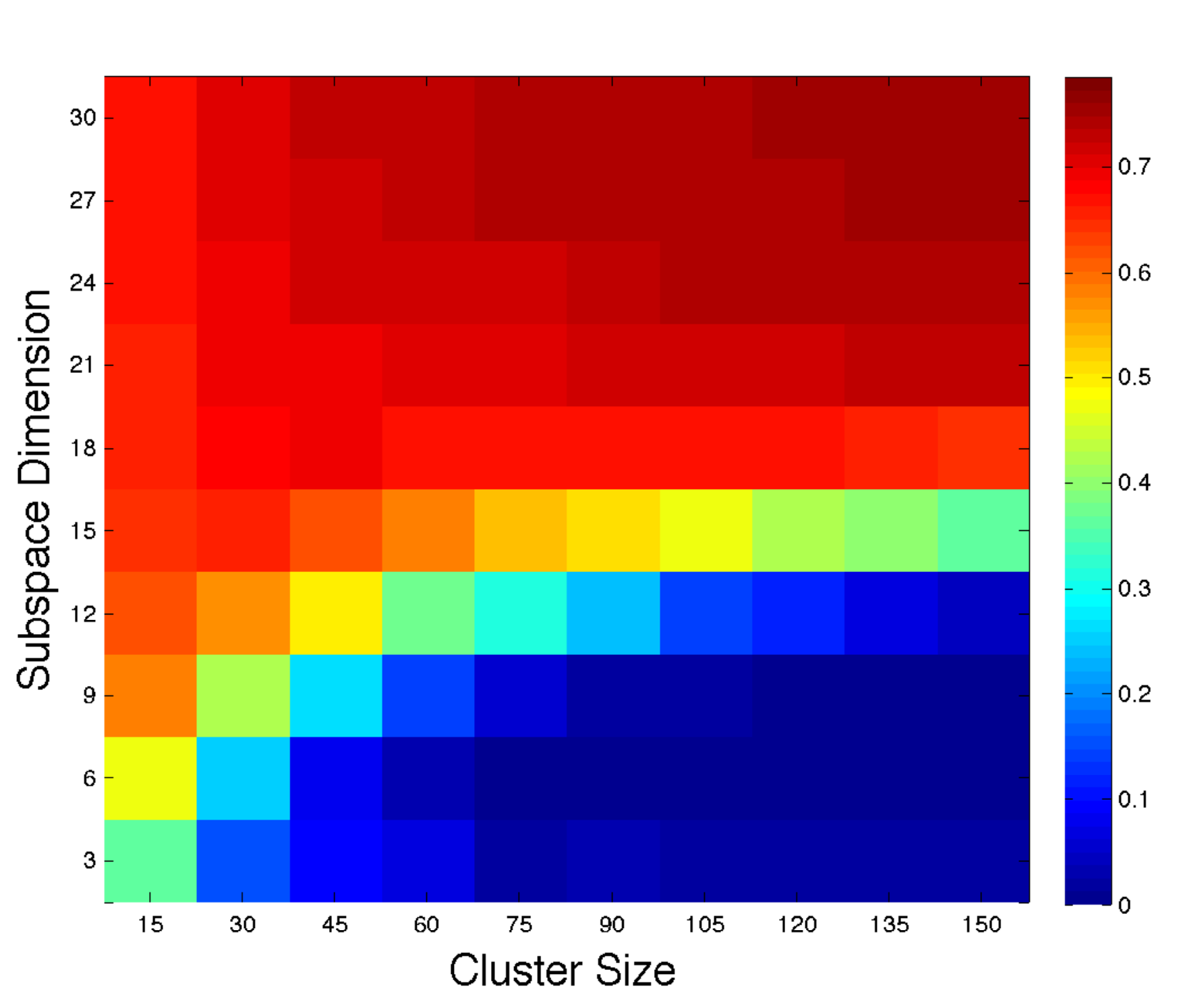}}
	\subfloat[kSSC, $D = 50$]{
	\includegraphics[width=0.25\textwidth]{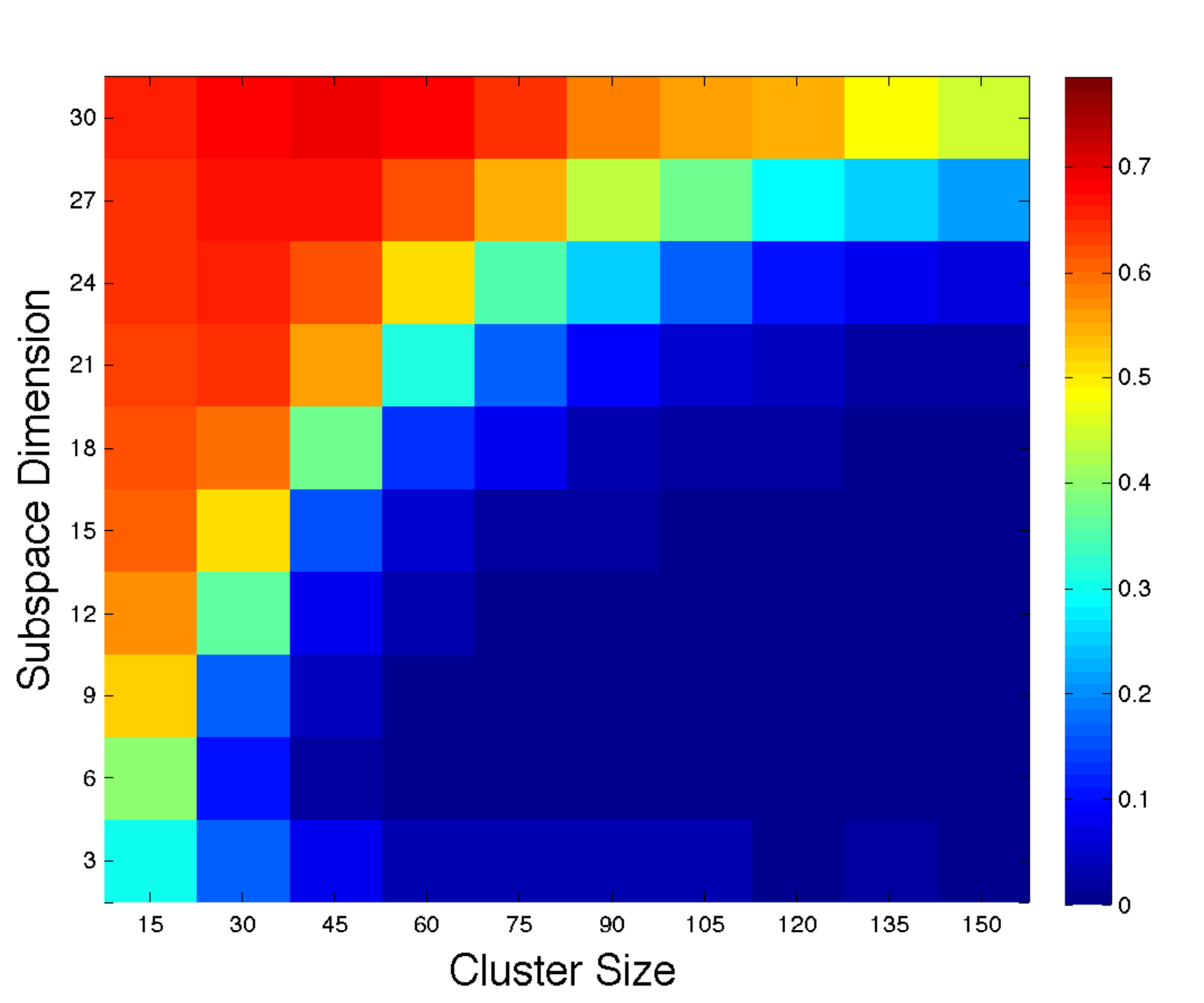}}
	\subfloat[kSSC, $D = 100$]{
	\includegraphics[width=0.25\textwidth]{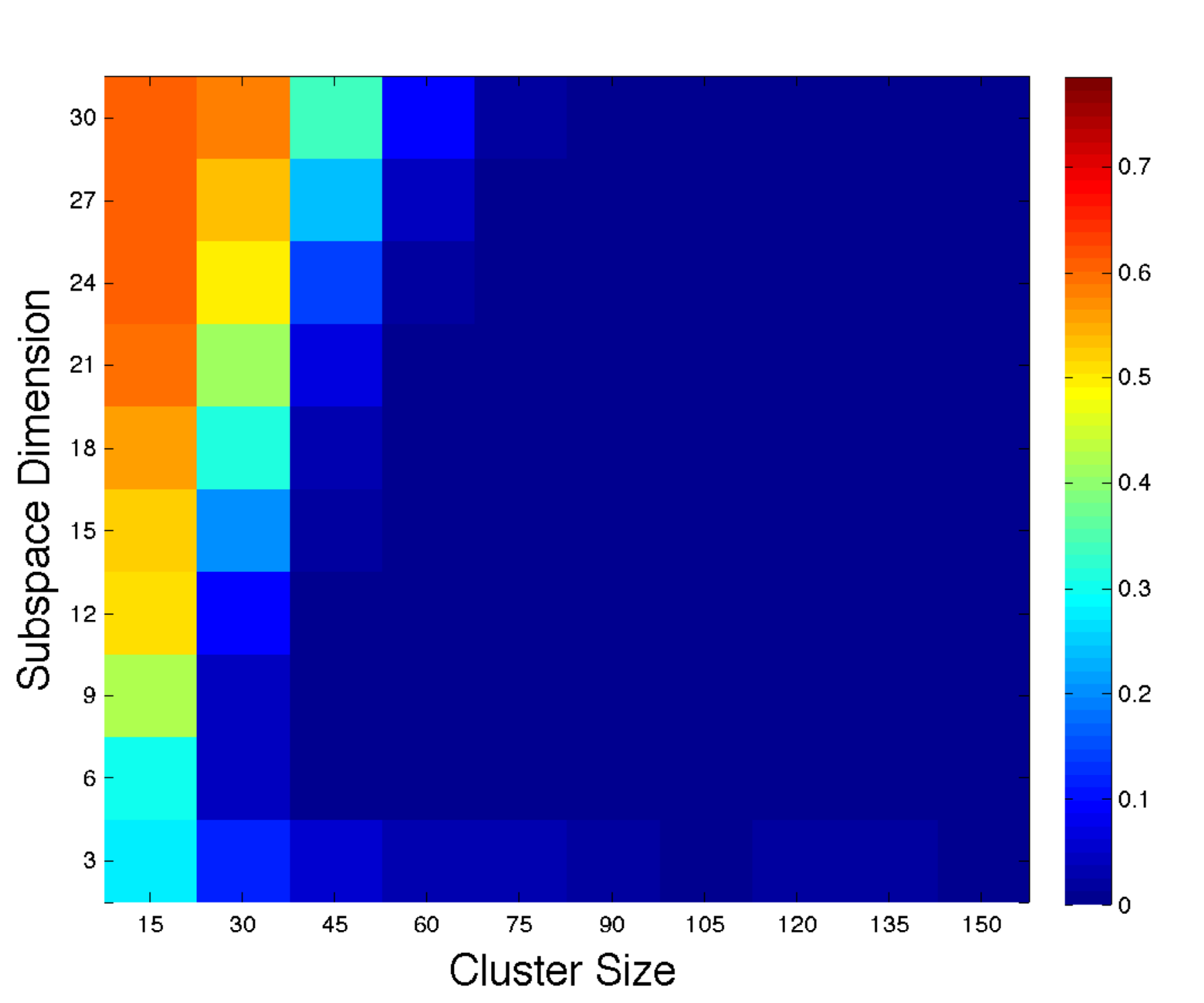}}
\caption{Effect of subspace dimension, cluster size and ambient dimension}
\label{Figure:synthetic_dvsn}
\end{figure}

\subsection{Effect of mean, variance and noise in subspace distribution}

Previous works have only performed empirical evaluation on uniformly distributed synthetic data. That is, the coefficients chosen are uniform random, as was used in the prior subsection. However we contend instead that in reality most data points encountered in the real world are Gaussian distributed over their respective subspace's basis. Furthermore the data points are often corrupted with noise, which we assume will be $\mathcal{N} (0, 1)$.

For this experiment we vary the mean $\mu$ and variance $\sigma^2$ of Gaussian distributed data points using random orthonormal vectors as the basis for each subspace. We create $5$ subspaces with $d_i = 5$, $N_i = 50$ and $D = 50$. For each pair of $\mu$ and $\sigma^2$ we take the mean SCE over $50$ problem instances. We repeat this again for $3$ instances, each time increasing the noise factor, which we report using PSNR. For this experiment we set $k = 10$. The results shown in Figure \ref{Figure:synthetic_mv} that kSSC can match the performance of SSC even when $k \ll d_i$. We note that the effect of mean and variance on point distribution in the subspaces is significantly more pronounced as PSNR decreases.

\begin{figure}[]
\centering
	\subfloat[SSC, PSNR 100]{
	\includegraphics[width=0.25\textwidth]{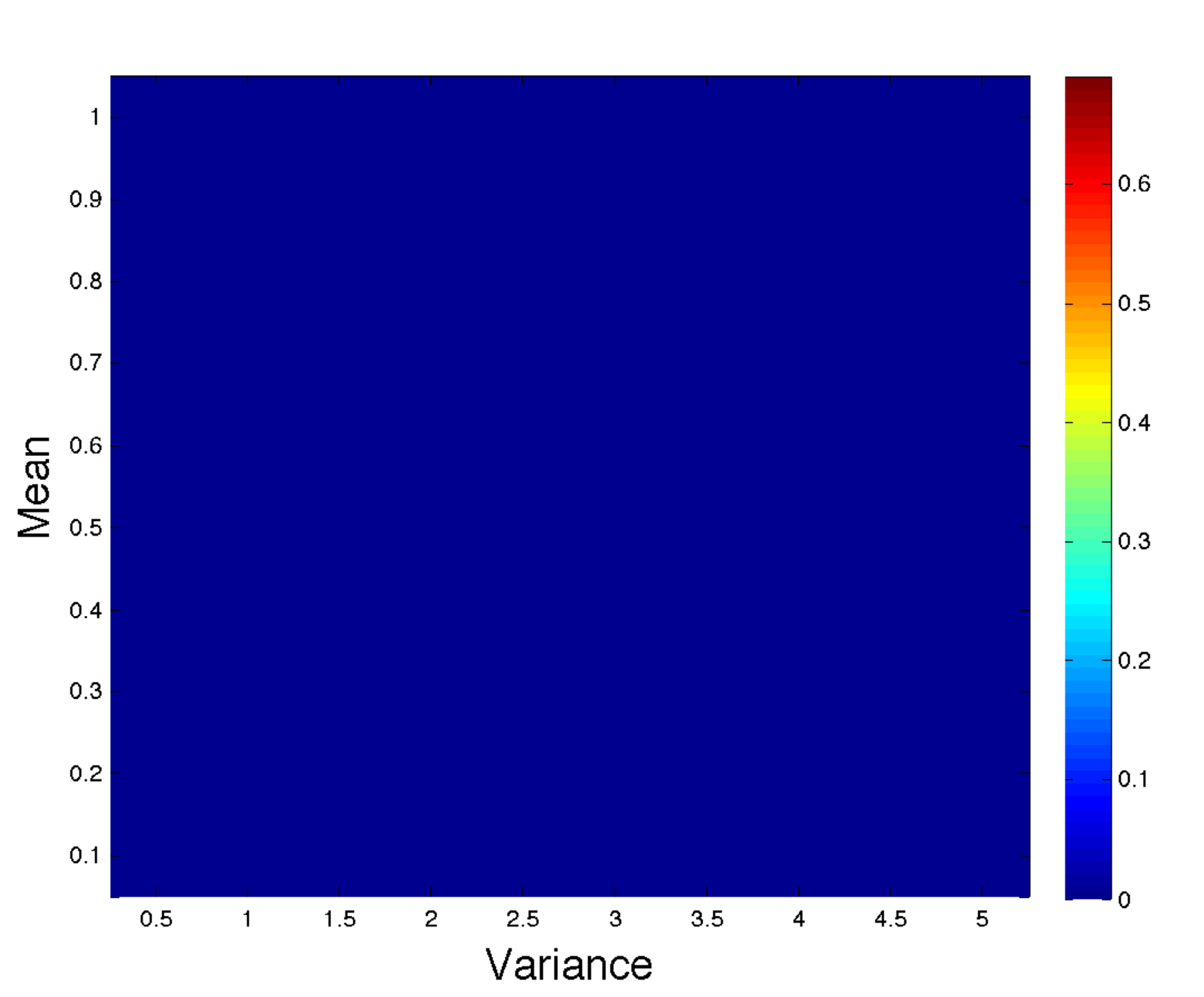}}
	\subfloat[SSC, PSNR 60]{
	\includegraphics[width=0.25\textwidth]{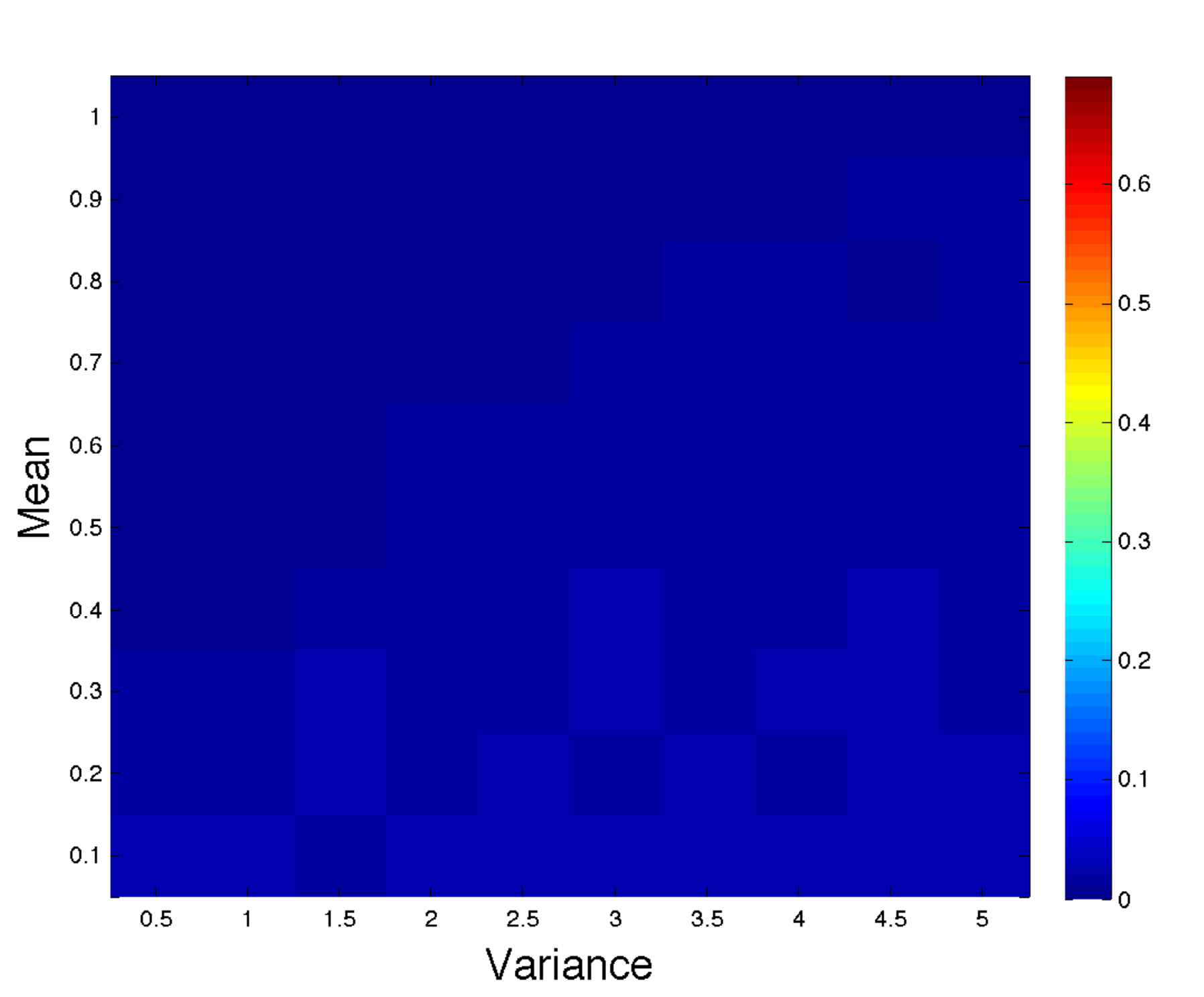}}
	\subfloat[SSC, PSNR 46]{
	\includegraphics[width=0.25\textwidth]{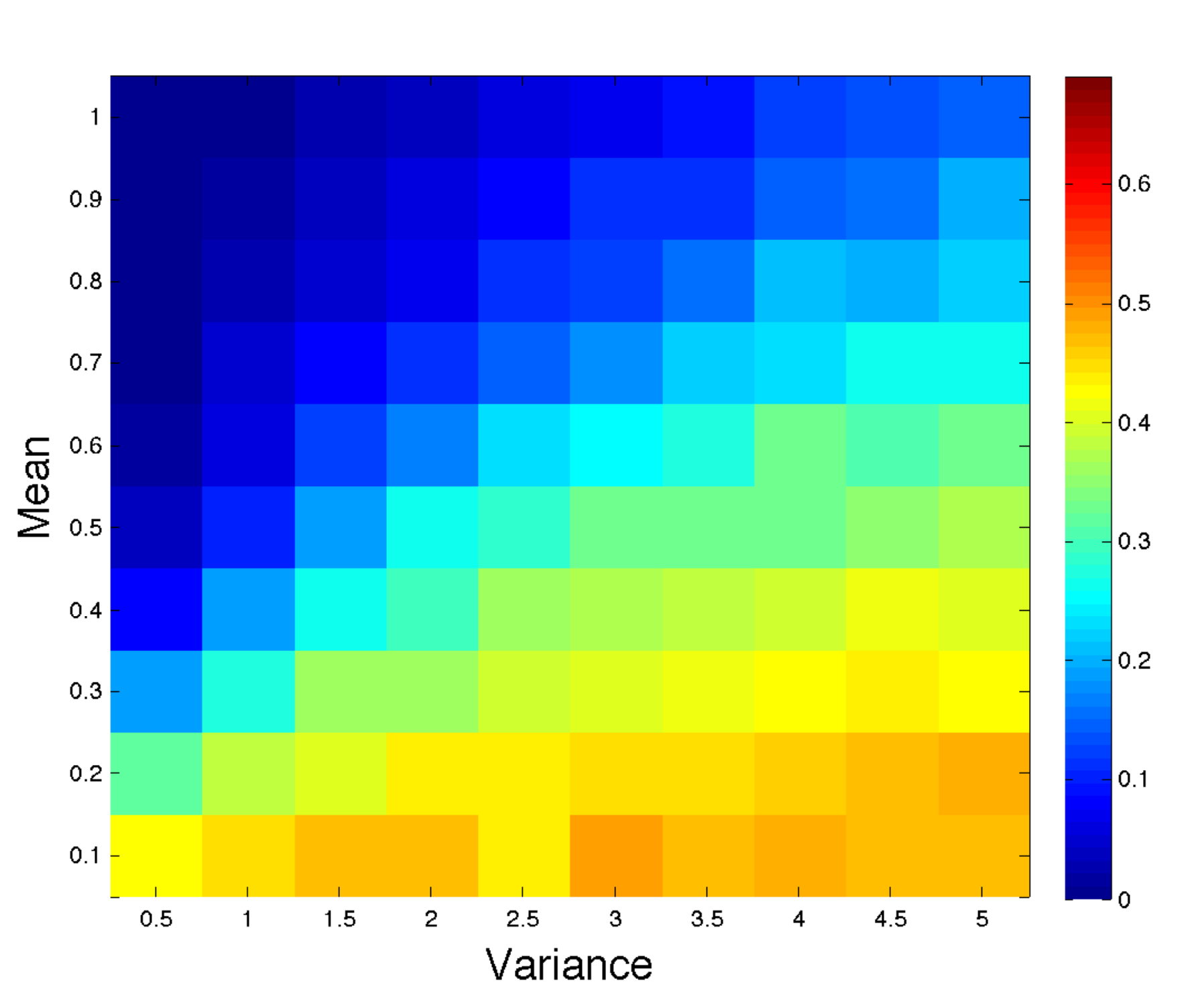}}\\
	\subfloat[kSSC, PSNR 100]{
	\includegraphics[width=0.25\textwidth]{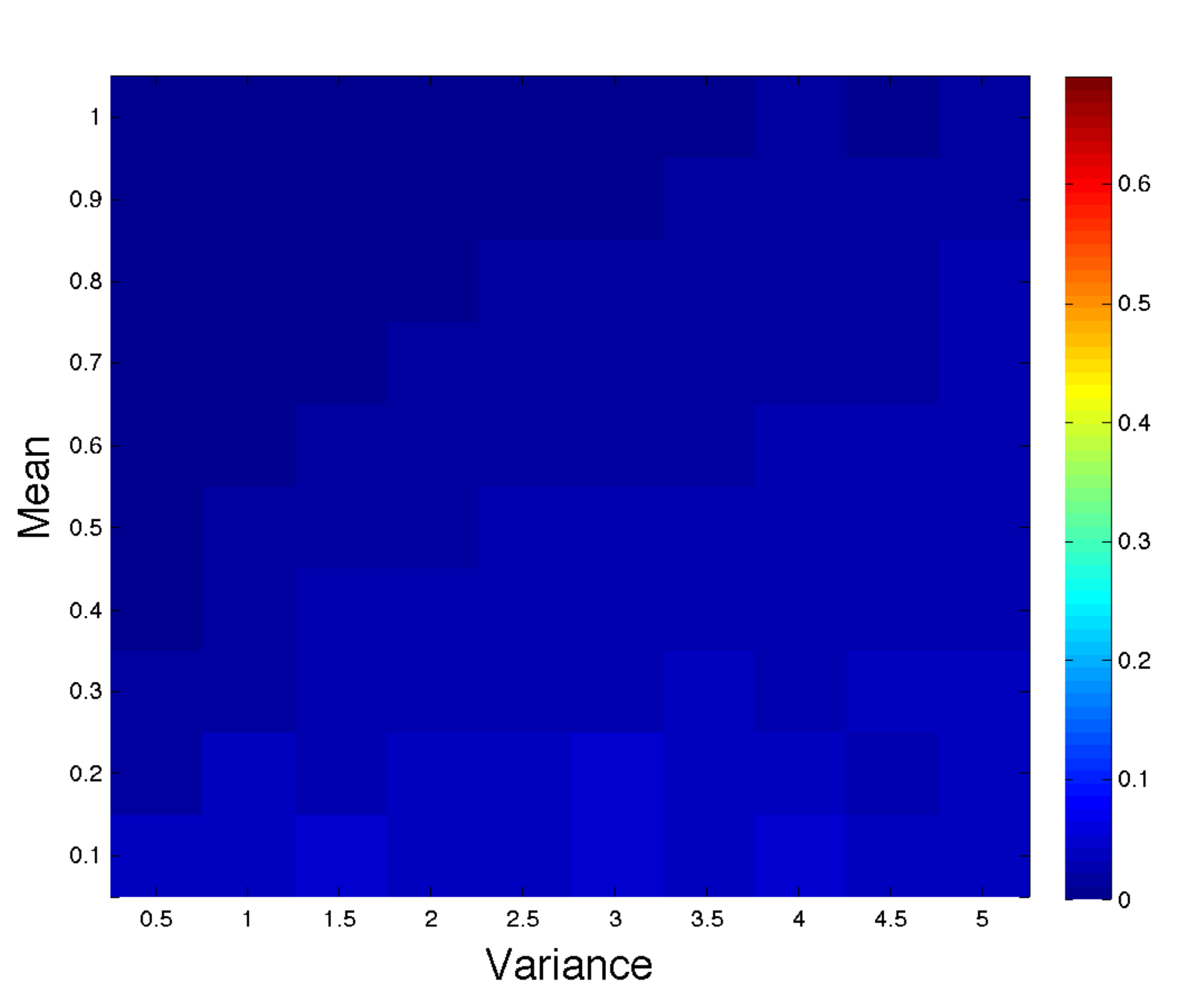}}
	\subfloat[kSSC, PSNR 60]{
	\includegraphics[width=0.25\textwidth]{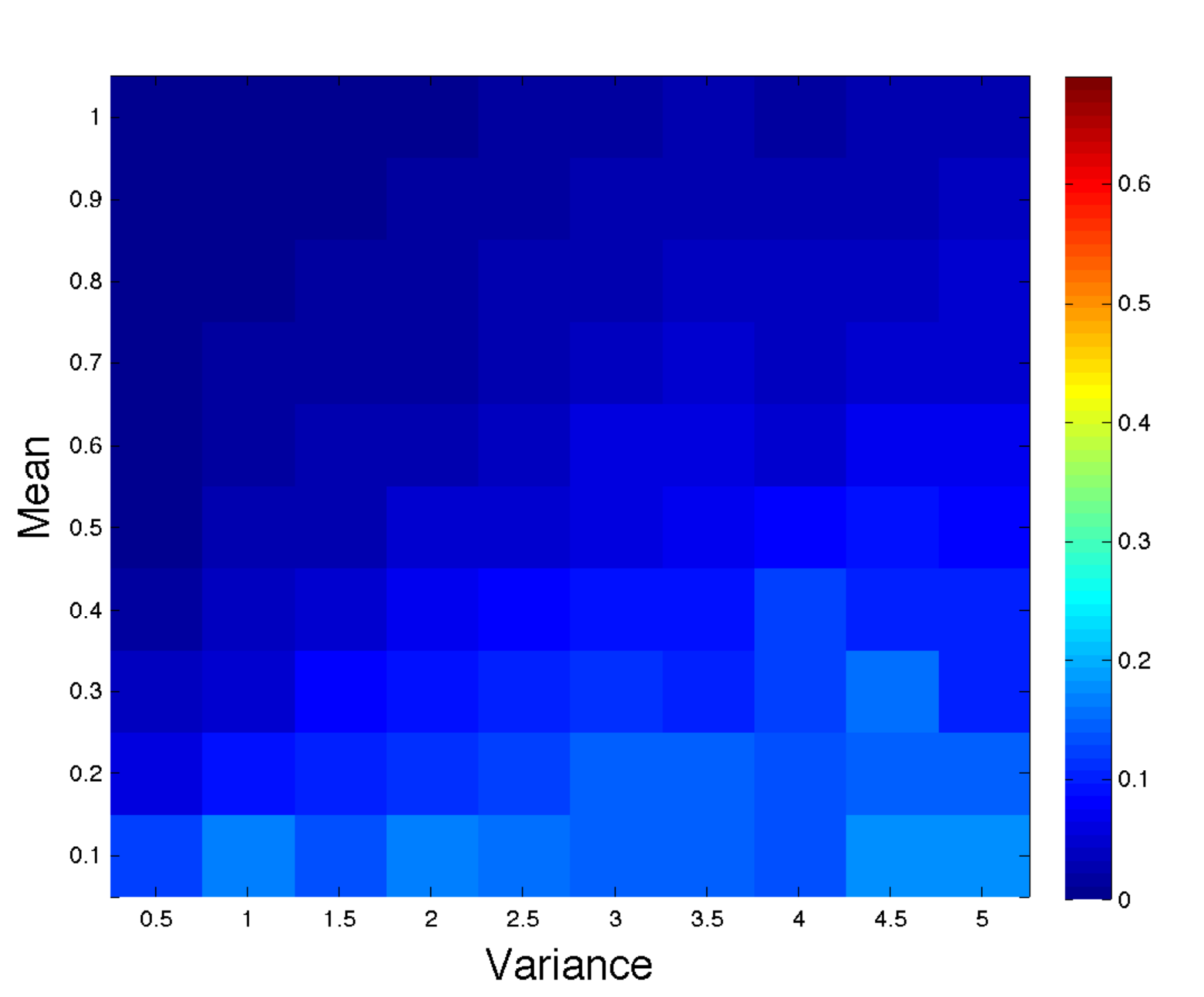}}
	\subfloat[kSSC, PSNR 46]{
	\includegraphics[width=0.25\textwidth]{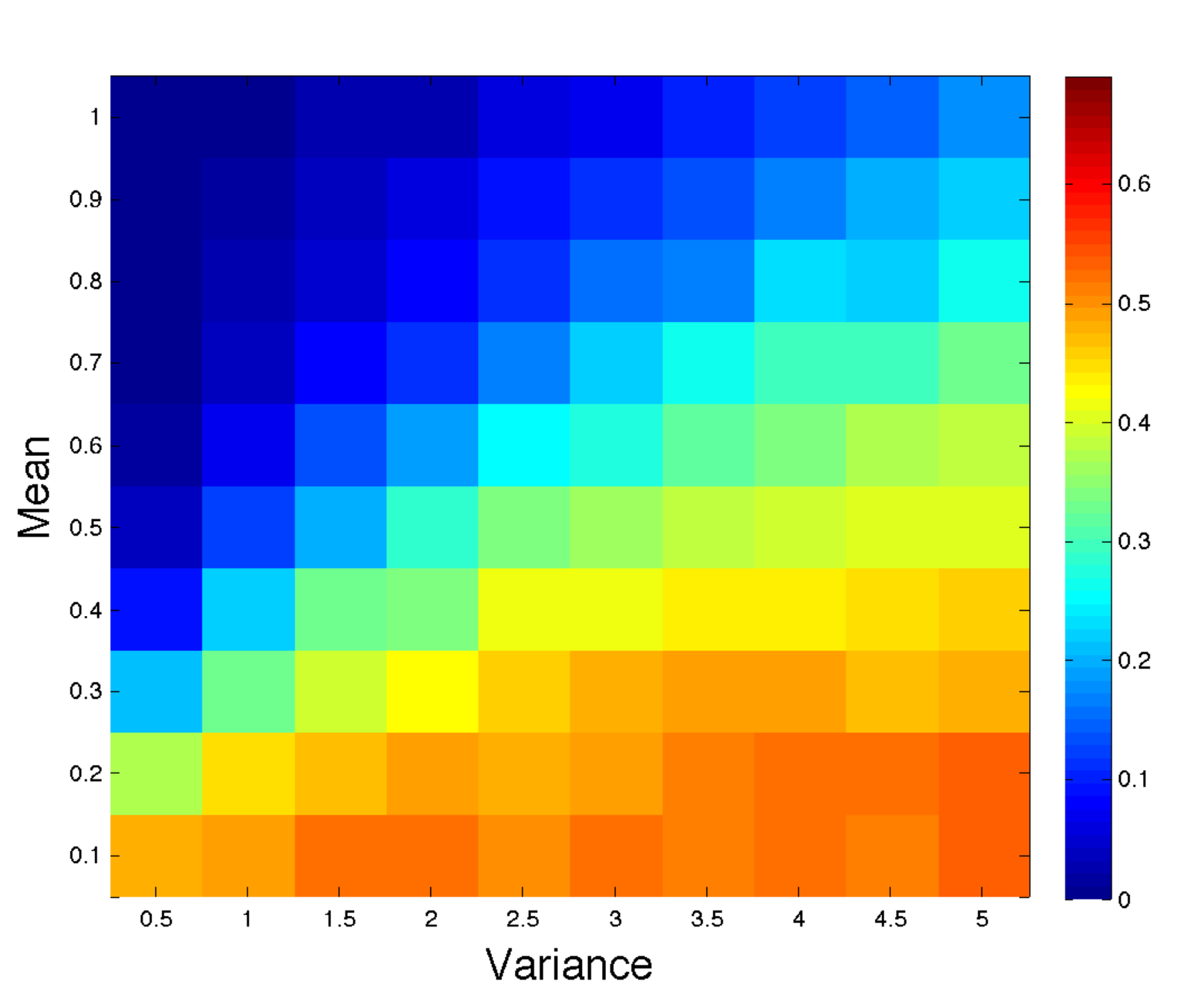}}
\caption{Effect of mean, variance and noise in subspace distribution}
\label{Figure:synthetic_mv}
\end{figure}

\subsection{Effect of subspace intersection}

The intersection of subspaces (shared basis vectors) plays an important role in the clustering accuracy. As previously reported by others, the clustering accuracy decreases as the dimension of intersection increases. To demonstrate this effect and that kSSC can match SSC, we perform the same experiment as found in Section 8.1.1 of \citep{heckel2013robust} and Section 5.1.2 of \citep{SoltanolkotabiCandesothers2012}. We generate two subspaces with $D = 200$, $d_i = 10$ and $N_i = 20d_i$ and vary the number of shared basis vectors $t$ from $0$ to $d_i$. We generate $\mathbf U \in \mathbb R^{D \times 2 d_i - t}$ random orthonormal basis vectors and set the basis vectors for $\mathcal S_1$ to the first $d_i$ columns of $\mathbf U$ and correspondingly the basis for $\mathcal S_1$ to the last $d_i$ columns. We then take the average SCE over $20$ problem instances for each $t$. Results are reported in Figure \ref{Figure:synthetic_intersection}, where we can clearly see that kSSC closely matches the performance of SSC.

\begin{figure}[]
\centering
	\includegraphics[width=0.50\textwidth]{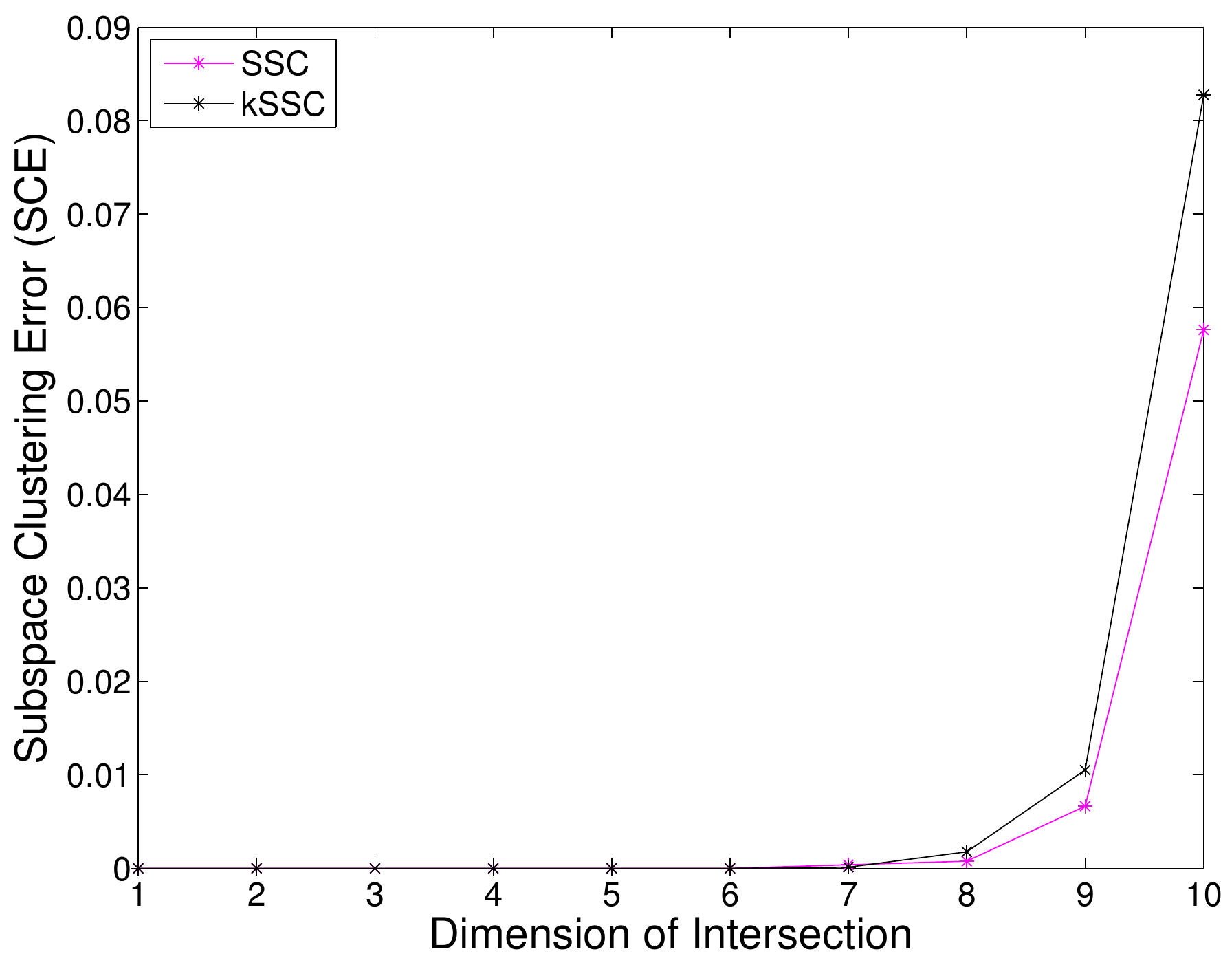}
\caption{Effect of subspace intersection}
\label{Figure:synthetic_intersection}
\end{figure}

\section{Experimental Evaluation}
\label{Section:experiments}

In this section, we evaluate the clustering performance of kSSC on semi-synthetic and real world datasets. We vary the amount of additional noise in some of these experiments to compare the robustness of kSSC against the pre-existing competitor algorithms Greedy Feature Selection (GFS), Greedy Subspace Clustering (GSC), Scalable Sparse Subspace Clustering (SSSC) and Robust Subspace Clustering via Thresholding (TSC). Additionally we use SSC to gauge baseline performance. 


The running times of the experiments carried out in Sections \ref{experiment:tir} and \ref{experiment:motion} can be found in Figure \ref{Figure:tir_motion_time}. Since these experiments are small in size the running time reduction of kSSC is not that significant. However these tests indicate that kSSC matches the clustering accuracy of SSC very closely, an attribute that is not found in other methods. We perform a test in Section \ref{experiment:large_scale} to evaluate the running time of kSSC for a large scale data set.

\begin{figure*}
\centering
\includegraphics[height=0.3\textwidth]{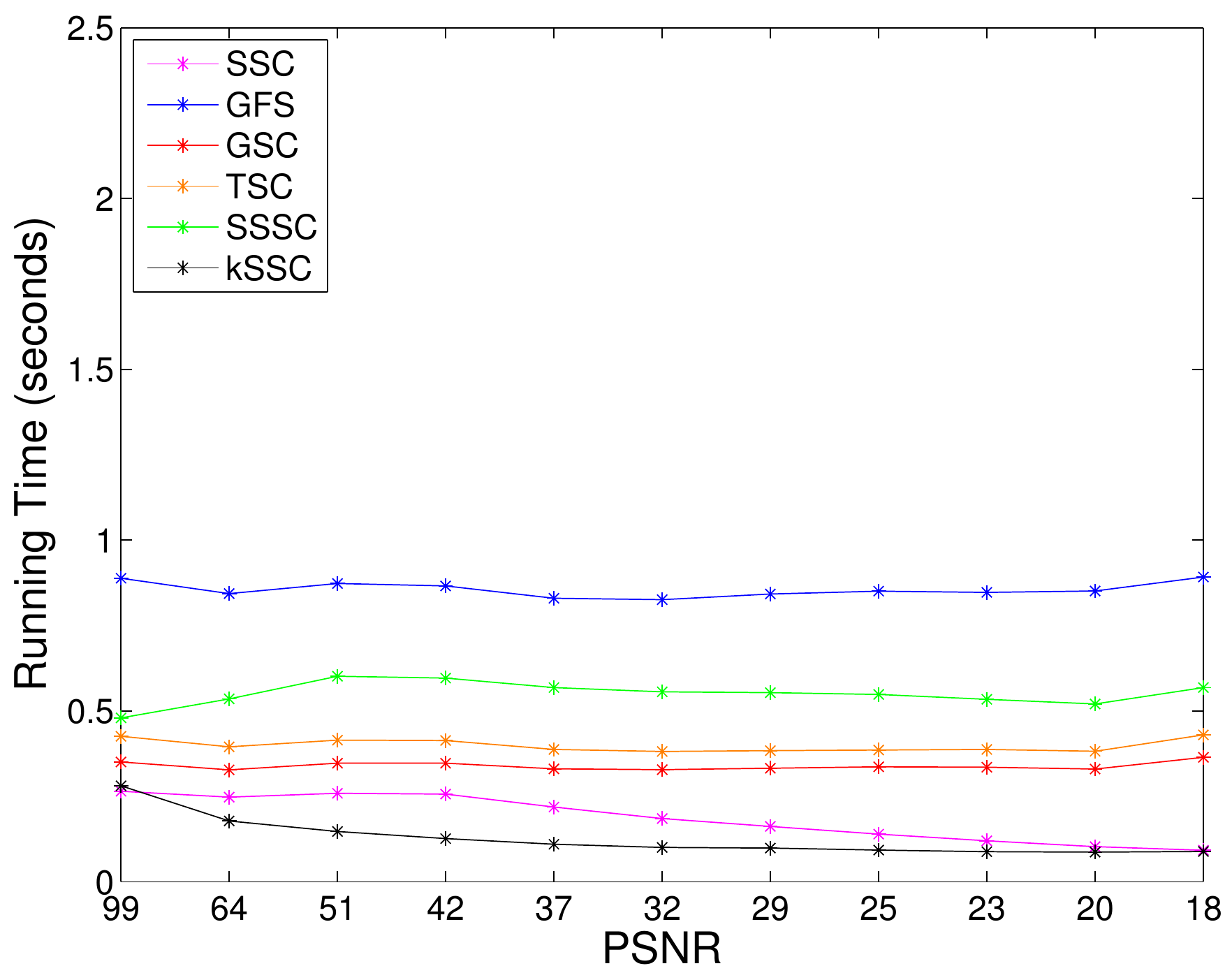}
\includegraphics[height=0.3\textwidth]{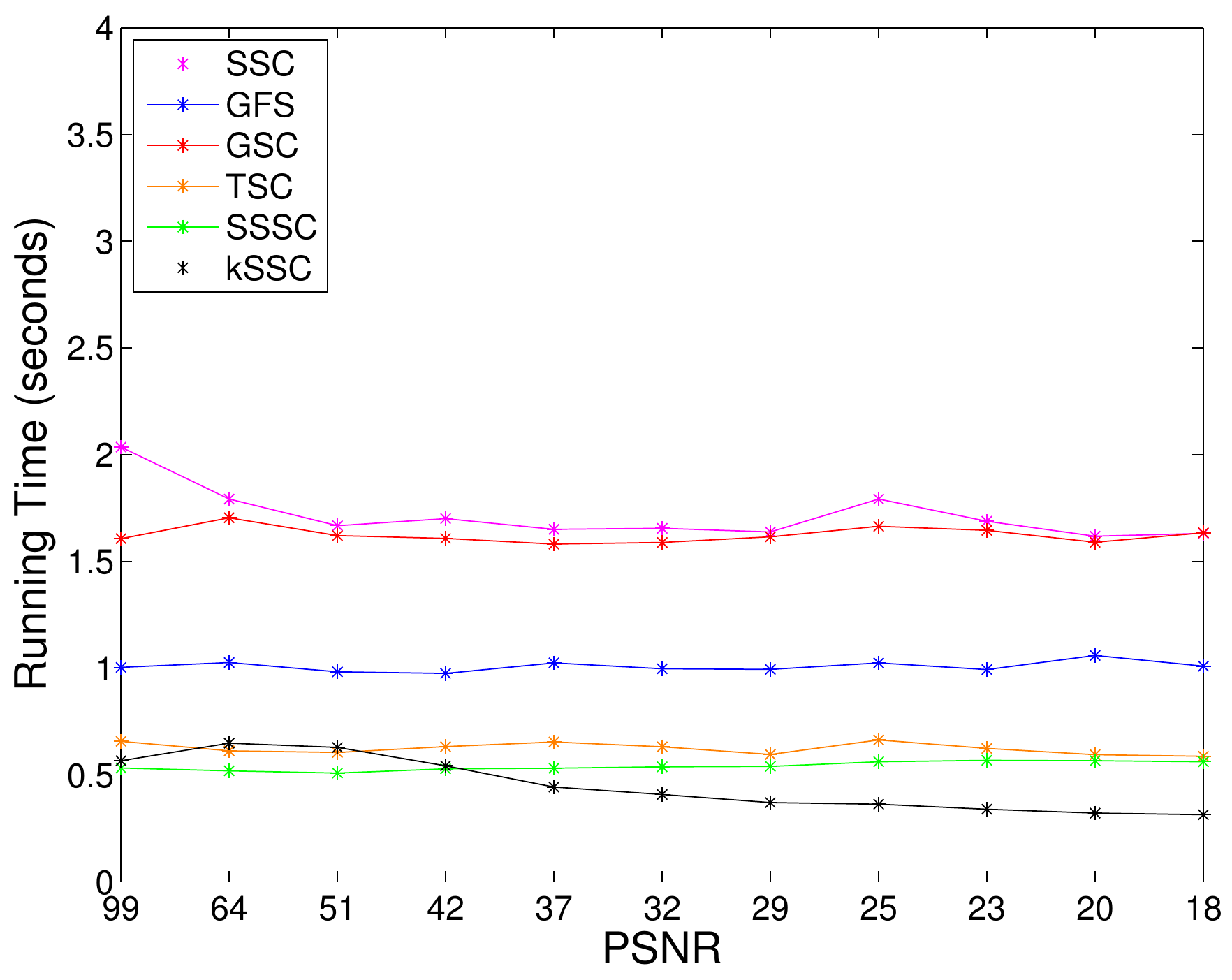}
\caption{Left: Median running time of each tested algorithm for the Thermal Infrared Data Segmentation experiment found in Section \ref{experiment:tir}. Right: Median running time of of each tested algorithm for the Motion Segmentation experiment found in Section \ref{experiment:motion}. Overall in these experiments the benefit of kSSC is slight in comparison to other methods since $N$ is low. We refer to readers Section \ref{experiment:large_scale} for a comparison in running time where $N$ is large.}
\label{Figure:tir_motion_time}
\end{figure*}

\subsection{Thermal Infrared Data Segmentation}
\label{experiment:tir}

We assemble synthetic data from a library of thermal infrared (TIR) hyper spectral mineral data. The library consists of $120$ spectra samples with $D = 321$. We generate $5$ subspaces with $d_i = 5$. For each subspace we randomly select $5$ basis vectors from the spectra samples in the TIR library and generate $50$ points using uniform random coefficients. We then corrupt data with various levels of standard Gaussian noise and evaluate clustering performance of our framework SSC and the Scalable SSC. The experiment is repeated for 50 problem instances for each level of noise to obtain an average SCE. Results can be found in Figure \ref{Figure:tir_results}. kSSC closely tracks the performance of SSC and outperforms all other methods.

\begin{figure}[]
\centering
	\subfloat[Mean Error]{
	\includegraphics[width=0.25\textwidth]{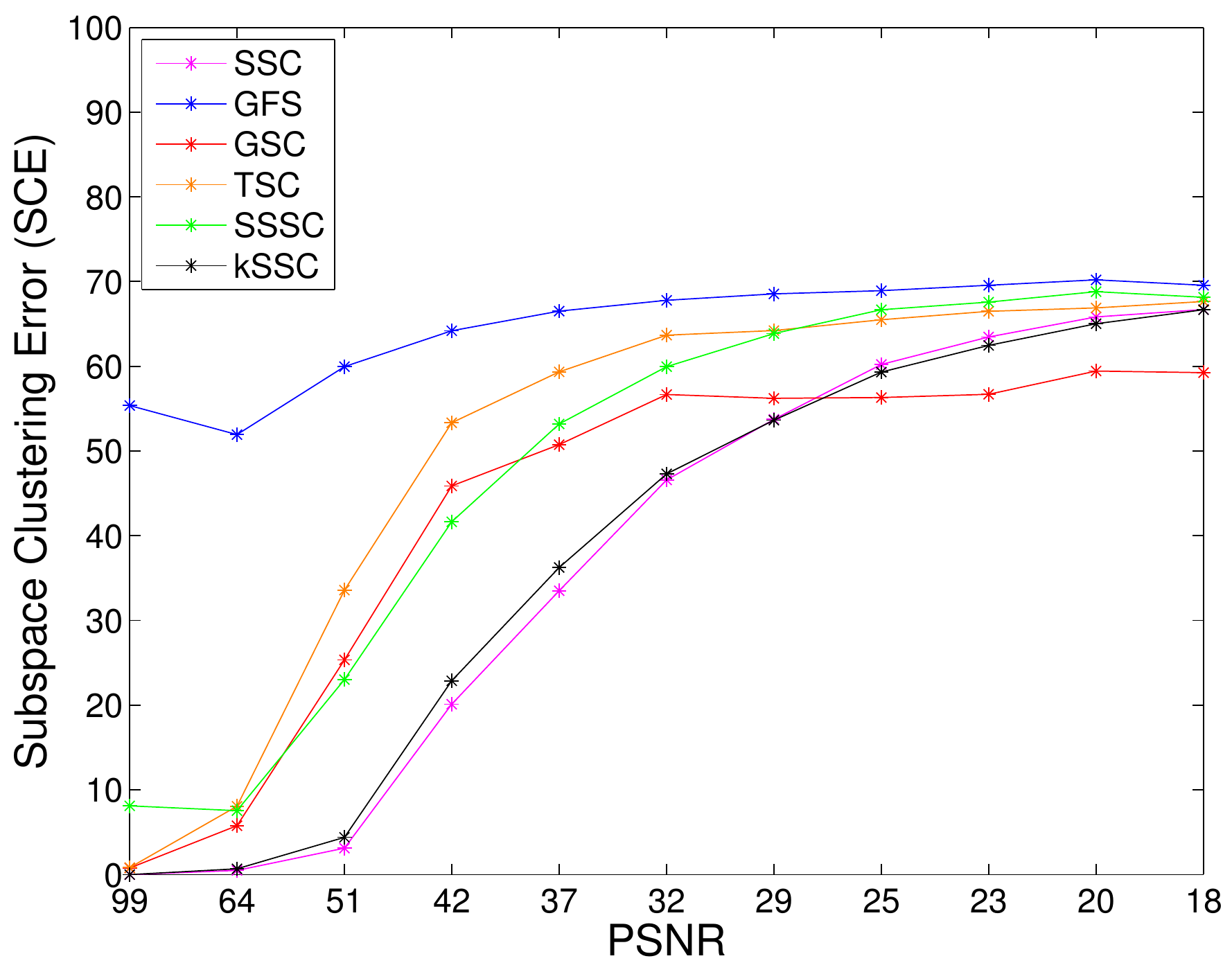}}
	\subfloat[Median Error]{
	\includegraphics[width=0.25\textwidth]{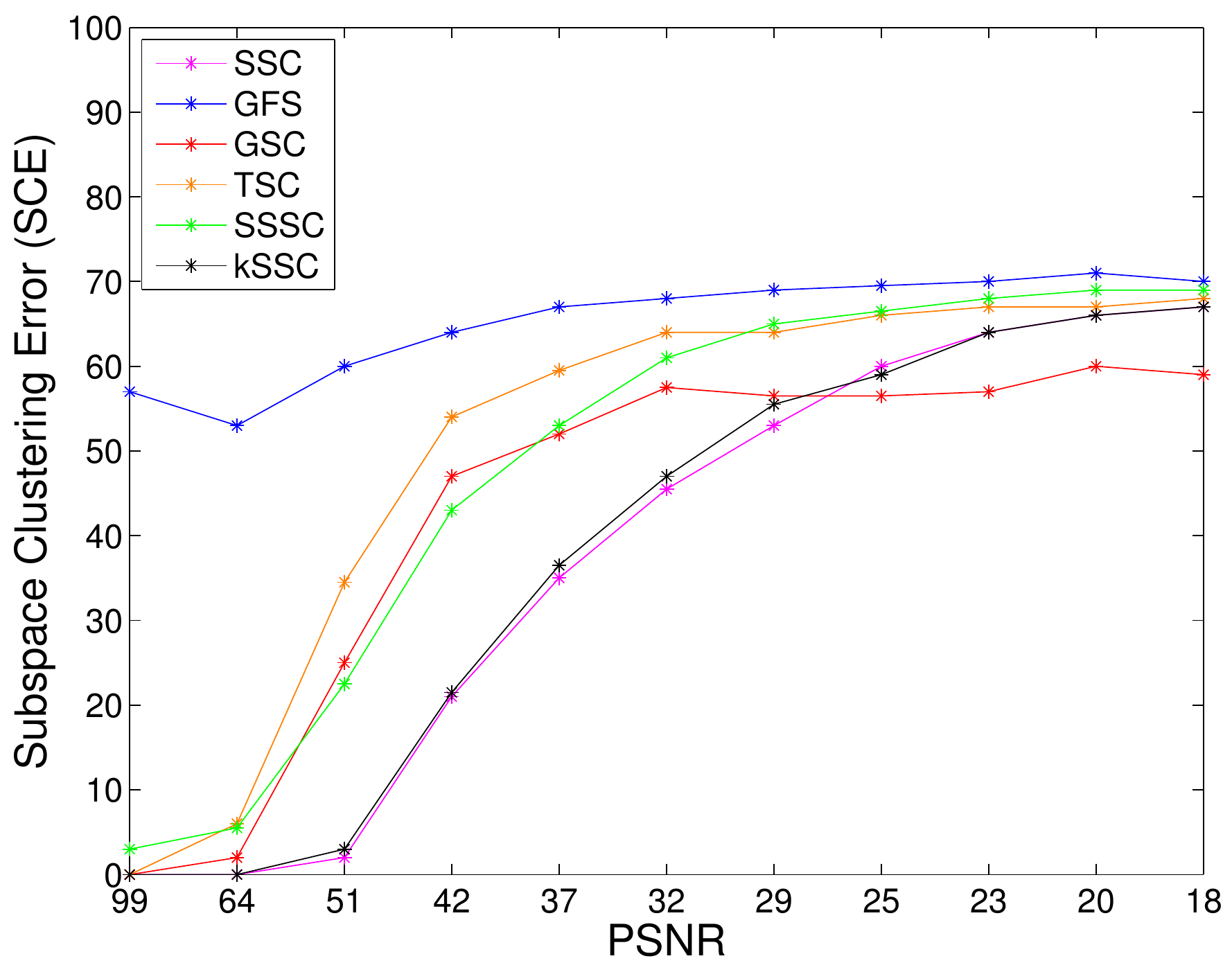}}
	\subfloat[Max Error]{
	\includegraphics[width=0.25\textwidth]{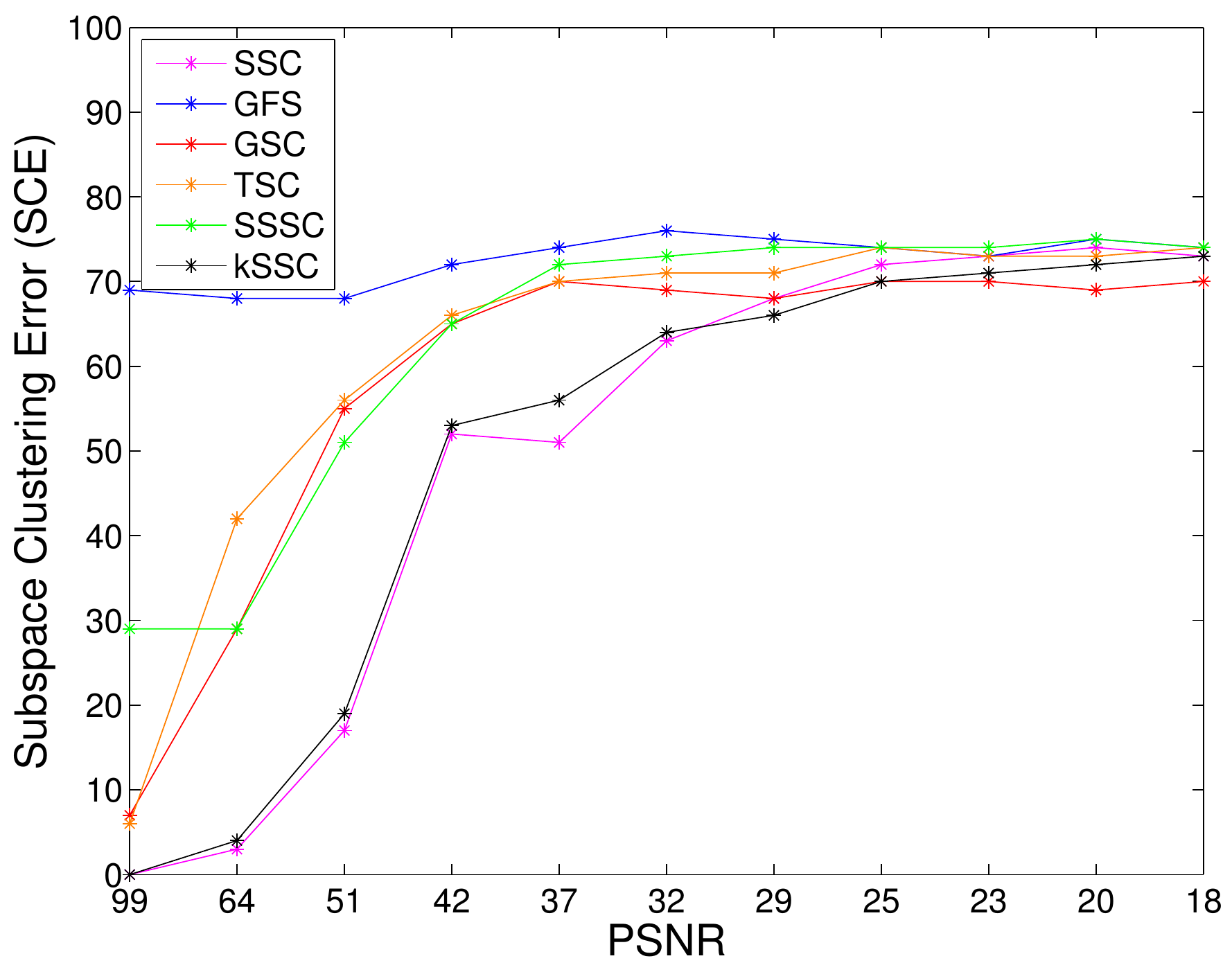}}
	\subfloat[Min Error]{
	\includegraphics[width=0.25\textwidth]{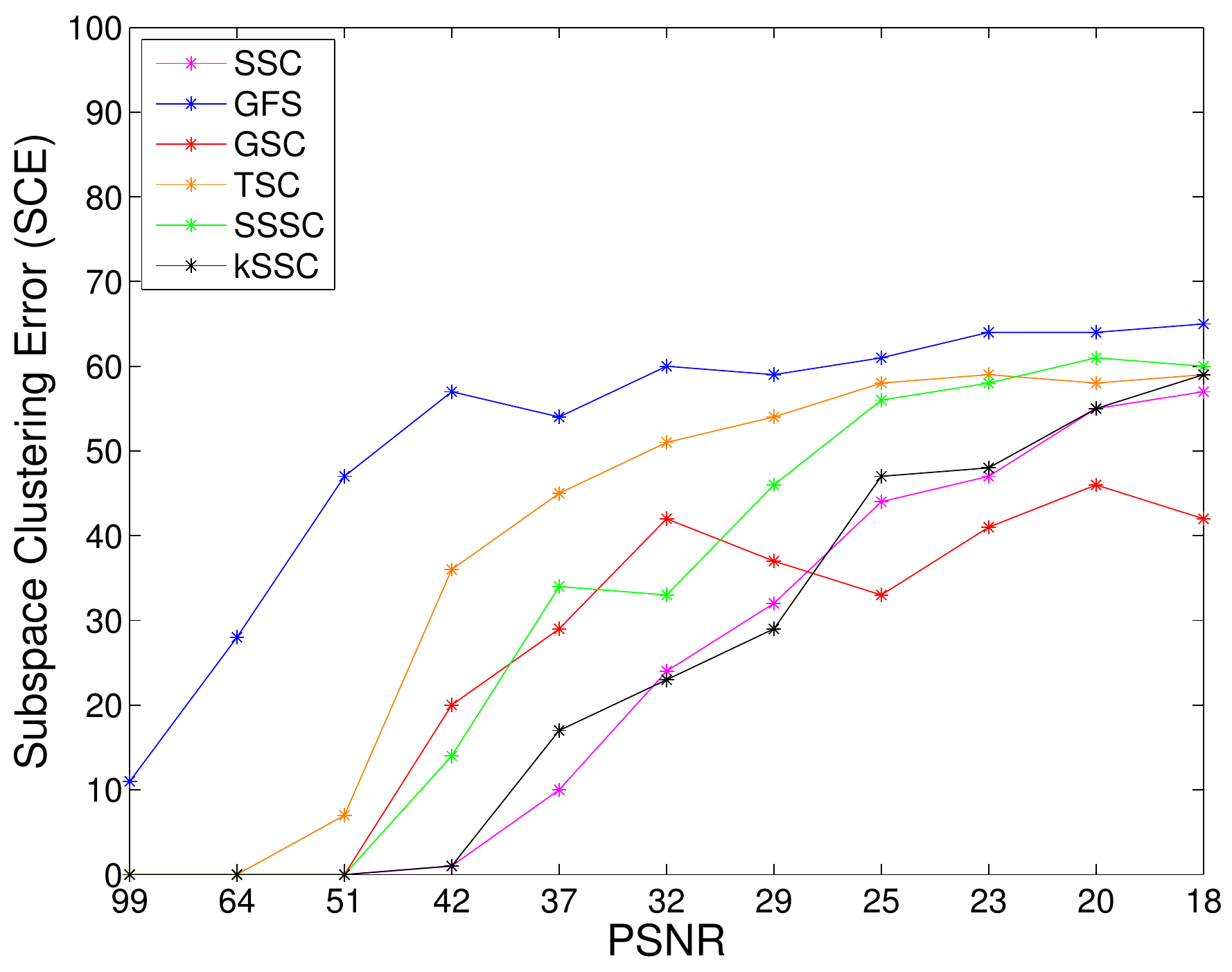}}
	
\caption{Semi-synthetic Hyperspectral TIR}
\label{Figure:tir_results}
\end{figure}

\subsection{Large Scale Thermal Infrared Segmentation}
\label{experiment:large_scale}

The main goal of kSSC is to maintain SSC's clustering accuracy but in a fraction of the time. To confirm this ability, we create a large scale semi-synthetic dataset from the TIR data used in the previous subsection. We generate data in a similar fashion to the previous section. However for each subspace we generate $N_i$ points using uniform random coefficients where we vary $N_i$ from $100$ to $4000$. For this experiment, we stop SSC, GFS and GSC early since they do not scale well in this application (see Section \ref{Section:background}). From Figure \ref{Figure:large_results} we find that kSSC has similar run time characteristics to TSC and SSSC.

\begin{figure}[]
\centering
	\includegraphics[width=0.5\textwidth]{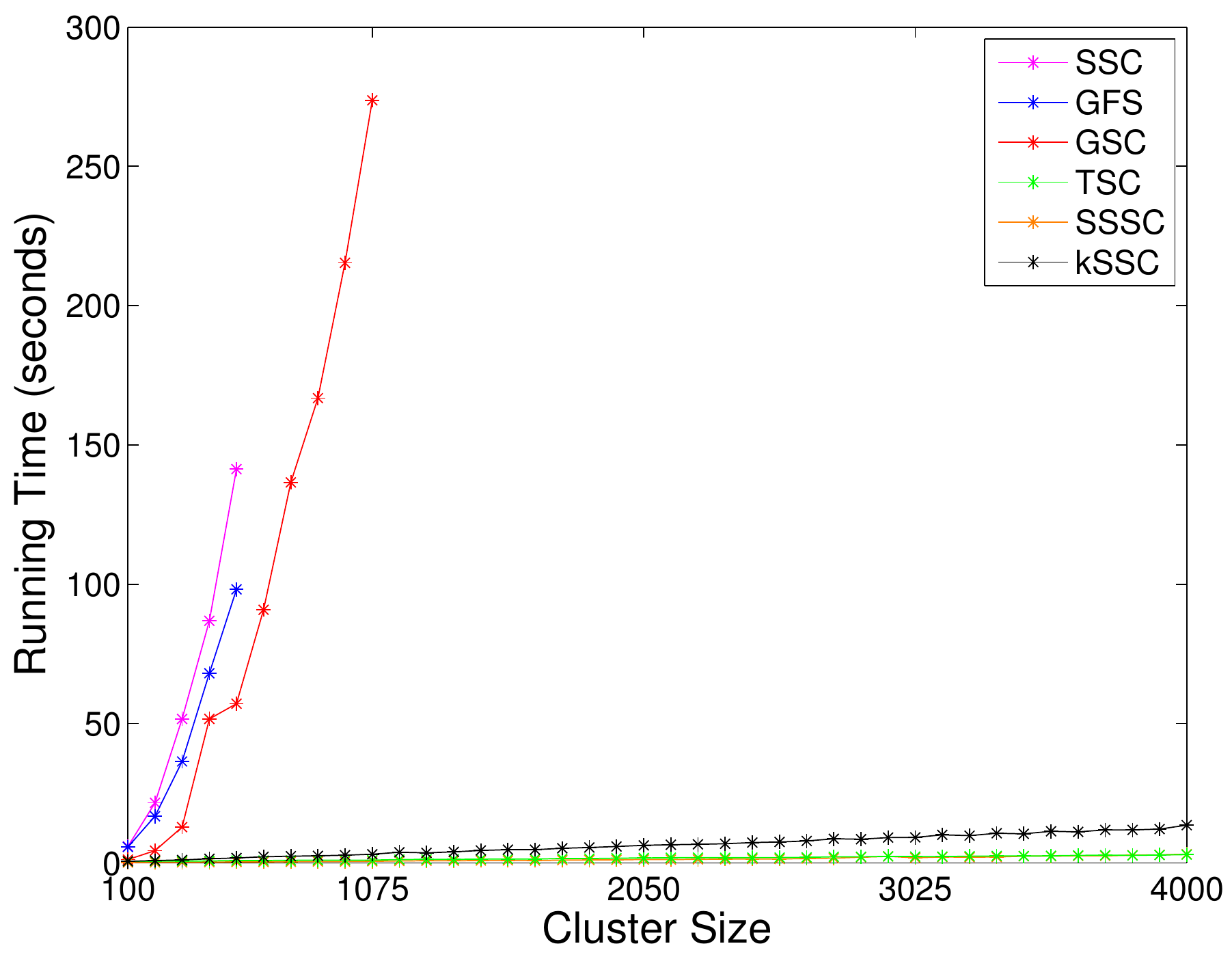}
\caption{Running time of Large Scale Experiment}
\label{Figure:large_results}
\end{figure}

\subsection{Hopkins 155 Motion Segmentation}
\label{experiment:motion}

The aim of this experiment is to assign feature points extracted from a video to their corresponding motion or object in the scene. As previously mentioned, it has been shown in \cite{elhamifar2012sparse} that these features trajectories actually correspond to low-dimensional subspaces.  The data from this experiment is drawn from the rigid motion sequences of the Hopkins 155 dataset \cite{tron2007benchmark}. These sequences have around $200$-$500$ feature trajectories and range in number of frames from $20$-$60$. Results can be found in Figure~\ref{Figure:tir_results2}.. Again kSSC closely tracks the performance of SSC and consistently performs as the PSNR decreases.


\begin{figure}[]
\centering
	\subfloat[Mean Error]{
	\includegraphics[width=0.25\textwidth]{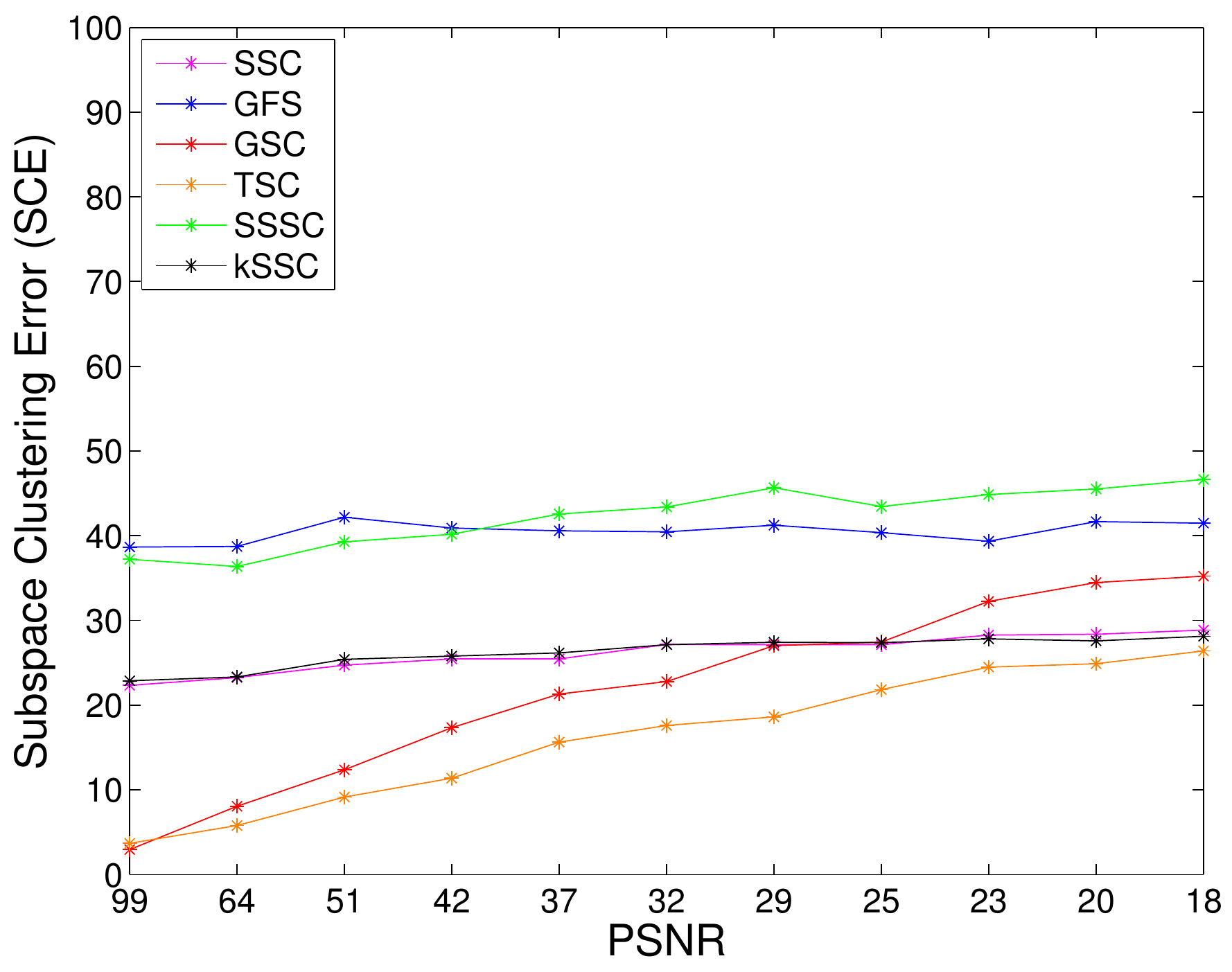}}
	\subfloat[Median Error]{
	\includegraphics[width=0.25\textwidth]{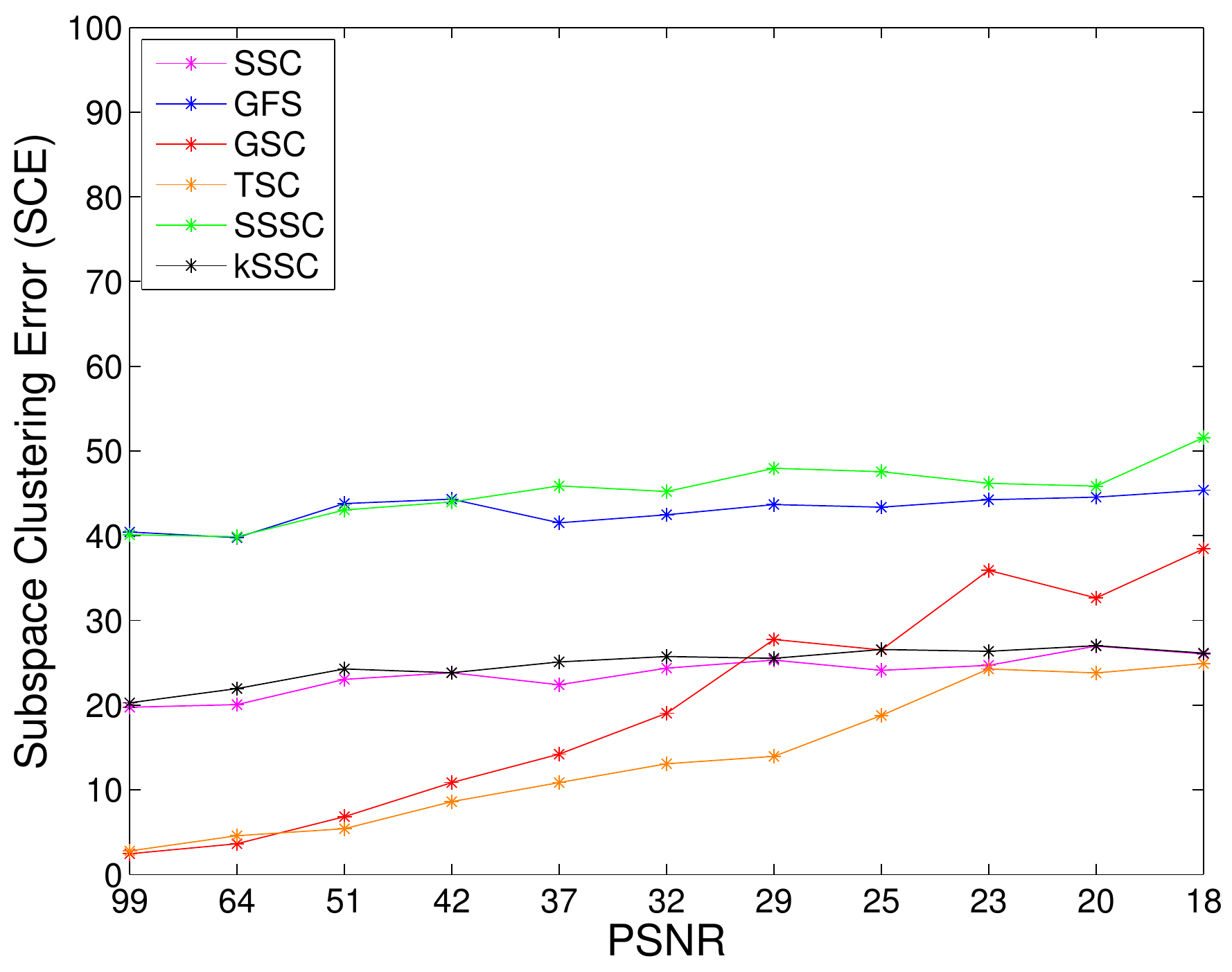}}
	\subfloat[Max Error]{
	\includegraphics[width=0.25\textwidth]{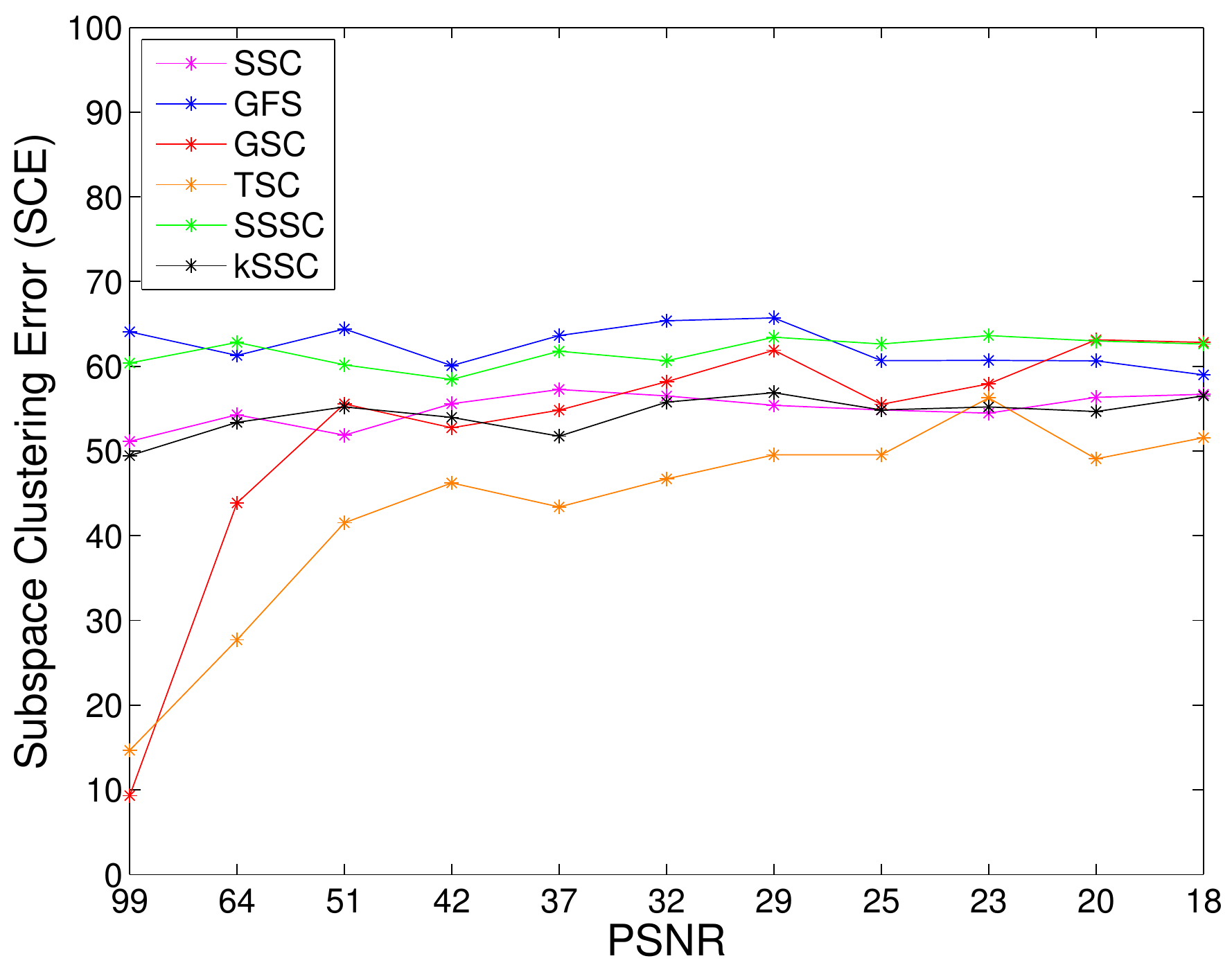}}
	\subfloat[Min Error]{
	\includegraphics[width=0.25\textwidth]{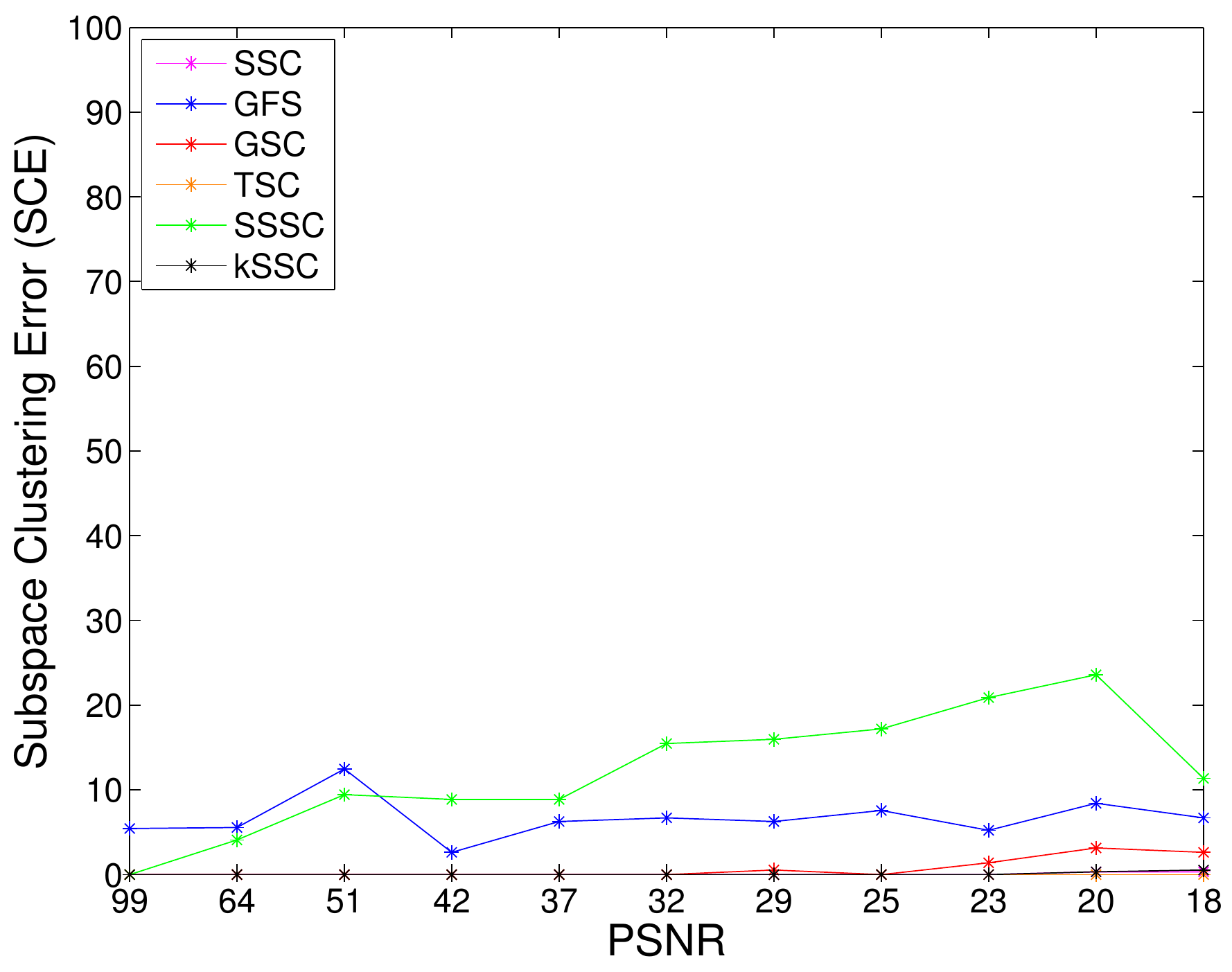}}
	
\caption{Rigid Motion Segmentation}
\label{Figure:tir_results2}
\end{figure}


\subsection{Extended Yale B Face Clustering}
\label{experiment:faces}

\begin{table*}
\centering
\begin{tabular}{c c c c c c c}
\hline
 & Mean & Median & Min & Max & Std & Mean Run Time (s) \\
\hline

SSC & 28.1\% & 31.5\% & \hl{0.0\%} & 65.6\% & 24.5\% & 21.38 \\ 
GSC & 30.6\% & 29.9\% & 0.5\% & \hl{55.2\%} & 15.6\% & 30.38 \\ 
TSC & 58.1\% & 61.2\% & 36.5\% & 66.1\% & \hl{7.3\%} & 7.59 \\ 
SSSC & 59.5\% & 58.6\% & 31.8\% & 100.0\% & 17.2\% & \hl{0.66} \\ 
kSSC & \hl{22.3\%} & \hl{17.2\%} & \hl{0.0\%} & 65.1\% & 19.3\% & 30.38 \\ 

\end{tabular}

\caption{Face Clustering results from the Extended Yale B Dataset.}
\label{Table:face_results}
\end{table*}

The aim of this experiment is to cluster unique human subjects from a set of face images. We draw our data from the Exteded Yale Face Database B \cite{georghiades2001few}. The dataset consists of approximately 64 photos of 38 subjects under varying illumination. We select three subjects randomly then resample their images to $96 \times 84$ and form data vectors $\mathbf x_i \in \mathbb{R}^{2016}$ by concatenating them together. This test was repeated 50 times with new random subjects each time. This is a challenging dataset since the original data is already corrupted by shadows from the varied illumination. Results can be found in Table \ref{Table:face_results}.  Surprisingly we find that kSSC even outperforms SSC for this task, which may be due to the aggressive kNN pre-screening process removing all but the most similar data points. In this dataset, there are many face images of different subjects that contain large regions of highly similar data due to the extreme occlusions from shadows. We believe the nearest neighbour filtering selection helps to prevent the possibility of extreme false positives connections in $\mathbf Z$.

\section{Conclusion}
\label{Section:conclusion} 

In this paper we proposed a new algorithm, kSSC, to accurately and tractably approximate SSC for large scale datasets. By accurately screening out the vast majority of eligible data points as neighbours the memory and computational requirements are reduced from $\BigO{N^2}$ to $\BigO{N}$. Our theoretical analysis shows that we can theoretically match the subspace identification performance of SSC provided that we have sufficient sampling and the magnitude of the noise is not too great. Moreover our empirical results on synthetic and real data demonstrate that kSSC outperforms the existing SSC approximation methods in terms of accuracy and matches or beats the computational and memory requirements.

\section*{Acknowledgment}

The research project is supported by the Australian Research Council (ARC) through grant DP140102270. It is also partly supported by the  Natural Science Foundation of China (NSFC) through project 41371362.

\section*{References}

\bibliography{references}

\appendix

\section{Analysis of True Discoveries by kNN}
\label{Appendix:Analysis}

We need the concentration of measure Theorem 1 in \citep{BerestyckiNickl2009}
\begin{thm}\label{thm:concentrationmeasure}
Suppose that $Z\in\Real^d$ is uniformly distributed on unit sphere $\mathbb S^{d-1}$ and A is any subset with $P(A)\le 1/2$. $A_\epsilon$ is the $\epsilon$-neighborhood as
\[
A_\epsilon = \{ \vc x \in \mathbb S^{d-1} : d(\vc x,\vc y) <\epsilon \text{ for some }\vc y\in A\}
\]
then $Z$ belongs to $A_\epsilon$ with some probability, i.e.
\[
P(Z \in A_\epsilon) \le 1- e^{-(d-2)\epsilon^2/2}.
\]
\end{thm}
Then we obtain the concentration of samples in a small patch when the number of samples is large.
\begin{thm}\label{thm:kconcentration}
Let $A_\epsilon$ be the small $\epsilon$-neighborhood defined previously in $\Real^d$ on $\mathbb S^{d-1}$, and $N$ the number of random variables uniformly distributed on $\mathbb S^{d-1}$. For $N = \frac{k_0}{\epsilon^2} e^{(d-2)\epsilon^2/2}$, we have
\[
P(K > k_0/C) \ge 1 - e^{-k_0}(eC)^{k_0/C}.
\]
where $k_0$ and $C>1$ are some constant.
\end{thm}
\begin{proof}
For any given $\vc x\in\Real^d$, using Theorem \ref{thm:concentrationmeasure}, we know that
\[
p = P(\vc x \in A_\epsilon) \le 1- e^{-(d-2)\epsilon^2/2}.
\]
Then the event that $K$ variables fall into $A_\epsilon$ follows binomial distribution B$(N,p)$
\[
P(K=k) = \binom{N}{k}p^k(1-p)^{N-k}.
\]
When $N$ is large and $p$ is small, which is the case here, the above binomial distribution can be approximated by a Poisson distribution Poi($\lambda$) with
\[
\lambda = N(1- e^{-(d-2)\epsilon^2/2}).
\]
Using Chernoff bound on Poisson distribution
\[
P(K \le k) \le e^{-\lambda+k}(\frac{\lambda}{e})^k.
\]
Substitute the condition on $N$, we have
\[
P(K\le\lambda/C) \le e^{-k_0}(eC)^{k_0/C}.
\]
This completes the proof.
\end{proof}
The above theorem states that if $N$ is large enough, with large probability there are $k_0/C$ variables in side $A_\epsilon$.

Next we bound the inner product between samples from different subspaces. Here we use the arguments in Lemma 7.5 in \citep{SoltanolkotabiCandesothers2012}.
\begin{lem}\label{lem:ipbound}
Let $\mat A\in\Real^{d_1\times N_1}$ be a matrix with columns uniformly distributed in $\mathbb S^{d_1-1}$, $\vc y\in\Real^{d_2}$ be a vector uniformly sampled from $\mathbb S^{d_2-1}$ and a deterministic matrix $\vv\Sigma\in\Real^{d_1\times d_2}$. For $t>0\in\Real$, the inner product between any column in $\mat A$ and $\vv\Sigma\vc y$ is bounded as follows
\[
\vc a_i^\top\vv\Sigma\vc y \le \frac{2\sqrt{t\log N_1+t^2}\|\vv\Sigma\|_F}{\sqrt{d_1}}
\]
with probability at least $1-2 e^{-t}$.
\end{lem}
\begin{proof}
Using Borell's inequality on the mapping $\vc y \mapsto \|\vv\Sigma\vc y\|$ with Lipschitz constant of $\sigma_1$, the largest singular value of $\vv\Sigma$ leads to
\[
P(\|\vv\Sigma\vc y\| > \varepsilon+\sqrt{E\|\vv\Sigma\vc y\|^2}) < e^{-\frac12\varepsilon^2/\sigma_1^2}.
\]
As $E\|\vv\Sigma\vc y\|^2 = \|\vv\Sigma\|_F^2/d_2$, we choose $\varepsilon = (b-1)\|\vv\Sigma\|_F/\sqrt{d_2}$ so that
\be\label{E:ipbound_eq1}
P(\|\vv\Sigma\vc y\| > \frac{b\|\vv\Sigma\|_F}{\sqrt{d_2}}) < e^{-\frac12(b-1)^2/d_2},
\ee
where we used the fact that $\|\vv\Sigma\|_F/\sigma_1 >1$.

The next step is to bound the inner product of a column in $\mat A$ ($\vc a_i$ $i=1\ldots N_1$) with any vector $\vc x\in\Real^{d_1}$ by upper bound of spherical caps
\[
P(\vc a^\top\vc x>\varepsilon\|\vc x\|) < e^{-\frac12d_1\varepsilon^2},
\]
which leads to the following using the union bound
\be\label{E:ipbound_eq2}
P(\bigcup_i\vc a_i^\top\vc x>\varepsilon\|\vc x\|) < N_1e^{-\frac12d_1\varepsilon^2}.
\ee
Let $\varepsilon = \sqrt{2\log N_1+2t}$, $b=\sqrt{2d_2t}$. Substituting \eqref{E:ipbound_eq1} to \eqref{E:ipbound_eq2} gives
\[
P(\bigcup_i\vc a_i^\top\vv\Sigma\vc y > \frac{2\sqrt{t\log N_1+t^2}\|\vv\Sigma\|_F}{\sqrt{d_1}})  \le 2 e^{-t}.
\]
Therefore
\[
P(\bigcap_i\vc a_i^\top\vv\Sigma\vc y \le \frac{2\sqrt{t\log N_1+t^2}\|\vv\Sigma\|_F}{\sqrt{d_1}})  \ge 1-2 e^{-t},
\]
which concludes the proof.
\end{proof}

The above gives the upper bound of the inner product, which connects to samples in subspaces with the following corollary.
\begin{cor}\label{cor:innerprodupperbound}
Let $\mat X\in\Real^{d\times N_1}$ be a matrix with columns randomly sampled from subspace $\mathcal S_1$ with dimensionality $d_1$ and $\vc y\in\Real^d$ be a sample from subspace $\mathcal S_2$ with dimensionality $d_2$. Both the $\vc y $ and columns of $\mat X$ have been rescaled to have unitary $\ell_2$ norm. The inner product between any column in $\mat X$ and $\vc y$ is bounded as the following
\[
\vc x_i^\top \vc y \le \frac{2\mathcal A_{1,2}\sqrt{\min\{d_1,d_2\}(t\log N_1+t^2)}}{\sqrt{d_1}}
\]
with probability at least $1-2e^{-t}$ for $t$ given previously.
\end{cor}
This is the simple application of \ref{lem:ipbound} with $\vv\Sigma = \mat U_1^\top\mat U_2$ where $\mat U_j$ $j=1,2$ is the orthonormal basis for subspace $\mathcal S_j$ and $\mathcal A_{1,2}$ is the affinity between subspaces $\mathcal S_1$ and $\mathcal S_2$ described in Definition 1.2 in \citep{soltanolkotabi2014robust}, which we recall here:
\[
\mathcal A_{\mathcal S_i, \mathcal S_j} = \sqrt{ \frac{\cos^2\theta^{(1)} + \cdots + \cos^2\theta^{(d_i \wedge d_j)}}{d_i \wedge d_j}}
\]
where $\{\cos^2\theta^{(1)}, \dots, \cos^2\theta^{(d_i \wedge d_j)}\}$ are the principal angles between subspaces $\mathcal S_i$ and $\mathcal S_j$. Please see Definition 1.1 in \citep{soltanolkotabi2014robust} for full details.

Without loss generality, we consider a sample $\vc x_1$ from subspace $\mathcal S_1$. The following theorem ensures that $k$NN will find $k$ nearest neighbors of $\vc x_1$ from $\mathcal S_1$ only.
\begin{thm}\label{thm:knn}
Let $d_m$ be the minimum dimensionality of all the subspaces. Given the conditions in Theorem \ref{thm:kconcentration}, if $\mathcal A_{\ell,1}\le\frac{\sqrt{d_m}(1-\epsilon^2/2)}{2\sqrt{d(t\log N_{\ell}+t^2)}}$, then the samples selected for any sample $\vc x_1$ from subspace $\mathcal S_1$ by using $k$NN (with $k=k_0/C$) contains no samples from other subspaces but $\mathcal S_1$ with probability at least $1-2e^{-t}-e^{-k_0}(eC)^{k_0/C}$.
\end{thm}
\begin{proof}
This is a straightforward application of Theorem \ref{thm:kconcentration} and Collary \ref{cor:innerprodupperbound} and the following. If
\[
 A_{\ell,1}\le \frac{\sqrt{d_\ell}(1-\epsilon^2/2)}{2\sqrt{\min\{d_1,d_2\}(t\log N_{\ell}+t^2)}}
\]
then
\[
\frac{2\mathcal A_{1,2}\sqrt{\min\{d_1,d_2\}(t\log N_1+t^2)}}{\sqrt{d_1}} \le 1-\epsilon^2/2.
\]
The distance and inner product is connected by
\[
d(\vc x,\vc y) = 2 - 2\vc x^\top\vc y,
\]
when vectors are normalised.
\end{proof}
The above discussion deals with clean data only. In the following section we show that the results are similar for noisey data as long as the noise level is not too great. We begin with the following series of lemmas
\begin{lem}\label{lem:IPgausandgaus}
Let random variables $X$ and $Y$ in $\Real^d$ both be from Gaussian distribution $\mathcal N(0,\sigma^2\mat I)$. For any given positive $\delta$, we have
\[
P(|X^\top Y|>\delta) \le \frac{d\sigma^4}{\delta^2}.
\]
\end{lem}
\begin{proof}
First we assume $X$ and $Y$ are standard Gaussian, we have
\[
P(|X^\top Y|>\delta) = P((X^\top Y)^2>\delta^2) \le \frac{E(X^\top Y)^2}{\delta^2},
\]
where the inequality is by Chernoff bound. Since both $X$ and $Y$ are both standard Gaussian, they are isotropic, so we have
\[
E(X^\top Y)^2 = d.
\]
The above can be obtained by
\[
E(X^\top Y)^2 = E_X\{E_{X|Y}(X^\top y)^2\} = E_Y(|y|^2) = d,
\]
where the first equality comes from law of iterative expectation, the second from isotropic property and the last from the fact that sum of standard Gaussian is Chi-square with $d$ degrees of freedom.

After proper rescaling, we obtain the result in the lemma.
\end{proof}

\begin{lem}\label{lem:ipgausandfix}
Let $X\in\Real^d$ be Gaussian random variable from $\mathcal N(0,\sigma^2\mat I)$ and $\vc y \in \Real^d$ be a fixed vector. The following holds with any positive $\delta$
\[
p(|X^\top \vc y|\ge\delta) \le \exp(1-\frac{c\delta^2}{\sigma^2\|\vc y\|_2^2}),
\]
where $c$ is a constant related to sub-Gaussian norm \citep{Vershynin2012} of a standard Gaussian.
\end{lem}
This is a straightforward application of sub-Gaussian tail to $X^\top \vc y$ with rescaling.
Now we consider the inner product between two unitary vectors in subspaces with noise. We use the following model
\be\label{e:noisemodel}
\mathbf Y = \mathbf X + \mathbf E
\ee
where $Y$ is the observation, $X$ is the clear signal in some subspace and $E$ is the noise assumed to be from $\mathcal N(0,\sigma^2\mat I)$. We assume that the observations have been rescaled properly such that $X$ is from a unit sphere in $\mathbb S^{d-1}$ and the variance of the noise is bounded, i.e. $\sigma<\sigma_0$.

First we note that under these conditions, the noise can increase and decrease inner product between observed signals by only a small amount, which is shown in the following lemma.
\begin{lem}\label{lem:genipwithnoise}
Let $\vc y_i$ ($i=1,2$) be observations from the model in \eqref{e:noisemodel}, such that $\vc y_i = \vc x_i+\vc e_i$, $\|\vc x_i\|_2 =1$ and $\vc e_i\sim\mathcal N(0,\sigma^2\mat I)$. If $P(|\vc x_1^\top\vc x_2| > v)\ge p$, we have
\[
P(\vc y_1^\top \vc y_2 > v - 3\delta) \ge p - \exp(1-\frac{c\delta^2}{\sigma^2}) - \frac{d\sigma^4}{2\delta^2}.
\]
If $P(\vc x_1^\top\vc x_2 < v)\ge p$, we have
\[
P(\vc y_1^\top \vc y_2 < v + 3\delta) \ge p - \exp(1-\frac{c\delta^2}{\sigma^2}) - \frac{d\sigma^4}{2\delta^2}.
\]
\end{lem}
\begin{proof}
We prove the $P(\vc x_1^\top\vc x_2 > v)\ge p$ case. The other cases can be proved similarly. Writing $\vc y_1^\top\vc y_2$ in terms and using triangular inequality gives
\[
|\vc y_1^\top\vc y_2| \ge |\vc x_1^\top\vc x_2| - |\vc x_1^\top\vc e_2| - |\vc e_1^\top\vc x_2| - |\vc e_1^\top \vc e_2|.
\]
Then
\[
P(\vc y_1^\top \vc y_2 > v - 3\delta) = P(|\vc x_1^\top\vc x_2|>v \bigcap \vc x_1^\top\vc e_2 \ge -\delta \bigcap \vc e_1^\top\vc x_2 \ge -\delta \bigcap \vc e_1^\top \vc e_2\ge -\delta ).
\]
By using Lemma \ref{lem:IPgausandgaus} and \ref{lem:ipgausandfix}, we obtain the desired result.
\end{proof}
Lemma \ref{lem:genipwithnoise} states that the noise will dispel the vectors when they are very close and attract them when they are far away in terms of the inner product induced distance. The effect of noise for a given sample in subspace $\mathcal S_1$ is then to make the samples from other subspaces closer to it and more difficult to separate reflected by the reduced probability as shown in the following theorem.
\begin{thm}\label{thm:knnnoise}
Let $d_m$ be the minimum dimensionality of all the subspaces. Given the conditions in Theorem \ref{thm:kconcentration} and the noise model \eqref{e:noisemodel}, for a small positive $\delta$, if $\mathcal A_{\ell,1}\le\frac{\sqrt{d_m}(1-\epsilon^2/2-6\delta)}{2\sqrt{d(t\log N_{\ell}+t^2)}}$, then the samples selected for any sample $\vc y_1$ from subspace $\mathcal S_1$ by using $k$NN (with $k=k_0/C$) contains no samples from other subspaces but $\mathcal S_1$ with probability at least $1-2e^{-t}-e^{-k_0}(eC)^{k_0/C}-2\exp(1-\frac{c\delta^2}{\sigma^2}) - \frac{d\sigma^4}{\delta^2}$.
\end{thm}
\begin{proof}
According to the model \eqref{e:noisemodel}, $\vc y_1 = \vc x_1 + \vc e_1$ and $\vc x_1$ is in a unit sphere. From Theorem \ref{thm:kconcentration}, we know that there are at least $\lfloor k_0/C \rfloor$ samples from $\mathcal S_1$ in the patch $A_\epsilon$ centred at $\vc x_1$ with probability at least $1 - e^{-k_0}(eC)^{k_0/C}$. This leads to the following combing Lemma \ref{lem:genipwithnoise}
\[
P(\min_{j\in\mathcal N_1}\{\vc y_1^\top\vc y_{i_j}\} \ge 1-\epsilon^2/2 - 3\delta) \ge 1-e^{-k_0}(eC)^{k_0/C}-\exp(1-\frac{c\delta^2}{\sigma^2}) - \frac{d\sigma^4}{2\delta^2}
\]
where $\mathcal N_1$ is the set of $\lfloor k_0/C \rfloor$ samples around $\vc x_1$ in $A_\epsilon$ patch.

Using Corollary \ref{cor:innerprodupperbound} and Lemma \ref{lem:genipwithnoise} results in that with probability at least $1-2e^{-t}-\exp(1-\frac{c\delta^2}{\sigma^2}) - \frac{d\sigma^4}{2\delta^2}$
\[
\vc y_i^\top \vc y_1 \le \frac{2\mathcal A_{1,2}\sqrt{\min\{d_1,d_2\}(t\log N_1+t^2)}}{\sqrt{d_1}} + 3\delta
\]
for any $\vc y_i$ from subspace $\mathcal S_{\ell}$.

Similar to Theorem \ref{thm:knn}, combing the above two statements, if
\[
 A_{\ell,1}\le \frac{\sqrt{d_\ell}(1-\epsilon^2/2-6\delta)}{2\sqrt{\min\{d_1,d_2\}(t\log N_{\ell}+t^2)}}
\]
then
\[
\frac{2\mathcal A_{1,2}\sqrt{\min\{d_1,d_2\}(t\log N_1+t^2)}}{\sqrt{d_1}} + 3\delta \le 1-\epsilon^2/2 - 3\delta
\]
with the probability stated in this theorem.
\end{proof}
To obtain the high probability, it is required that $\sigma$ be as low as possible given that $\delta$ is at the scale of $\epsilon$ and the probability is also tied up with the ambient dimensionality.

\section{Solving Exact kSSC via LADMPSAP}
\label{Appendix:kSSC_Exact}

The exact objective for kSSC is as follows:
\begin{align}
\min_{\mathbf Z, \mathbf E} \lambda \| \mathbf Z \|_1 + \frac{1}{2}\| \mathbf E \|_F^2 \quad \text{s.t.}\ \mathbf x_i = \sum_{j \in \Omega_i}^{k} \mathbf x_j Z_{ji} + \mathbf e_i, \textrm{diag}(\mathbf Z) = \mathbf 0.
\label{kssc_objective_exact}
\end{align}

To solve \eqref{kssc_objective_exact} one can use use LADMPSAP (Linearized Alternating Direction Method with Parallel Spliting and Adaptive Penalty) \citep{DBLP:conf/acml/LiuLS13}. Essentially LADMPSAP converts the original objective into a series of closed form proximity problems which can be solved at an element wise level. This allows us to resolve the first constraint of \eqref{kssc_objective_exact} since we can enforce it by ignoring those elements outside of $\Omega$. LADMPSAP provides guaranteed convergence and we refer readers to \citep{DBLP:conf/acml/LiuLS13} for full details. Furthermore the sub-steps involving the update of primary variables are independent of each other and if one chooses can be computer parallel. However since we already compute each column of $\mathbf Z$ in parallel we find that such further parallelisation is unnecessary. 

The full algorithm can be found in Algorithm \ref{Algorithm:kSSCExact}. We begin discussing the details by re-writing, with some abuse of notation, the original objective \eqref{kssc_objective_exact} for a single column of $\mathbf Z$ and $\mathbf E$
\begin{align}
\min_{\mathbf z_i, \mathbf e_i} \lambda \| \mathbf z_i \|_1 + \frac{1}{2}\| \mathbf e_i \|_2^2 \quad \text{s.t.}\ \mathbf x_i = \mathbf X_i \mathbf z_i + \mathbf e_i
\end{align}
where $\mathbf z_i = \mathbf z_{{\Omega_i}i}$ i.e.\ the vector of rows $\Omega_i$ and column $i$ of $\mathbf Z$, $\mathbf  e_i$ is column $i$  of $\mathbf E$ and $\mathbf X_i = \mathbf X_{(:,\Omega_i)}$ i.e.\ the matrix formed from the columns of $\mathbf X$ indexed by $\Omega_i$. Note that we have removed the constraint $\textrm{diag}(\mathbf Z) = \mathbf 0$ since we enforce it by ensuring that no diagonal entries are present in each $\Omega_i$.

Next we relax the constraints and form the Lagrangian function 
\begin{align}
\min_{\mathbf z_i, \mathbf e_i, \mathbf y_i} \lambda \| \mathbf z_i \|_1 + \frac{1}{2}\| \mathbf e_i \|_F^2
+ \langle \mathbf y_i , \mathbf X_i \mathbf z_i  - \mathbf x_i + \mathbf e_i \mathbf  \rangle + \frac{1}{2} \| \mathbf X_i \mathbf z_i  - \mathbf x_i + \mathbf e_i \mathbf \|_2^2
\end{align}
where $\mathbf y_i$ is a Lagrange multiplier. Then LADMPSAP consists of an iterative procedure over the optimisation variables We outline the procedure in Algorithm \ref{Algorithm:kSSCExact} Here we detail the solutions to the sub problems. Denote $\mathbf z_i^{\mathtt t}, \mathbf e_i^{\mathtt t}, \mathbf y_i^{\mathtt t}$ the variables at iteration $t$. Then:
\begin{enumerate}
\item Fix others and update $\mathbf z_i^{\mathtt t+1}$
\[
\mathbf z_i^{\mathtt t+1} = \argmin_{\mathbf{z_i}} \lambda\|\mathbf z_i\|_{1} + \frac{\rho}{2} \| \mathbf z_i - (\mathbf z_i^{\mathtt t} - \frac{1}{\rho} \partial F(\mathbf z_i^{\mathtt t})) \|_2^2
\]
where $\partial F(\mathbf z_i^{\mathtt t}) =  \mu^{\mathtt t} \mathbf X_i^T (\mathbf X_i \mathbf z_i^{\mathtt t} - (\mathbf x_i - \mathbf e_i^{\mathtt t} - \frac{1}{\mu} \mathbf y_i^{\mathtt t} ))$. For which we have a closed form solution given by
\begin{align}
\mathbf z_i^{\mathtt t+1} = \textrm{sign}(\mathbf b^{\mathtt t} ) \max(|\mathbf b^{\mathtt t} | - \frac{\lambda}{\rho}, 0),
\label{update_z}
\end{align}
where $\mathbf b_i^{\mathtt t} = \mathbf z_i^{\mathtt t} - \frac{1}{\rho} \partial F(\mathbf z_i^{\mathtt t})$, see \citep{bach2011convex, liu2010efficient} for details.

\item Fix others and update $\mathbf e_i^{\mathtt t+1}$
\[
\mathbf e^{\mathtt t+1} = \argmin_{\mathbf{e_i}} \frac{1}{2} \|\mathbf e_i \|_2^2  + \frac{\mu^t}{2} \|\mathbf e_i - (\mathbf x_i - \mathbf X_i \mathbf z_i^{\mathtt t} - \frac{1}{\mu^t} \mathbf y_i^{\mathtt t}) \|_2^2
\]
which has the closed form solution given by
\begin{align}
\mathbf e_i^{\mathtt t+1} = \frac{\mathbf x_i - \mathbf X_i \mathbf z_i^{\mathtt t} - \frac{1}{\mu^{\mathtt t}} \mathbf y_i^{\mathtt t}}{\frac{1}{\mu^{\mathtt t}} + 1},
\label{update_e}
\end{align}

\item Update $\mathbf y_i^{\mathtt t+1}$
\begin{align}
\mathbf y_i^{\mathtt t+1} = \mathbf y_i^{\mathtt t} + \mu^{\mathtt t} (\mathbf X_i \mathbf z_i^{\mathtt t+1}  - \mathbf x_i + \mathbf e_i^{\mathtt t+1})
\label{update_y}
\end{align}

\end{enumerate}

\begin{algorithm}[!t]
\caption{Solving \eqref{kssc_objective_exact} via LADMPSAP}
\begin{algorithmic}[1]
\label{Algorithm:kSSCExact}

\REQUIRE $\mathbf x_i$, $\mathbf X_i$, $\Omega_i$, $\lambda$, $\epsilon_1$, $\epsilon_2$,

\STATE Initialise: $\mathbf z_i = \mathbf 0$, $\mathbf e_i = \mathbf 0$, $\mathbf y_i = \mathbf 0$, $\mu = 0.1$, $\mu_{\text{max}} = 1$, $\gamma^0 = 1.1$, $\rho = \| \mathbf X_i \|_F$

\WHILE{not converged}

\STATE Update $\mathbf z_i^{\mathtt t+1}$ using \eqref{update_z}
\STATE Update $\mathbf e_i^{\mathtt t+1}$ using \eqref{update_e}
 
 \STATE Set $q$
\begin{align*}
q = \frac{\mu^{\mathtt t} \sqrt{\rho}}{\| \mathbf X_i \|_F} \textrm{max} ( \| \mathbf z_i^{\mathtt t+1} - \mathbf z_i^{\mathtt t} \|_2, \| \mathbf e_i^{\mathtt t+1} - \mathbf e_i^{\mathtt t}\|_2  )
\end{align*}

\STATE Check stopping criteria
\begin{align*}
\frac{\| \mathbf X_i \mathbf z_i^{\mathtt t+1}  - \mathbf x_i + \mathbf e_i^{\mathtt t+1} \|_2 }{\| \mathbf X_i \|_F} < \epsilon_1, q < \epsilon_2
\end{align*}

\STATE Update $\mathbf y_i^{\mathtt t+1}$ using \eqref{update_y}

\STATE Update $\gamma^{\mathtt t+1}$
\begin{align*}
\gamma^{\mathtt t+1} = 
\begin{cases}
\gamma^0 & \text{if} \;\; q < \epsilon_2 \\
1 & \text{otherwise,}
\end{cases}
\end{align*}

\STATE Update $\mu^{\mathtt t+1}$
\begin{align*}
\mu^{\mathtt t+1} = \textrm{min}( \mu_{\text{max}}, \gamma \mu^{\mathtt t})
\end{align*}
 
\ENDWHILE

\RETURN $\mathbf z_i$

\end{algorithmic}
\end{algorithm}

\section{Solving Relaxed SSC via Accelerated Gradient Descent}
\label{Appendix:SSC_Relaxed}

Here we discuss how to solve \eqref{SSCRelaxedObjective} via FISTA. The full algorithm is outlined in Algorithm \ref{Algorithm:SSC_relaxed}. Denote
\begin{align}
\min_{\mathbf Z} L = \lambda\|\mathbf Z\|_{1} + \frac{1}{2}\|\mathbf X - \mathbf X\mathbf Z\|^2_F
\end{align}
and
\[
\min_{\mathbf Z} \widetilde{L}_{\rho}(\mathbf{Z, Z^{\mathtt t}}) =  \lambda\| \mathbf Z \|_{1} + \frac{\rho}{2} \| \mathbf Z - ( \mathbf Z^{\mathtt t} - \frac{1}{\rho} \partial F(\mathbf Z^{\mathtt t})) \|_F^2, 
\]
$F = \frac12\|\mathbf X - \mathbf X\mathbf Z\|^2_F$ and $\partial F = - \mathbf X^T (\mathbf X - \mathbf X \mathbf Z^{\mathtt t}$).
The solution is given by the closed-form $\ell_1$ shrinkage function $\mathcal S_{\tau}$ as follows
\begin{align}
\mathcal S_{\frac{\lambda}{\rho}}(\mathbf Z^{\mathtt t}) = \textrm{sign}(\mathbf Z^{\mathtt t} - \frac{1}{\rho} \partial F(\mathbf Z^{\mathtt t})) \max(| \mathbf Z^t - \frac{1}{\rho} \partial F(\mathbf Z^{\mathtt t}) | - \frac{\lambda}{\rho}, 0).
\end{align}

\begin{algorithm}
\caption{Solving \eqref{SSCRelaxedObjective} via FISTA}
\begin{algorithmic}
\label{Algorithm:SSC_relaxed}

\REQUIRE $r = \infty$, $\mathbf Z = \mathbf 0$,  $\mathbf J = \mathbf 0$, $\alpha = 1$, $\lambda$, $\rho$, $\gamma$, $\epsilon$

\WHILE{$r^{\mathtt t} - r^{\mathtt t-1} \geq \epsilon$ }

	\WHILE{$L(\mathcal S_{\frac{\lambda}{\rho}}(\mathbf J^{\mathtt t})) \geq \widetilde{L}_{\rho}(\mathcal S_{\frac{\lambda}{\rho}}(\mathbf J^{\mathtt t}), \mathbf J^{\mathtt t})$}
	
		\STATE $\rho = \gamma \rho$	
	
	\ENDWHILE

	\STATE $\mathbf Z^{\mathtt t+1} = \mathcal S_{\frac{\lambda}{\rho}}(\mathbf J^{\mathtt t}))$
	\STATE $\alpha^{\mathtt t+1} = \frac{1 + \sqrt{1 + 4 \alpha^t{^2})}}{2}$
	\STATE $\mathbf J^{\mathtt t+1} = \mathbf Z^{\mathtt t+1} + \left ( \frac{\alpha^{\mathtt t} - 1}{\alpha^{\mathtt t+1}} \right ) (\mathbf Z^{\mathtt t+1} - \mathbf Z^{\mathtt t})$
	\STATE $r^{\mathtt t+1} = \frac12\|\mathbf X - \mathbf X\mathbf Z^{\mathtt t+1}\|^2_F + \lambda\|\mathbf Z^{\mathtt t+1} \|_{1}$
	
\ENDWHILE

\end{algorithmic}
\end{algorithm}

\section{Solving Exact SSC via LADMPSAP}
\label{Appendix:SSC_Exact}

Here we discuss how to solve \eqref{SSCExactObjective} via LADMPSAP. The full algorithm is outlined in Algorithm \ref{Algorithm:SSCExact}. Denote the Augmented Lagrangian form
\[
\min_{\mathbf{E, Z}} \frac12\|\mathbf E \|^2_F + \lambda\|\mathbf Z\|_{1} + \langle \mathbf{Y, XZ - X + E } \rangle + \frac{\mu}{2} \|\mathbf{XZ - X + E } \|_F^2 
\]
Iterate the following:
\begin{enumerate}
\item Fix others and update $\mathbf Z^{\mathtt t+1}$
\[
\min_{\mathbf{Z}} \lambda\|\mathbf Z\|_{1} + \langle \mathbf{Y, XZ} \rangle + \frac{\mu}{2} \|\mathbf{XZ - (X - E) } \|_F^2 
\]
\[
\min_{\mathbf{Z}} \lambda\|\mathbf Z\|_{1} + \frac{\rho}{2} \| \mathbf Z - (\mathbf Z^{\mathtt t} - \frac{1}{\rho} \partial F(\mathbf Z^{\mathtt t}) \|_F^2 
\]
where $F = \frac{\mu}{2} \|\mathbf{XZ - (X - E} - \frac{1}{\mu} \mathbf Y ) \|_F^2$ and $\partial F =  \mu^{\mathtt t} \mathbf X^T (\mathbf{XZ^{\mathtt t} - (X - E} - \frac{1}{\mu^{\mathtt t}} \mathbf Y \mathbf Z^{\mathtt t} ))$. Then 
\begin{align}
\mathbf Z^{\mathtt t+1} = \textrm{sign}(\mathbf B^{\mathtt t} ) \max(|\mathbf B^{\mathtt t} | - \frac{\lambda}{\rho}, 0),
\label{ExactSSC:update_z}
\end{align}
where $\mathbf B_i^{\mathtt t} = \mathbf Z^{\mathtt t} - \frac{1}{\rho} \partial F(\mathbf Z^{\mathtt t})$.

\item Fix others and update $\mathbf E^{\mathtt t+1}$
\[
\min_{\mathbf{E}} \frac12\|\mathbf E \|^2_F + \langle \mathbf{Y, E} \rangle + \frac{\mu^{\mathtt t}}{2} \|\mathbf{XZ - X + E } \|_F^2 
\]
\[
\min_{\mathbf{E}} \frac12\|\mathbf E \|^2_F  + \frac{\mu^{\mathtt t}}{2} \|\mathbf E - (\mathbf X - \mathbf X \mathbf Z^{\mathtt t} - \frac{1}{\mu^{\mathtt t}} \mathbf Y^{\mathtt t}) \|_F^2 
\]
\begin{align}
\mathbf e_i^{\mathtt t+1} = \frac{\mathbf X - \mathbf X \mathbf Z^{\mathtt t} - \frac{1}{\mu^{\mathtt t}} \mathbf Y^{\mathtt t}}{\frac{1}{\mu^{\mathtt t}} + 1},
\label{ExactSSC:update_e}
\end{align}

\item Update $\mathbf Y^{\mathtt t+1}$
\begin{align}
\mathbf Y = \mathbf Y + \mu^{\mathtt t} (\mathbf{XZ}^{\mathtt t+1} - \mathbf{X} + \mathbf{E}^{\mathtt t+1})
\label{ExactSSC:update_y}
\end{align}

\end{enumerate}

\begin{algorithm}[!t]
\caption{Solving \eqref{SSCExactObjective} via LADMPSAP}
\begin{algorithmic}[1]

\label{Algorithm:SSCExact}

\REQUIRE $\mathbf X$, $\lambda$, $\epsilon_1$, $\epsilon_2$,

\STATE Initialise: $\mathbf Z = \mathbf 0$, $\mathbf E = \mathbf 0$, $\mathbf Y = \mathbf 0$, $\mu = 0.1$, $\mu_{\text{max}} = 1$, $\gamma^0 = 1.1$, $\rho = \| \mathbf X_i \|_F$

\WHILE{not converged}

\STATE Update $\mathbf Z_i^{\mathtt t+1}$ using \eqref{ExactSSC:update_z}
\STATE Update $\mathbf E_i^{\mathtt t+1}$ using \eqref{ExactSSC:update_e}
 
 \STATE Set $q$
\begin{align*}
q = \frac{\mu^{\mathtt t} \sqrt{\rho}}{\| \mathbf X \|_F} \textrm{max} ( \| \mathbf Z^{\mathtt t+1} - \mathbf Z^{\mathtt t} \|_F, \| \mathbf E^{\mathtt t+1} - \mathbf E^{\mathtt t}\|_F  )
\end{align*}

\STATE Check stopping criteria
\begin{align*}
\frac{\| \mathbf X \mathbf Z^{\mathtt t+1} - \mathbf X + \mathbf E^{\mathtt t+1} \|_2 }{\| \mathbf X \|_F} < \epsilon_1, q < \epsilon_2
\end{align*}

\STATE Update $\mathbf Y_i^{\mathtt t+1}$ using \eqref{ExactSSC:update_y}

\STATE Update $\gamma^{\mathtt t+1}$
\begin{align*}
\gamma^{\mathtt t+1} = 
\begin{cases}
\gamma^0 & \text{if} \;\; q < \epsilon_2 \\
1 & \text{otherwise,}
\end{cases}
\end{align*}

\STATE Update $\mu^{\mathtt t+1}$
\begin{align*}
\mu^{\mathtt t+1} = \textrm{min}( \mu_{\text{max}}, \gamma \mu^{\mathtt t})
\end{align*}
 
\ENDWHILE

\RETURN $\mathbf Z$

\end{algorithmic}
\end{algorithm}

\end{document}